\documentclass[11pt]{article}
\usepackage[margin=1.1in]{geometry}
\usepackage{amsmath,amssymb,amsthm,mathtools}
\usepackage{booktabs}
\usepackage{longtable}
\usepackage{seqsplit}
\usepackage{graphicx}
\usepackage[numbers,sort&compress]{natbib}
\usepackage[colorlinks=true,linkcolor=blue,citecolor=blue,urlcolor=blue]{hyperref}
\usepackage[capitalize,noabbrev]{cleveref}
\usepackage{xcolor}
\usepackage{tikz}
\usetikzlibrary{shapes.geometric,shapes.misc,arrows.meta,positioning,fit,backgrounds,calc,patterns}

\definecolor{figblack}{gray}{0.0}
\definecolor{figdark}{gray}{0.25}
\definecolor{figmid}{gray}{0.5}
\definecolor{figlight}{gray}{0.85}
\definecolor{figpale}{gray}{0.95}

\theoremstyle{plain}
\newtheorem{theorem}{Theorem}[section]
\newtheorem{lemma}[theorem]{Lemma}
\newtheorem{proposition}[theorem]{Proposition}
\newtheorem{corollary}[theorem]{Corollary}

\theoremstyle{definition}
\newtheorem{definition}[theorem]{Definition}

\theoremstyle{remark}
\newtheorem{remark}[theorem]{Remark}
\newtheorem{observation}[theorem]{Observation}

\newcommand{\Phic}{\Phi_{C}}                 
\newcommand{\PhiY}[1]{\Phi_{#1}}             
\newcommand{\Aut}{\mathrm{Aut}}              
\newcommand{\AutX}{\Aut(\Phic)}              
\newcommand{\Sym}{\mathrm{Sym}}
\newcommand{\Gind}{\prod_{i=1}^{n}\Sym(S_i)} 
\newcommand{\Ftwo}{\mathbb{F}_2}
\newcommand{\NONTRIV}{\textup{\textsc{Nontriv-Aut}}}
\newcommand{\DEADVAR}{\textup{\textsc{Dead-Var}}}
\newcommand{\VALISO}{\textup{\textsc{Val-Iso}}}
\newcommand{\coNP}{\mathsf{coNP}}
\newcommand{\SigmaTwoP}{\Sigma_2^{\mathsf{p}}}
\newcommand{\PH}{\mathsf{PH}}
\newcommand{\AM}{\mathsf{AM}}

\newcommand{\citeyearonly}[1]{\cite{#1}}   

\title{Reasoning Shortcuts and Value Symmetries: What Symmetry\\
Permits, Architecture Realizes, and Optimization Selects}
\author{Xin Xu\\
Carnegie Mellon University \,\textperiodcentered\, University of Pennsylvania\\
\texttt{xuxin@cmu.edu}\quad\texttt{xinx@upenn.edu}}
\date{}

\begin{document}
\maketitle

\begin{abstract}
Reasoning shortcuts are rule solutions that reach correct predictions through unintended concepts. A recent framework of Takemura, Inoue, and Nishino analyzes them through an automorphism group of value relabelings, asking when rules pin concepts down. Its key definition, one value permutation shared across all positions, does not apply as stated to any of its four heterogeneous benchmarks, and the most direct embedding, padding domains to a common size, produces confident false pathology: 90.91\% of solution pairs unexplained on CLE4EVR, versus 0\% under every well-defined rung of the hierarchy we introduce, whose componentwise member grants each position its own native-domain permutation; the padded verdict rotates under configuration-file ordering. Across eleven rule families under fifteen pre-specified predictions (thirteen confirmed), unexplained-pair rates span 0\% to 99.9999\% and track provable structure: six theorems give sufficient conditions for transitivity and its failure. For circuit-given rules, symmetry-inertness of a coordinate is coNP-complete; automorphism existence is coNP-hard under randomized reductions, lies in $\Sigma_2^p$, is not $\Sigma_2^p$-complete in the Boolean case unless PH collapses, and is coNP-complete on monotone circuits. Boolean transitivity is classified exactly: automorphisms explain everything iff the solution set is an affine coset. Weakly supervised models place all 94 observed shortcuts at the one level the theory flags, none at the 48 it certifies transitive, and none at twelve typed-ambiguous levels. Relocating the rule's absorbing element moves every shortcut with it; a confusion null attributes the location to geometry while the observed rate exceeds it by half again. Trained end to end on CLE4EVR's rule and heterogeneous domains through a synthetic prototype front end, models produce 20,223 label-preserving errors with zero different-orbit exceptions, as transitivity predicts, where the padded instrument would misreport 78-88\% of them. Every number traces through a source comment to a machine-computed artifact.
\end{abstract}

\section{Introduction}
\label{sec:intro}

A neurosymbolic system couples a learned perception component with a
fixed body of symbolic knowledge. The perception component maps raw
inputs to symbolic concepts. The knowledge maps concepts to predictions.
When the system reaches the right predictions through the wrong
concepts, it has learned a \emph{reasoning shortcut}: every rule is
satisfied at training time, yet the learned concept mapping is not the
intended one, and the error surfaces only later, on downstream tasks
that reuse the concepts~\cite{marconato2023shortcuts,rsbench2024}. A
self-driving benchmark makes the stakes concrete. A model can learn to
predict \emph{stop} correctly while confusing pedestrians with red
lights, because both concepts imply stopping~\cite{rsbench2024}. The
prediction is right. The concept is wrong. Everything downstream of the
concept is compromised.

Takemura, Inoue and Nishino~\citeyearonly{takemura2026} recently gave this
problem an algebraic form. Model the rules as a constraint satisfaction
problem, collect the concept mappings that satisfy every rule into a
solution set, and ask what structure the set of alternative solutions
carries. Their central tool is an automorphism group: a set of value
relabelings that map solutions to solutions. If two shortcuts differ
only by an automorphism, they are the same failure wearing two labels,
and symmetry-aware training or evaluation can treat them as one. Their
paper closes by naming its most pressing open question: characterize
when the rules pin the concepts down, beyond the case where the
automorphism group is trivial.

This paper starts from a fact about that framework which, to our
knowledge, has not been reported. The framework's central definition
does not apply as stated to the native, typed structure of the
benchmarks it was evaluated on. Takemura et al.'s
automorphism group applies one shared permutation of one shared value
set at every position. The four rsbench benchmarks in their own
evaluation all have heterogeneous attribute domains: colors, shapes and
materials of different sizes, or twenty binary detectors whose count
can never equal their two-valued domain's size. On such instances the
definition does not apply as stated: its hypotheses fail (the shared
domain for CLE4EVR and Kandinsky, the bijectivity requirement
$|N|=|S|$ for BDD-OIA and SDD-OIA), and every use requires an embedding
of typed attribute domains into one shared value set that the framework
itself does not supply. The most direct such embedding, padding every domain
to a common size, does not fail conservatively. On CLE4EVR's
real rule it reports that $90.91\%$ of solution pairs are unexplained
pathology, where every well-defined member of the value-symmetry
hierarchy reports $0\%$ on the same instance. The reported
pathology decomposes entirely into bookkeeping: which values the padded
group identifies as interchangeable is decided by the order in which
attributes happen to be listed in a configuration file, and permuting
that order permutes the answer. A practitioner who reaches for the published
definition on a real, heterogeneous benchmark gets a confident, wrong,
and unstable number.

We therefore generalize the definition rather than the benchmarks.
\emph{Componentwise value symmetry} lets every position carry its own
value permutation over its own domain. On homogeneous instances it
contains the original group whenever both are read off one common
solution set (\Cref{sec:background}). On heterogeneous instances it is the only
one of the two definitions that is defined at all. The rest of the
paper asks, with this corrected instrument, the question the original
framework raised: when do automorphisms account for the alternative
solutions of real neurosymbolic rules, when do they not, what
distinguishes the two cases, and how hard is it to decide?

The answers are specific enough to summarize, and the sharpest of them
anchors the rest: in the Boolean case, which covers BDD-OIA and
SDD-OIA, transitivity is classified exactly, automorphisms explain
every alternative solution precisely when the solution set is an
affine coset (\Cref{sec:complexity}), and the empirical chapter's
measured percentages are that law's values to the last digit. The
contributions run as follows.

\paragraph{Contribution 1: a corrected instrument (\Cref{sec:background}).}
We show that porting the global definition to heterogeneous benchmarks
by padding produces confident false pathology, quantify it on CLE4EVR's
public rule ($90.91\%$ against $0\%$, with the content of the reported
pathology tracking configuration-file order rather than rule
structure), and confirm on a $5{,}433$-instance synthetic sweep that
the padded reading collapses to the trivial group on $97\%$ of
heterogeneous instances. Padding is not the only translation: a
type-tagged disjoint-union encoding is just as direct and recovers
exactly the typed rung of our hierarchy (\Cref{sec:background}). The
definition attaches to a heterogeneous benchmark only through a
translation it does not specify, different translations return
different verdicts, and the padded verdict agrees with no rung.
Componentwise symmetry needs no translation at all.

\paragraph{Contribution 2: structural determinants of pathology
(\Cref{sec:empirical}).}
Across three published benchmark families and eight further rule
families, with fifteen structure-to-outcome predictions recorded before
measurement (thirteen confirmed), the fraction of solution pairs that
automorphisms fail to explain ranges from $0\%$ to $99.9999\%$, and the
variation tracks checkable structural features of the rule. Mild
conjunctive structure is transitive: every solution pair is related by
an automorphism. One checkable criterion, branches exchangeable by a value permutation
with each branch internally transitive, separated transitive from
pathological disjunctions in every instance on which it was checked.
An absorbing element, a multiplicative zero, breaks an otherwise safe
conjunction. A counting bound forces pathology on all-different grids
whose solutions outnumber the diagonal cap. The two counterexamples isolated the
operative condition, branch exchangeability by value permutation
rather than syntactic disjunction, that the coarse prediction missed.

\paragraph{Contribution 3: six theorems (\Cref{sec:algebra}).}
A single Forcing
Lemma underlies sufficient conditions for transitivity (matching
decompositions, exchangeable branches) and for its failure (anchored
inequality patterns, degree invariants, an orbit-stabilizer counting
bound, and a Free Slot Lemma that certifies Kandinsky's intransitivity
directly from its branch structure, without enumerating the group). A byproduct settles a natural conjecture negatively:
conjunctive structure alone does not imply transitivity, by an explicit
two-constraint counterexample found in a randomized search of $4{,}000$
instances and then proved as a theorem.

\paragraph{Contribution 4: the complexity of deciding
(\Cref{sec:complexity}).}
For rules given as compact circuits, the Boolean case gets an exact
structure theory: $\AutX$ is always an $\Ftwo$-linear subspace, equal
to the orthogonal complement of the Fourier support, which explains a
constant the empirical chapter measures eight times over, and
transitivity admits a complete classification, not merely sufficient
conditions: the group acts freely, so $\AutX$ explains all shortcuts
exactly when $\Phic$ is an affine coset, and an exact orbit law
reproduces every BDD-OIA percentage of \Cref{sec:empirical} to the
last digit. On monotone circuits, deciding whether any nontrivial
automorphism exists (\NONTRIV) is coNP-complete under deterministic
reductions, because monotone functions provably cannot carry the
camouflage symmetries that force the general reduction to randomize.
In general, \NONTRIV{} is coNP-hard under randomized reductions and
lies in $\SigmaTwoP$; the main complexity result is that, in the
Boolean case, it is not $\SigmaTwoP$-complete unless the polynomial
hierarchy collapses to $\Sigma_3^{\mathsf p}$. The collapse theorem
adapts Agrawal and Thierauf's argument for Boolean
isomorphism~\citeyearonly{agrawalthierauf2000formula} to a group that permutes
values rather than named positions, through a laundering protocol that
never names objects at all. The randomized hardness survives the nonempty-solution
promise natural to neurosymbolic tasks
(\Cref{prop:nonempty-hardness}), and the designated-coordinate
hardness survives even when every other coordinate is known alive
(\Cref{prop:dead-var-oracle}).
Deciding whether a designated coordinate is symmetry-inert (\DEADVAR)
is coNP-complete, a known baseline we state for
completeness~\cite{beyersdorff2009implication}.

\paragraph{Contribution 5: the shortcut geography of trained models
(\Cref{sec:realmodels}).}
We train standard weakly supervised neurosymbolic models on rsbench's
MNIST arithmetic tasks, ten seeds per task. All $94$ observed
concept-level shortcuts occur at exactly the one target level the
componentwise theory identifies as pathological, and none at the $48$
levels it certifies as transitive ($34{,}800$ of the experiment's
$40{,}000$ model-instance evaluations). The typed reading then sharpens
the result into this paper's cleanest experiment: it flags twelve
further levels as architecture-realizable ambiguities, and not one of
them produces a single shortcut; every observed failure sits at the
one level carrying the multiplicative zero, the location the loss
algebra predicted in advance. A dual-head
control, rerunning everything with independently weighted perception
networks so that the full componentwise group becomes
architecture-realizable, replicates the geography under pre-specified
predictions: all 101 of its \texttt{product} shortcuts land on the same
single level, the twelve typed-ambiguous levels stay empty, and the
same-orbit share moves only from $70.2\%$ to $71.3\%$. Three further
measurements separate what the geometry fixes from what training
contributes. A confusion-matrix null, sampling independent slot errors
from each trained model's own per-digit confusion rates, concentrates
at the same level, so the location follows from the absorbing-element
geometry alone; the trained models exceed that null's rate by half
again ($9.4$ against $6.2$ label-preserving errors per seed, positive
in $10/10$ seeds), an excess that independent slot errors at the
models' own accuracies cannot produce. Conjugating the rule so that its absorbing value
moves from $0$ to $7$ moves every observed shortcut with it. And the
model's predictive mass, not just its argmax, sits on orbit-external
solutions ($0.75\%$) exactly where the theory says such solutions
exist, and at exactly zero everywhere else. A second experiment trains
end to end on CLE4EVR's rule and heterogeneous domains through a
synthetic prototype front end, the setting the
methodological argument is about: because $\AutX$ is transitive there,
every label-preserving error must be same-orbit, and across
$20{,}223$ such errors from twenty models it is, with zero exceptions,
while the padded instrument would report $78$ to $88\%$ of those same
real errors as unexplained pathology. Automorphism
orbits explain $70\%$ of the observed shortcuts, and the correspondence
between symbolic prediction and learned failure separates three things
a single number usually conflates: what symmetry permits, what
architecture realizes, and what optimization selects.

\paragraph{Instrument and evidence discipline.}
Componentwise value symmetry is positioned, not presumed. It is the
finest member of a hierarchy of well-defined value-symmetry readings
(global, typed, componentwise; \Cref{sec:background}); the members
agree exactly on the headline instances, and where they genuinely part,
on unconstrained slots and on equality patterns, \Cref{sec:realmodels}
measures the divergence and turns it into the paper's sharpest
experiment. No member of the hierarchy supports the padded verdict. A
second axis is pinned just as precisely: a fiber-intersection identity
(\Cref{sec:empirical}) connects the global symmetries of a rule to the
per-label symmetries its fibers carry, binary-output rules have no gap
between the two, and the arithmetic tasks' global groups are provably
trivial, which is exactly why the fiber level is where their
identifiable structure lives. The
theory arrives in two matched halves, orbit combinatorics for general
domains and linear algebra for Boolean ones, and each half carries a
real benchmark that needs it: Kandinsky's certificate comes from the
Free Slot Lemma, BDD-OIA's from the subspace structure. Every number in
the paper traces, through a comment in the source, to a
machine-computed artifact in the verification data
(\Cref{app:verification}).

\Cref{sec:background} develops the framework and the false positive.
\Cref{sec:empirical} reports the measurements. \Cref{sec:algebra} proves
the six theorems. \Cref{sec:complexity} maps the decision problems.
\Cref{sec:realmodels} closes the loop on trained models, and
\Cref{sec:related,sec:conclusion} situate and conclude.
\section{The Value-Symmetry Hierarchy on Heterogeneous Domains}
\label{sec:background}

This section establishes the paper's central methodological finding
and builds the formal apparatus the rest of the paper uses. Forcing
Takemura et al.'s original definition onto a heterogeneous benchmark
by padding its domains to a common size reports a $90.91\%$ pathology
rate on CLE4EVR's real rule
where the componentwise definition finds none: the naive port is not
merely less general, it is wrong.
Section~\ref{sec:takemura-framework} recalls constraint satisfaction (CSP)
notation and the automorphism-based account of reasoning shortcuts due to
Takemura, Inoue and Nishino~\citeyearonly{takemura2026}, and shows that their
central definition, a single global permutation of one shared value set,
does not apply as stated to any of the four real neurosymbolic benchmarks
it was evaluated on. Section~\ref{sec:componentwise} introduces the
generalization this paper studies, componentwise value symmetry
$\AutX$, in which every position chooses its own value permutation.
Section~\ref{sec:padding-false-positive} measures the false positive
that makes the generalization mandatory.

\subsection{Reasoning shortcuts and the automorphism framework}
\label{sec:takemura-framework}

A neurosymbolic system learns to map raw inputs to symbolic concepts,
then reasons over those concepts with a fixed set of logical or
arithmetic rules. A reasoning shortcut occurs when the system satisfies
every rule at training time without recovering the intended concept
mapping. The network learns some mapping consistent with the rules, but
not necessarily the one a human would call correct. Takemura, Inoue and
Nishino~\citeyearonly{takemura2026} formalize this as a constraint satisfaction
problem and ask when the rules alone pin the mapping down to a single,
intended solution.

\begin{definition}[Constraint satisfaction instance]
\label[definition]{def:general-csp}
Fix a finite index set of positions $N=\{1,\dots,n\}$. Each position
$i\in N$ has a finite, nonempty local domain $S_i$. The domains need not
coincide, so $|S_i|$ can differ across positions. A constraint set $C$ is
a finite set of relations, each over some subset of positions. A mapping
$\phi\in\prod_{i\in N}S_i$ satisfies $C$ if it satisfies every relation
in $C$. The solution set is
\[
\Phic \;=\; \Big\{\, \phi \in \textstyle\prod_{i\in N} S_i \;:\;
\phi \text{ satisfies every constraint in } C \,\Big\}.
\]
A constraint-based neurosymbolic learning problem additionally fixes an
intended mapping $\phi^*\in\Phic$ and a dataset $D$. A mapping
$\phi\in\Phic$ is a \emph{reasoning shortcut} if $\phi\neq\phi^*$. The
problem is \emph{shortcut-free} if $\Phic=\{\phi^*\}$, and the shortcut
multiplicity is $\mathrm{SM}(C):=|\Phic|-1$.
\end{definition}

\begin{remark}
Definition~\ref{def:general-csp} is Takemura et al.'s framework~\citeyearonly{takemura2026}, generalized to allow heterogeneous local domains
$S_i$. Their Definition~1 is the special case $S_i\equiv S$ for all
$i\in N$, so that a solution is a mapping $N\to S$. Their Definition~3
(valid mapping and shortcut) and Definition~4 (shortcut multiplicity)
match Definition~\ref{def:general-csp} verbatim, with one notational
point worth flagging. Takemura et al. work by default with the bijective
restriction of the solution set. They write $\Phi_C^{\mathrm{all}}$ for
the general set we call $\Phic$, and reserve the unornamented $\Phi_C$
for its bijective subset
$\Phi_C^{\mathrm{bij}}=\{\phi\in\Phi_C^{\mathrm{all}} : \phi \text{ is a
bijection}\}$. Throughout this paper, $\Phic$ denotes the general,
not-necessarily-bijective solution set, matching their
$\Phi_C^{\mathrm{all}}$.
\end{remark}

Takemura et al.'s central definition targets the bijective case
directly. It requires $S_i\equiv S$ for every position, so that
$\Phic^{\mathrm{bij}}$ is well defined, and it requires $|N|=|S|$, since
a bijection $N\to S$ can only exist when the two sets have the same
size.

\begin{definition}[Global value symmetry; Takemura et al.~\citeyearonly{takemura2026}, Definition~7]
\label[definition]{def:takemura-autx}
Suppose $S_i\equiv S$ for all $i\in N$ and $|N|=|S|$. Each permutation
$\sigma\in\Sym(S)$ acts on mappings by post-composition,
$(\sigma\circ\phi)(i)=\sigma(\phi(i))$. The automorphism group of
$X=(N,S,C)$ is
\[
\Aut(X) \;:=\; \big\{\, \sigma\in\Sym(S) \;:\; \sigma\circ\phi\in
\Phic^{\mathrm{bij}} \text{ for all } \phi\in\Phic^{\mathrm{bij}} \,\big\}.
\]
\end{definition}

A single global permutation is applied identically at every position.
$\Aut(X)$ is the setwise stabilizer of $\Phic^{\mathrm{bij}}$ inside
$\Sym(S)$, acting diagonally.

\begin{proposition}[Value-symmetry elimination; Takemura et al.~\citeyearonly{takemura2026}, Proposition~2]
\label[proposition]{prop:value-symmetry-elimination}
If $\Aut(X)$ is trivial, that is $\Aut(X)=\{\mathrm{id}\}$, then no two
distinct solutions in $\Phic^{\mathrm{bij}}$ are related by a
permutation in $\Aut(X)$.
\end{proposition}

The proof is immediate. If $\phi'=\sigma\circ\phi$ for some
$\sigma\in\Aut(X)$, triviality forces $\sigma=\mathrm{id}$, so
$\phi'=\phi$. Proposition~\ref{prop:value-symmetry-elimination} rules
out one specific mechanism for multiple solutions: a global relabeling
drawn from $\Aut(X)$ itself. It does not imply
$|\Phic^{\mathrm{bij}}|=1$. Two solutions can be related by a
permutation outside $\Aut(X)$, or by no permutation at all. A complete
characterization of sufficient conditions for uniqueness beyond trivial
automorphism groups is, in Takemura et al.'s own framing, the most
pressing question their analysis leaves open.

The CP literature distinguishes several notions of symmetry for a CSP
instance. Cohen, Jeavons, Jefferson, Petrie and Smith~\citeyearonly{cohen2006symmetry} separate \emph{constraint symmetry}, a
permutation of variable-value pairs that preserves the constraint
relations themselves, from \emph{solution symmetry}, a permutation that
merely preserves the solution set. They prove every constraint symmetry
is a solution symmetry (their Theorem~1), though the converse can fail
badly. They exhibit a CSP with $n$ variables and $d$ values each, with a
unique solution, whose solution-symmetry group nonetheless has order
$n!\,(n(d-1))!$ (their Example~3): freely permuting the $n$ variable-value
pairs the one solution uses contributes the $n!$ factor, and freely
permuting the $n(d-1)$ pairs it leaves unused contributes the rest, since
acting within either of these two disjoint sets on its own cannot disturb
the one solution.
Definition~\ref{def:takemura-autx} is a solution symmetry of a narrow
kind. It never permutes positions. It fixes $N$ pointwise and applies
a single permutation of the shared value set $S$ identically at every
position. We call it \emph{global value symmetry}. The generalization
introduced in Section~\ref{sec:componentwise}, where each position
chooses its own value permutation independently, sits strictly between
global value symmetry and Cohen et al.'s unrestricted solution symmetry.
We call it \emph{componentwise}, or \emph{per-variable}, \emph{value
symmetry}.

None of the four rsbench benchmarks used to evaluate
Definition~\ref{def:takemura-autx} satisfies its shared-domain
hypothesis. Bortolotti et al.'s rsbench~\citeyearonly{rsbench2024} specifies
CLE4EVR concretely. Two rendered objects each carry four attributes,
color, shape, material and size, with public default domain sizes 2, 3,
2 and 3 respectively.
The eight resulting positions do not share one value set $S$. They range
over four semantically unrelated, differently sized sets. Kandinsky's
three objects are more regular: shape and color domains both have size
3.
But shape values (circle, square, triangle) and color values (red,
yellow, blue) remain two different sets, not one shared $S$. BDD-OIA and
SDD-OIA are the starkest case. Takemura et al.'s own Table~1 records
$N=21$ for both.
The public rsbench decision logic is a function of ``4 actions from 20
interrelated concepts''~\cite{rsbench2024}, and direct inspection of its
arguments confirms it takes exactly 20 free binary inputs.
The 21st recorded field, \texttt{road\_clear}, is logically determined
by four of the other twenty. An exhaustive check over all $2^{20}$
assignments confirms
$\texttt{road\_clear}\equiv\neg(\texttt{car}\lor\texttt{person}\lor
\texttt{rider}\lor\texttt{other\_obstacle})$, so it carries no
independent information.
Twenty positions, each with its own two-valued domain $\{0,1\}$, drive
the four action labels, and there is no natural sense in which a
red-light detector and a left-turn-lane detector take values from one
shared label set $S$. Definition~\ref{def:takemura-autx}'s bijective
restriction additionally requires $|N|=|S|$. For CLE4EVR ($|N|=8$ split
across four differently sized attribute types) and for BDD-OIA/SDD-OIA
($|N|=20$ free positions, $|S_i|=2$ at every position) this equality has
no natural reading. Applying Definition~\ref{def:takemura-autx} to these
benchmarks requires an extra step the definition itself does not supply:
some way of embedding heterogeneous attribute domains into one shared
$S$. Section~\ref{sec:padding-false-positive} shows this step is not
innocuous.

Takemura et al.'s own published evaluation is consistent with this
reading, not in tension with it. Their formalization fixes a single
shared domain from the start (their Definition~1: variables $N$, one
domain $S$), and their bijective results additionally force $|S|=|N|$.
Their Table~1 reports $5{,}759$ bijective shortcuts for CLE4EVR at
$N=8$, which under their own definitions presupposes an $8$-element
shared value set, while the benchmark's native typed structure carries
ten values across four semantically distinct attribute types. Some
re-encoding of typed attributes into one untyped domain was therefore
necessarily applied on their side; every encoding printed in their paper
declares a single shared \texttt{val(0..k)} domain (their toy examples),
and the specific re-encoding behind the benchmark rows is not given.
We make no claim that their published counts are incorrect under their
own encoding. The claim is about the step in between: the framework's
definition attaches to a typed benchmark only through a translation into
one shared domain, the translation is not part of the framework, and
Section~\ref{sec:padding-false-positive} shows the resulting verdict
depends materially on which translation is chosen.

\subsection{The independent-coordinate generalization}
\label{sec:componentwise}

Section~\ref{sec:takemura-framework} shows that
Definition~\ref{def:takemura-autx} needs a single shared value set $S$
that none of the four benchmarks naturally provides. We now generalize
the definition itself, rather than the benchmarks, letting every
position keep its own local domain $S_i$ and its own value permutation.

\begin{definition}[Componentwise value symmetry]
\label[definition]{def:componentwise-autx}
For $\sigma=(\sigma_1,\dots,\sigma_n)\in\Gind$ and
$\phi\in\prod_{i\in N}S_i$, write $\sigma\cdot\phi$ for the mapping with
$(\sigma\cdot\phi)(i)=\sigma_i(\phi(i))$ for every $i\in N$. The
automorphism group of $\Phic$ is the setwise stabilizer of $\Phic$ under
this action,
\[
\AutX \;=\; \big\{\, \sigma\in\Gind \;:\; \sigma\cdot\phi\in\Phic \text{
for all } \phi\in\Phic \,\big\}.
\]
\end{definition}

Every position chooses its own value permutation $\sigma_i\in\Sym(S_i)$
independently. Nothing requires $S_i=S_j$ for $i\neq j$, and nothing
requires $\sigma_i=\sigma_j$.

\begin{observation}[Componentwise symmetry contains global value
symmetry, on any common solution set]
\label[observation]{obs:diagonal-containment}
Suppose $S_i\equiv S$ for all $i\in N$. Let
$\Delta=\{(\sigma,\dots,\sigma):\sigma\in\Sym(S)\}\leq\Gind$ be the
diagonal subgroup; the map $\sigma\mapsto(\sigma,\dots,\sigma)$ is a
group isomorphism $\Sym(S)\to\Delta$, and
$(\sigma,\dots,\sigma)\cdot\phi=\sigma\circ\phi$. Consequently, for
\emph{any} fixed solution set $\Psi$, the diagonal stabilizer of $\Psi$
is exactly $\mathrm{Stab}_{\Gind}(\Psi)\cap\Delta$: read off one common
solution set, componentwise symmetry contains global value symmetry as
its diagonal part. In particular $\Aut(X)$ from
Definition~\ref{def:takemura-autx} is isomorphic to the diagonal part
of the componentwise stabilizer of $\Phic^{\mathrm{bij}}$.
\end{observation}

The qualifier ``on any common solution set'' is load-bearing. The two
definitions as published read their groups off \emph{different} sets,
$\Phic^{\mathrm{bij}}$ for Definition~\ref{def:takemura-autx} and the
general $\Phic$ for Definition~\ref{def:componentwise-autx}, and across
that divide neither group need contain the other. A three-solution
example settles it: with $N=S=\{0,1\}$ and
$\Phic=\{(0,0),(0,1),(1,0)\}$ (the rule $\neg(x_1\wedge x_2)$), the
value swap preserves $\Phic^{\mathrm{bij}}=\{(0,1),(1,0)\}$, so it lies
in $\Aut(X)$, yet it maps $(0,0)$ to $(1,1)\notin\Phic$, so its
diagonal copy is not in $\AutX$. The bijective restriction is not a
cosmetic convention: it changes the group. Computationally verified;
see \Cref{app:verification}.
Observation~\ref{obs:diagonal-containment} only applies on homogeneous
instances, where $S_i\equiv S$ for every position. On instances whose
domains differ in size or content, no diagonal subgroup exists and
Definition~\ref{def:takemura-autx} simply has nothing to say; CLE4EVR
and Kandinsky are heterogeneous in this sense
(Section~\ref{sec:takemura-framework}). BDD-OIA and SDD-OIA formally
share the two-element domain $\{0,1\}$, so their diagonal subgroup
exists, but the definition's second hypothesis, $|N|=|S|$, fails at
$20$ against $2$, and the shared set identifies the labels of
semantically unrelated detectors. On all four benchmarks the
definition's stated hypotheses fail; componentwise value symmetry is
the only one of the two definitions whose hypotheses hold on them.

Componentwise is not the only well-defined reading, however, and we do
not claim it is uniquely correct. Between one permutation for all
positions and one permutation per position sits a natural intermediate.

\begin{definition}[Typed value symmetry]
\label[definition]{def:typed-autx}
Let $\tau:N\to T$ assign each position a semantic type, with a shared
domain $S_t$ for all positions of type $t$ (so $S_i=S_{\tau(i)}$). The
typed group is $G_\tau:=\prod_{t\in T}\Sym(S_t)$, acting by
$(\sigma\cdot\phi)(i):=\sigma_{\tau(i)}(\phi(i))$, and the typed
automorphism group $\Aut_\tau(\Phic)$ is its setwise stabilizer of
$\Phic$.
\end{definition}

Taking every position as its own type recovers
Definition~\ref{def:componentwise-autx}; taking one type for all
positions recovers the diagonal action of a single shared permutation
on $\Phic$, which coincides with Definition~\ref{def:takemura-autx}
only when both groups are read off one common solution set
(Observation~\ref{obs:diagonal-containment}; the three-solution example
above shows the published bijective restriction can differ). Typed value symmetry is well
defined on all four benchmarks, with types given by the attribute kinds
(CLE4EVR: color, shape, material, size, each shared by the two
objects). Which rung of this hierarchy counts two shortcuts as ``the
same failure'' is a modeling choice with an architectural reading, to
which we return in \Cref{sec:realmodels}: a perception network shared
across same-type positions admits exactly the typed output-relabeling
reparameterizations, while independent per-position heads admit
componentwise ones. This paper measures exactly where the choice
changes the answer rather than assuming it. On
CLE4EVR's six constrained dimensions the typed and componentwise groups
coincide exactly, element for element, and so do their orbits
($|\mathrm{Aut}|=24$, transitive, $0\%$); on Kandinsky they coincide as
well ($|\mathrm{Aut}|=36$, the same six orbits, $81.99\%$). The two
readings genuinely part in two situations: on entirely unconstrained
positions, as in CLE4EVR's free size slots
(Table~\ref{tab:clevr-comparison}: componentwise stays transitive,
typed splits equal-size from unequal-size pairs, $44.86\%$), and on
\emph{equality patterns} inside constrained levels, where the diagonal
action preserves whether two same-type positions agree while
independent permutations do not. \Cref{sec:realmodels} meets thirteen
levels of the second kind and turns the divergence into this paper's
sharpest experiment, a three-way separation of what symmetry permits,
what architecture realizes, and what optimization selects. On the
heterogeneous CLE4EVR configurations of
Table~\ref{tab:clevr-comparison}, the padded reading agrees with no
rung of the hierarchy; on the measured all-different grids,
pair-richness forces every componentwise automorphism to be diagonal
(\Cref{lem:diagonal-collapse}), and the two readings coincide exactly
(\Cref{sec:empirical-extension}).
Throughout, componentwise is the default instrument, the coarsest-orbit
member (most generous to the symmetry hypothesis) requiring no typing
choice, and every place the rung changes an answer is flagged where it
occurs.

One property separates the well-defined rungs from the padded reading
before any measurement is taken, and it is worth stating as a fact
rather than a preference.

\begin{proposition}[Representation equivariance]
\label[proposition]{prop:naturality}
Let $\eta=(\eta_1,\dots,\eta_n)$ with each $\eta_i$ a bijection of
$S_i$ (an independent relabeling of each native domain), acting on
solution sets by $\eta\cdot\phi=(\eta_i(\phi_i))_i$. Then
\[
\Aut(\eta\,\Phic)\;=\;\eta\,\AutX\,\eta^{-1},
\]
and the same holds for the typed group when $\eta$ assigns equal
relabelings to positions of equal type. Consequently group order, orbit
count, orbit sizes, transitivity, and $\rho$ are invariant under how a
configuration file happens to name each domain's values.
\end{proposition}

\begin{proof}
$\sigma$ stabilizes $\Phic$ iff $\eta\sigma\eta^{-1}$ stabilizes
$\eta\,\Phic$, a one-line conjugation computation; conjugation by
$\eta$ is a group isomorphism carrying orbits to orbits bijectively.
For the typed group, conjugating a type-constant tuple by a
type-constant $\eta$ yields a type-constant tuple.
\end{proof}

Computationally verified; see \Cref{app:verification}. The padded
diagonal extension has no such property, and
Section~\ref{sec:padding-false-positive} does not merely claim this: it
measures the violation, exhibiting four independent relabelings of one
native domain that leave the padded group's order fixed while rotating
which values it declares interchangeable.

Componentwise value symmetry still sits inside Cohen et al.'s general
solution symmetry~\citeyearonly{cohen2006symmetry}. It fixes positions
pointwise and factors through the product $\Gind$, so it cannot reach
the unstructured pair permutations behind their $n!\,(n(d-1))!$
blow-up. It remains sensitive to a milder version of the same
underlying issue. A position that no constraint in $C$ actually
restricts contributes its full $\Sym(S_i)$ to $\AutX$ regardless of the
rest of the instance, inflating the group with automorphisms that
reflect an absent constraint rather than a genuine symmetry of the
rule. Section~\ref{sec:padding-false-positive} encounters exactly this
pattern in CLE4EVR's unconstrained size attributes, and we return to
recognizing such positions systematically, as a decision problem in its
own right, in \Cref{sec:complexity}.

\begin{observation}[Boolean domains give linear structure]
\label[observation]{obs:boolean-linear}
Suppose $S_i=\{0,1\}$ for every $i\in N$. Then $\Sym(S_i)$ has order 2,
generated by the bit flip, so $\Sym(S_i)\cong\Ftwo$ as a group.
Consequently $\Gind\cong\Ftwo^{n}$, coordinatewise. Since
$\AutX\leq\Gind$ is a subgroup, and every subgroup of an $\Ftwo$-vector
space is automatically an $\Ftwo$-linear subspace, $\AutX$ is an
$\Ftwo$-linear subspace of $\Ftwo^n$.
\end{observation}

Observation~\ref{obs:boolean-linear} is a statement about the ambient
group. It does not say which subspace actually arises for a given rule.
BDD-OIA and SDD-OIA are Boolean in exactly this sense
(Section~\ref{sec:takemura-framework}), and the algebraic
characterization in \Cref{sec:algebra} determines which subspace $\AutX$
occupies once the rule structure is taken into account.

\subsection{The padded extension: a measured false positive}
\label{sec:padding-false-positive}

The generalization in Section~\ref{sec:componentwise} is only useful if
it changes the answer on real rules. This section measures that on
CLE4EVR, using rsbench's public default configuration
\cite{rsbench2024}, and finds that it does. The two definitions do not
just differ in generality. On the same data, they disagree about
whether a pathology exists at all.

Definition~\ref{def:takemura-autx} cannot be evaluated on CLE4EVR
directly. Its eight positions split across four differently sized
attribute domains (Section~\ref{sec:takemura-framework}), so there is no
shared $S$ and no bijective $\Phic^{\mathrm{bij}}$. The most direct way
to force the definition to apply, and the one a practitioner reaching
for Definition~\ref{def:takemura-autx} off the shelf would likely try
first, is to pad every position's domain to a common size
$M=\max_i|S_i|$, using the order in which values are listed in the
source configuration file as an arbitrary numeric labeling
$\{0,\dots,|S_i|-1\}$. The search then looks for a single permutation
$\sigma\in\Sym(\{0,\dots,M-1\})$, applied identically, by numeric
label, at every position. Padding is not the only such embedding:
tagging each value with its attribute type gives a disjoint-union
domain whose $\Phic$-stabilizing permutations act exactly as the typed
group of \Cref{def:typed-autx}, already measured as
\Cref{tab:clevr-comparison}'s typed rows; what the comparison below
isolates is what the unspecified translation step decides.

One thing must be pinned down before any number is reported: the padded
object is \emph{not} Definition~\ref{def:takemura-autx} verbatim, and
cannot be. Padding equalizes domain sizes, but it does not restore the
definition's bijective hypothesis: CLE4EVR has $|N|=6$ or $8$ against
$|S|=M=3$, so $\Phic^{\mathrm{bij}}$ is empty after padding, and the
literal Definition~\ref{def:takemura-autx} evaluates to the full group
$\Sym(M)$ stabilizing an empty solution set, a verdict about nothing.
The informative object, the one a practitioner in fact computes, is the
definition's action extended to the general solution set, the
\emph{padded diagonal extension}
\[
\Aut_{\mathrm{pad}} \;:=\; \big\{\,\sigma\in\Sym(\{0,\dots,M-1\}) \;:\;
\sigma\circ\phi\in\Phic \text{ for all } \phi\in\Phic \,\big\}.
\]
Every padded figure in this paper is $\Aut_{\mathrm{pad}}$, labeled as
the extension it is; that the port must already drop the bijective
restriction to produce any verdict at all is one more sense in which
the translation step is not innocuous. The comparison below is
therefore between the two objects one can actually evaluate on the
benchmark, $\Aut_{\mathrm{pad}}$ and the componentwise $\AutX$
(Definition~\ref{def:componentwise-autx}), on the same real rule and
the same real solution set.

CLE4EVR pairs two rendered objects, each with four attributes: color,
shape, material and size, with domain sizes 2, 3, 2 and 3 respectively.
The public default rule requires the two objects to match on the first
three attributes and leaves size unconstrained,
\[
C:\quad \mathrm{color}_1=\mathrm{color}_2 \ \wedge\
\mathrm{shape}_1=\mathrm{shape}_2 \ \wedge\
\mathrm{material}_1=\mathrm{material}_2 .
\]
Restricting to the six positions the rule actually mentions
(\texttt{color\_1}, \texttt{shape\_1}, \texttt{material\_1},
\texttt{color\_2}, \texttt{shape\_2}, \texttt{material\_2}), the raw
space has $2\cdot3\cdot2\cdot2\cdot3\cdot2=144$ tuples, of which
$|\Phic|=12$
satisfy $C$: the diagonal set of matching (color, shape, material)
triples, $2\cdot3\cdot2=12$.
Including the two free size coordinates gives a second, larger instance
with $|\Phic|=108=12\times 9$, confirming that the free coordinates
contribute exactly a multiplicative factor of $3\times3$ and nothing
else.

This is a different, well-defined quantity from Takemura et al.'s own
reported count of 5{,}759 bijective shortcuts for CLE4EVR. Their
Appendix~G states only that all 8 concepts are preserved and that 4
synthetic samples encode the classification rule, without specifying
how four differently sized attribute domains are packed into one
bijective label space.
This encoding is not reconstructible from the public rsbench
configuration. rsbench's own paper is not fully consistent about
CLE4EVR's shape and color domain sizes either, reporting 10 shapes and
10 colors in one table and nine shapes with eight predefined colors in
its appendix, against the repository default of 3 shapes and 2 colors
used throughout this paper.
We therefore measure the general solution set $\Phic$
(Definition~\ref{def:general-csp}) directly, on the real public rule and
domains. This is the reading that is actually well defined for this
data.

\begin{table}[htbp]
\centering
\small
\caption{Componentwise value symmetry (Definition~\ref{def:componentwise-autx})
and typed value symmetry (Definition~\ref{def:typed-autx})
against the padded diagonal extension $\Aut_{\mathrm{pad}}$ of
Takemura et al.'s global value symmetry
(Section~\ref{sec:padding-false-positive}; the literal
Definition~\ref{def:takemura-autx} degenerates after padding, since no
bijection $N\to S$ exists), on the same real rule and solution set. ``6-dim''
restricts to the six positions the rule mentions (the primary measurement);
``8-dim'' includes the two unconstrained size positions. Pairwise\% is the
fraction of unordered solution pairs $\{\phi,\phi'\}\subseteq\Phic$ that lie in
different orbits of the corresponding group. On the constrained six
dimensions the two well-defined readings agree exactly, as groups and
orbit by orbit; they part only on the free size slots, and the padded
verdict agrees with neither on either instance.}
\label{tab:clevr-comparison}
\begin{tabular}{llrrrrcr}
\toprule
Definition & Dim.\ & $|\Phic|$ & $|\mathrm{Aut}|$ & \%\,of full group & orbits & transitive & pairwise\% \\
\midrule
Componentwise & 6 & 12 & 24 & 4.17\% & 1 & yes & 0.00\% \\
Componentwise & 8 & 108 & 864 & 4.17\% & 1 & yes & 0.00\% \\
Typed & 6 & 12 & 24 & 100\% & 1 & yes & 0.00\% \\
Typed & 8 & 108 & 144 & 100\% & 2 & no & 44.86\% \\
Padded global & 6 & 12 & 2 & 33.33\% & 6 & no & 90.91\% \\
Padded global & 8 & 108 & 2 & 33.33\% & 54 & no & 99.07\% \\
\bottomrule
\end{tabular}
\end{table}

\Cref{fig:padding-false-positive} visualizes this same contrast directly,
including how the padded definition's choice of which shape counts as
``fixed'' rotates with attribute list order alone.

\begin{figure}[htbp]
\centering
%
%
\definecolor{figblue}{HTML}{0072B2}
\definecolor{figverm}{HTML}{D55E00}
\tikzset{
  fpx/.style={
    circle, draw=figblack, fill=figdark, minimum size=2.2mm, inner sep=0pt
  },
  fpbigorbit/.style={
    draw=figblue, thick, rounded corners=4mm, fill=figblue!8, inner sep=0.24cm
  },
  fpsmallorbit/.style={
    draw=figverm, thick, rounded corners=2.2mm, fill=figverm!10, inner sep=0.115cm
  },
  fpinvis/.style={draw=none, fill=none, inner sep=0.24cm},
  fphead/.style={font=\small\bfseries, align=center},
  fpstat/.style={font=\footnotesize, align=center, inner sep=0.06cm},
  fpmid/.style={font=\footnotesize\itshape, text=figdark, align=center},
  fprotbox/.style={
    draw=figdark, rounded corners=1.5pt, fill=white, align=center,
    inner sep=0.10cm, font=\scriptsize, minimum width=2.28cm, minimum height=0.72cm
  },
  fprotlab/.style={font=\footnotesize\itshape, text=figdark, align=right},
  fprotarrow/.style={-{Stealth[length=1.5mm]}, draw=figmid, semithick}
}

\begin{tikzpicture}[x=1cm, y=1cm]

  \foreach \c in {0,...,5} {
    \node[fpx] (Lt\c) at (\c*0.85, 0.38) {};
    \node[fpx] (Lb\c) at (\c*0.85, -0.38) {};
  }
  \begin{scope}[on background layer]
    \node[fpbigorbit, fit=(Lt0)(Lt5)(Lb0)(Lb5)] (Lenv) {};
  \end{scope}

  \node[fphead, above=0.16cm of Lenv.north] (Lhead)
    {Componentwise (Definition~\ref{def:componentwise-autx})};
  \node[fpstat, below=0.16cm of Lenv.south] (Lstat)
    {$|\AutX|=24$, \, 1 orbit\\[1pt] \textcolor{figblue!75!black}{\textbf{0.00\% unexplained}}};

  \begin{scope}[xshift=6.85cm]
    \foreach \c in {0,...,5} {
      \node[fpx] (Rt\c) at (\c*0.85, 0.38) {};
      \node[fpx] (Rb\c) at (\c*0.85, -0.38) {};
    }
    \node[fpinvis, fit=(Rt0)(Rt5)(Rb0)(Rb5)] (Renv) {};
  \end{scope}
  \begin{scope}[on background layer]
    \foreach \c in {0,...,5} {
      \node[fpsmallorbit, fit=(Rt\c)(Rb\c)] {};
    }
  \end{scope}

  \node[fphead, above=0.16cm of Renv.north] (Rhead)
    {Padded extension of Definition~\ref{def:takemura-autx}};
  \node[fpstat, below=0.16cm of Renv.south] (Rstat)
    {$|\Aut_{\mathrm{pad}}|=2$, \, 6 orbits\\[1pt] \textcolor{figverm!80!black}{\textbf{90.91\% unexplained}}};

  \node[fpmid] at ($(Lenv.east)!0.5!(Renv.west)$) {same 12\\solutions};

  \coordinate (stripmid) at ($(Lstat.south)!0.5!(Rstat.south)$);
  \node[fprotbox, below=0.42cm of stripmid, xshift=-2.72cm] (ord2)
    {(toroid, star, cube)\\[1.5pt] fixed: \textcolor{figverm!80!black}{\textbf{cube}}};
  \node[fprotbox, left=0.42cm of ord2] (ord1)
    {(cube, star, toroid)\\[1.5pt] fixed: \textcolor{figverm!80!black}{\textbf{toroid}}};
  \node[fprotbox, right=0.42cm of ord2] (ord3)
    {(star, cube, toroid)\\[1.5pt] fixed: \textcolor{figverm!80!black}{\textbf{toroid}}};
  \node[fprotbox, right=0.42cm of ord3] (ord4)
    {(toroid, cube, star)\\[1.5pt] fixed: \textcolor{figverm!80!black}{\textbf{star}}};
  \node[fprotlab, left=0.30cm of ord1] {config\\order:};

  \draw[fprotarrow] (ord1.east) -- (ord2.west);
  \draw[fprotarrow] (ord2.east) -- (ord3.west);
  \draw[fprotarrow] (ord3.east) -- (ord4.west);

\end{tikzpicture}
\caption{The same $|\Phic|=12$ CLE4EVR solutions from
Table~\ref{tab:clevr-comparison}'s 6-dim row, under the componentwise
definition (Definition~\ref{def:componentwise-autx}) versus the padded
diagonal extension $\Aut_{\mathrm{pad}}$ of Takemura et al.'s global
definition (the literal Definition~\ref{def:takemura-autx} has no
bijective solutions to act on after padding). Left: $|\AutX|=24$, a single orbit,
$0.00\%$ of solution pairs unexplained. Right: $|\Aut_{\mathrm{pad}}|=2$,
the same 12 solutions split into 6 orbits of size 2, $90.91\%$ of pairs
unexplained. Bottom: re-running the padded search after reordering only
the shape attribute's listing in the source config file leaves the group
size and orbit count unchanged but changes which shape is left unpaired
(toroid, then cube, then toroid, then star across four orderings)---the
padded pathology's \emph{content} is an artifact of list order, not of
rule structure.}
\label{fig:padding-false-positive}
\end{figure}
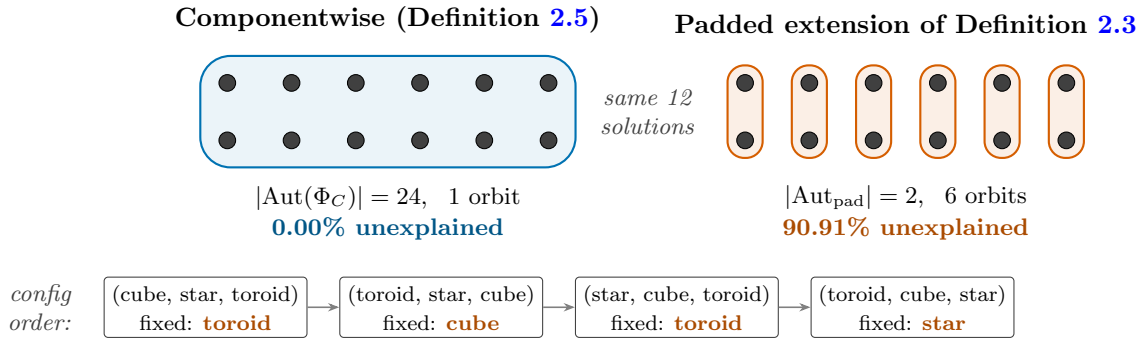

The padded definition's headline number, $90.91\%$ of solution pairs
unexplained by any automorphism ($60$ of $\binom{12}{2}=66$ pairs, since
6 orbits of size 2 each account for 6 same-orbit pairs),
looks like a serious pathology. Decoding the one non-trivial padded
automorphism back into real attribute values shows what it actually is.
Under the natural ordering, the order in which colors, shapes and
materials are listed in the public config file, the non-trivial element simultaneously
maps $\mathrm{red}\leftrightarrow\mathrm{blue}$,
$\mathrm{rubber}\leftrightarrow\mathrm{metal}$, and swaps two of the
three shapes while fixing the third
($\mathrm{cube}\leftrightarrow\mathrm{star}$, toroid fixed).
Padding forces three semantically unrelated attributes, color, material
and shape, to share one abstract label space, so the search for a
single global permutation finds one only by inventing cross-attribute
correspondences the rule never states. Which shape ends up the odd one
out is decided entirely by list order, not by the rule. Re-running the
same search after permuting only the order in which shapes are listed
changes which shape is fixed, while the group's size stays exactly 2
throughout:
\begin{itemize}
\item shape order \texttt{(cube, star, toroid)}, the order in the
public config file: toroid is fixed;
\item reversed, \texttt{(toroid, star, cube)}: cube is fixed;
\item \texttt{(star, cube, toroid)}: toroid is fixed;
\item \texttt{(toroid, cube, star)}: star is fixed.
\end{itemize}
The group's order is stable under these relabelings. Three further
reorderings, reversing the shape list, reversing the color list, and
swapping the color list, all leave $|\mathrm{Aut}_{\mathrm{pad}}|=2$, 6
orbits, and the $90.91\%$ pairwise figure unchanged.
So the magnitude of the padded pathology is not noise. Its content,
which values it identifies as interchangeable, is bookkeeping, not
structure. Nothing in the rule distinguishes toroid from cube from star.
Padding manufactures a distinction anyway, and hands out a different one
depending on how the config file happens to list its attributes.

The componentwise definition disagrees with this completely. On the
same rule and the same solution set, $\AutX$ is transitive: all 12
solutions form a single orbit, and $0\%$ of pairs are unexplained.
This is not a matter of the componentwise definition being more
permissive and therefore finding more structure by construction. The
separating property is Proposition~\ref{prop:naturality}: the
componentwise verdict is equivariant under independent relabelings of
each native domain, so it is a property of the rule, invariant under
how the configuration file names values, while the padded verdict's
content rotates under exactly such relabelings (the four orderings
above, group order fixed at $2$ throughout, identity of the
``interchangeable'' values changing every time). The disagreement is
not a rounding difference. One definition reports a rule with no unexplained
shortcuts. The other reports a rule that is $90.91\%$ pathological. Both
cannot be describing the same underlying structure well.
Observation~\ref{obs:diagonal-containment} explains why: on a
heterogeneous instance like CLE4EVR, Definition~\ref{def:takemura-autx}
is not merely coarser than Definition~\ref{def:componentwise-autx}. It
is undefined. Padding does not fix that. It silently substitutes a
different, config-order-dependent question and answers that one
instead.

We take this as the paper's central methodological finding.
Generalizing Definition~\ref{def:takemura-autx} to heterogeneous domains
is not a matter of elegance or coverage. Applying the original, global
definition to a heterogeneous benchmark by the most direct means
available does not merely lose precision. It produces a confident,
wrong answer, and it produces a different wrong answer depending on an
arbitrary choice, attribute list order, that carries no information
about the rule under analysis.

\begin{remark}
This is not specific to CLE4EVR. A synthetic sweep over 5{,}433
randomly generated heterogeneous constraint sets, spanning three
domain-size profiles ($[2,3,4]$, $[2,2,3,4]$ and $[2,3,3]$) and a range
of solution-set sizes, finds the padded definition collapses to the
trivial group in $97.00\%$ of instances, against $61.72\%$ for the
componentwise definition.
A definition that reports no structure on heterogeneous instances
nearly all of the time is not a safe default either. \Cref{sec:empirical}
extends the comparison in Table~\ref{tab:clevr-comparison} to Kandinsky,
BDD-OIA and SDD-OIA, using the componentwise definition throughout, once
this section has established why the padded one cannot be trusted.
\end{remark}

\section{Empirical Measurements: Structural Determinants of Pathology}
\label{sec:empirical}

\Cref{sec:background} establishes the two competing readings of value symmetry used throughout this paper: Takemura et al.'s global value symmetry $\Aut(X)$ (Definition~\ref{def:takemura-autx}), a single permutation applied identically at every position, and the componentwise generalization $\AutX$ (Definition~\ref{def:componentwise-autx}) this paper argues for, in which every position chooses its own value permutation. \Cref{sec:background} also shows the two disagree in kind, not just in generality: on CLE4EVR's real rule, forcing Definition~\ref{def:takemura-autx} to apply by padding every attribute domain to a common size reports $90.91\%$ ($6$-dimensional) to $99.07\%$ ($8$-dimensional) spurious pairwise pathology (\Cref{tab:clevr-comparison}), traced there to attribute-list bookkeeping rather than rule structure, where $\AutX$ finds none.
This section asks the constructive question that finding leaves open: once the padding artifact is removed, how often does $\AutX$-orbit membership actually account for the alternative solutions in $\Phic$, on real neurosymbolic rule sets? We extend Table~\ref{tab:clevr-comparison} to two further published benchmark families (\Cref{sec:empirical-three}), stress-test the resulting pattern on eight further rule families spanning 15 predictions fixed before measurement (\Cref{sec:empirical-extension}), and distill the result into a four-mechanism characterization (\Cref{sec:empirical-mechanisms}). Every reported number is machine-computed; LaTeX comments cite the source JSON file and field in the project's verification-data directory.

\subsection{Setup and measurement protocol}
\label{sec:empirical-setup}

\paragraph{The object of measurement.}
We use $\Phic$, $\Gind=\prod_{i=1}^n\Sym(S_i)$, and $\AutX$ exactly as fixed in Definitions~\ref{def:general-csp} and~\ref{def:componentwise-autx}: every position $i\in N$ keeps its own local domain $S_i$, and
\[
\AutX \;=\; \bigl\{\, \sigma = (\sigma_1,\dots,\sigma_n) \in \Gind \;:\; \sigma\cdot\varphi \in \Phic \text{ for all } \varphi \in \Phic \,\bigr\} \;\le\; \Gind
\]
is the coordinatewise value permutations that fix $\Phic$ setwise. The padded diagonal extension $\Aut_{\mathrm{pad}}$ of Definition~\ref{def:takemura-autx}'s action (\Cref{sec:background}'s notation and object, reused here) reappears twice below, for BDD-OIA/SDD-OIA and for the Latin-square/Sudoku family, where the two readings behave differently for reasons specific to those rules.

For $\varphi,\varphi' \in \Phic$ write $\varphi \sim \varphi'$ when they lie in the same $\AutX$-orbit. \Cref{sec:background} informally reports the fraction of unordered solution pairs not related by any automorphism as ``pairwise\%,'' in the caption of \Cref{tab:clevr-comparison}; we use it as the headline statistic throughout this paper and fix notation for it here, for $|\Phic|\geq2$,
\[
\rho(\Phic) \;=\; \frac{\bigl|\{\{\varphi,\varphi'\} \subseteq \Phic : \varphi \neq \varphi',\ \varphi \not\sim \varphi'\}\bigr|}{\binom{|\Phic|}{2}} \;\in\; [0,1],
\]
with $\rho = 0$ exactly when $\AutX$ acts transitively on $\Phic$ (a single orbit accounts for every solution pair) and $\rho \to 1$ when almost no pair of solutions is related by any automorphism. A companion \emph{existence} statistic, $\rho_\exists(\Phic)$, is the fraction of $\varphi \in \Phic$ with \emph{some} orbit-external partner elsewhere in $\Phic$; because a single non-trivial orbit already pulls $\rho_\exists$ close to $1$, it is close to a step function (0 iff transitive, close to 1 otherwise) and adds little beyond transitivity itself. We report it only in passing and use $\rho$ as the headline statistic throughout.

$\rho$ has an exact reading in terms of the intended mapping. Define,
again for $|\Phic|\geq2$,
the \emph{orbit coverage} of a candidate intended solution as
$\kappa(\varphi^*) := (|\mathrm{Orb}(\varphi^*)|-1)/(|\Phic|-1)$, the
fraction of alternative solutions that symmetry generates from
$\varphi^*$; then $1-\rho = \mathbb{E}[\kappa(\varphi^*)]$ under a
uniform draw of $\varphi^*$ from $\Phic$, since
$1-\rho=\sum_j\frac{|O_j|}{|\Phic|}\cdot\frac{|O_j|-1}{|\Phic|-1}$
over the orbits $O_j$. So $\rho$ is the expected unexplained fraction
over intended mappings, and conditioning on a specific $\varphi^*$ can
only sharpen it: on Kandinsky, $\kappa$ is $35/161=21.7\%$ when
$\varphi^*$ lies in a large orbit and $17/161=10.6\%$ in a small one,
so no single intended mapping has more than $21.7\%$ of its
alternatives explained, a statement stronger than the pooled
$\rho=81.99\%$. In the Boolean case the free action of
\Cref{prop:free-action} makes $\kappa$ constant across $\varphi^*$ and
the two views coincide.

\paragraph{Toolchain and cross-validation.}
Automorphism groups and orbits are computed two ways. For small instances, \texttt{real\_data\_measurement.py} enumerates $\Gind$ directly and tests setwise stabilization by brute force. When $\Gind$ is too large to enumerate (e.g.\ $(10!)^2 \approx 1.3\times10^{13}$ for two free MNIST digits), \texttt{empirical\_ext\_nauty\_indep.py} reduces the independent-coordinate automorphism problem to a colored-graph automorphism problem solved with \texttt{pynauty} \cite{mckay2014nauty}, giving each coordinate's value domain its own color class. The two engines were cross-checked against each other on CLE4EVR ($|\AutX|=24$, one orbit) and Kandinsky ($|\AutX|=36$, six orbits of sizes $36,36,36,18,18,18$): they agree not only on group order and orbit-size multiset but on the exact partition of $\Phic$ into orbits.
Group-axiom checks (closure under composition; every element's inverse present) and a second, independent orbit algorithm (direct image-set computation rather than union-find) were additionally run on the Kandinsky result, since it is this paper's first non-degenerate positive finding and the one result most exposed if the enumeration code itself had a bug.

\paragraph{Label-driven reformulation (BDD-OIA/SDD-OIA).}
BDD-OIA and SDD-OIA share one decision function, \texttt{sddoiaK}, mapping 20 free binary concepts to 4 binary actions (\Cref{sec:empirical-three}). \Cref{sec:background} shows Definition~\ref{def:takemura-autx} cannot even be posed here: the detectors share the truth-value domain $\{0,1\}$ but not a concept-label space, and Definition~\ref{def:takemura-autx} needs $S$ to serve as concept labels with $|N|=|S|$: twenty positions over a two-element set admit no bijection, so its bijective restriction $\Phic^{\mathrm{bij}}$ is empty and $\mathrm{Aut}(X)$ collapses to the full, uninformative $\Sym(S)$ under the universally-quantified-over-the-empty-set reading of that definition. Definition~\ref{def:componentwise-autx} does not need a shared $S$, but it does need a fixed intended mapping $\varphi^*$ to define shortcuts relative to, and BDD-OIA/SDD-OIA is naturally posed the other way around, as a decision rule that maps concepts to an observed action. We therefore reformulate around a fixed \emph{label} rather than a fixed mapping: for a reachable action vector $y^*$, define
\[
\PhiY{y^*} \;=\; \bigl\{\, c \in \{0,1\}^{20} \;:\; \texttt{sddoiaK}(c) = y^* \,\bigr\},
\]
directly specializing Definition~\ref{def:general-csp}'s $\Phic$ (which never requires bijectivity) to the single functional constraint ``$\texttt{sddoiaK}(\cdot)=y^*$,'' syntactically on the same footing as Takemura et al.'s own arithmetic-equality examples. $\Aut(\PhiY{y^*})$ is then defined exactly as $\AutX$ above, under $\Gind = \Sym(\{0,1\})^{20} \cong (\Ftwo)^{20}$.

The reformulation is a direction-dual of Takemura et al.'s framework, and the two directions are worth setting side by side. Definition~\ref{def:general-csp} (Takemura et al.'s Definitions~3/4) asks, for a fixed intended mapping $\varphi^*$, which \emph{other} concept assignments satisfy the same constraints: shortcuts relative to a ground-truth \emph{mapping}. $\PhiY{y^*}$ asks, for a fixed observed \emph{label}, which concept assignments produce it: shortcuts relative to a ground-truth \emph{label}. Both are specializations of the same general $\Phic = \{\varphi: N\to S \mid \varphi \text{ satisfies } C\}$, and the label-driven direction is the one BDD-OIA's own structure poses, a decision rule mapping concepts to an observed action; it is also the direction under which rsbench's flagship failure mode, the pedestrian confused for a red light, becomes a measurable orbit fact rather than an anecdote (\Cref{sec:empirical-three}).

\paragraph{The object ladder.}
The two directions, and the per-instance classification of
\Cref{sec:realmodels}, occupy exactly one ladder, and one identity
connects its rungs. For a rule $R:\prod_iS_i\to Y$ and any group
$H\leq\Gind$, write $\Aut_H(R):=\{\sigma\in H : R(\sigma\cdot
c)=R(c)\ \forall c\}$ for the \emph{global} symmetries, the relabelings
a system could apply uniformly everywhere, and
$\Aut_H(\PhiY{y})$ for each label's \emph{fiber} symmetries.

\begin{proposition}[Fiber intersection]
\label[proposition]{prop:fiber-intersection}
$\Aut_H(R)=\bigcap_{y\in\mathrm{im}(R)}\Aut_H(\PhiY{y})$. In
particular, for binary-output rules the two rungs coincide:
stabilizing $\Phic$ inside $\prod_iS_i$ stabilizes its complement, so
$\Aut_H(\Phic)=\Aut_H(R)$ exactly.
\end{proposition}

\begin{proof}
$R\circ\sigma=R$ holds iff $\sigma$ maps every fiber into itself, and
a bijection mapping each block of a partition into itself maps each
onto itself. For binary output the fibers are $\Phic$ and its
complement, and a bijection stabilizing a set setwise stabilizes its
complement.
\end{proof}

The identity settles which rung each measurement in this paper lives
on, and why. CLE4EVR and Kandinsky are binary-output rules, so their
fiber and global groups are one object and no gap exists. For
\texttt{sddoiaK}, the eight fiber groups all equal
$\{\mathrm{id},\mathrm{flip}_{\texttt{follow}}\}$, so the intersection,
the global group, is that same order-2 group: BDD-OIA's dead variable
is a full global grounding symmetry, symmetry of the strongest kind.
For the arithmetic rules the ladder separates: the global groups of
\texttt{sum} and \texttt{product} are provably trivial (a shift
$\sigma_1(a)=a+k$ cannot permute $\{0,\dots,9\}$ unless $k=0$;
$\sigma_1(a)\sigma_2(b)=ab$ at $b=1$ forces integrality and hence
identity), while individual fibers carry the nontrivial structure
\Cref{sec:empirical-extension} measures. This is not a mismatch to
apologize for; it is the finding. On arithmetic tasks the
global-relabeling story is vacuous by theorem, the fiber level is where
identifiable structure lives, and per-instance errors, each of which
satisfies exactly one fiber's constraint, are classified against the
finest rung that is sound for them. One rung sits above the rule
itself: a fixed system-wide grounding shortcut in Takemura et al.'s
original sense is a relabeling $\sigma$ with $\sigma\circ\gamma^*$
valid on every training pair, and when the data's concept support
covers the domain, $\sigma\mapsto\sigma\circ\gamma^*$ identifies such
groundings exactly with $\Aut_H(R)$, the top of this ladder. For the
arithmetic rules that top rung is provably empty of nonidentity
elements, so instance-level errors are the only form value-relabeling
ambiguity can take there, which is precisely what
\Cref{sec:realmodels} goes on to measure. Computationally verified; see
\Cref{app:verification}.

\subsection{Three benchmark families}
\label{sec:empirical-three}

\paragraph{CLE4EVR.}
\Cref{sec:background} already reports CLE4EVR in full (\Cref{tab:clevr-comparison}): on rsbench's real rule \cite{rsbench2024}, $\AutX$ is fully transitive at both $6$ dimensions ($|\Phic|=12$, $|\AutX|=24$) and $8$ dimensions ($|\Phic|=108$, $|\AutX|=864$), $\rho=0\%$ throughout, while the padded global reading reports $90.91\%$--$99.07\%$ pathology traced to an arbitrary attribute-list ordering rather than to the rule. CLE4EVR is this paper's cleanest transitive case, and the reference point the next two families are measured against.

\paragraph{Kandinsky.}
All 6 symbols are constrained: three objects, each with a shape (circle, square, or triangle) and a color (red, yellow, or blue), under
\[
\bigl(c_1{=}c_2 \wedge s_1{=}s_2 \wedge s_1{\neq}s_3\bigr) \;\vee\; \bigl(c_1{=}c_3 \wedge s_1{=}s_3 \wedge s_1{\neq}s_2\bigr) \;\vee\; \bigl(c_2{=}c_3 \wedge s_2{=}s_3 \wedge s_1{\neq}s_3\bigr),
\]
writing $c_i,s_i$ for object $i$'s color and shape: exactly one pair of the three objects matches on both attributes while differing from the third. Exhaustive enumeration gives $|\Phic|=162$ ($3\times54$, matching an independent hand combinatorial count of the three mutually exclusive disjuncts), $|\AutX|=36$, and \emph{six} orbits of sizes $36,36,36,18,18,18$: $\rho = 81.99\%$ ($10{,}692/13{,}041$ pairs).
This is this paper's first non-degenerate positive result, and it is not a weak-definition artifact: even under the independent-coordinate group (the more defensible reading of value symmetry argued for in \Cref{sec:background}), $100\%$ of candidate solutions have at least one orbit-external shortcut.
The mechanism is structural. No automorphism relabels a solution in which objects 1 and 2 match into one where objects 1 and 3 match: \Cref{sec:algebra} proves that every element of $\AutX$ preserves each of the three branches setwise, so the disjunction's branches are three orbit-separated position patterns. \Cref{sec:empirical-extension} sharpens this into a general branch-exchange criterion.

\paragraph{BDD-OIA and SDD-OIA.}
Both benchmarks share \texttt{sddoiaK}, extracted directly from the rsbench source and independently reproducing rsbench's own Table~28 worked example exactly \cite{rsbench2024}.
\Cref{sec:background} already establishes that \texttt{sddoiaK} takes exactly 20 free boolean parameters, not 21, since the 21st tabulated concept, \texttt{clear}, is identically $\neg(\texttt{car}\vee\texttt{person}\vee\texttt{rider}\vee\texttt{other\_obstacle})$, double-locked there by \texttt{inspect.signature} on the live function and by an exhaustive $2^{20}$ check, and independently corroborated by rsbench's own text \cite{rsbench2024}. Completing the measurement needs one further structural fact \Cref{sec:background} does not use: \texttt{move\_forward} is symbolically \emph{identical} to $\neg\texttt{stop}$ (proved by an UNSAT check over all $2^{20}$ assignments, then independently reconfirmed by full brute-force enumeration of the same $2^{20}$ points), cutting the 16 syntactically possible (stop, move-forward, turn-left, turn-right) combinations to the 8 that are actually reachable. The 20 free variables factor into three disjoint blocks (front-8, left-6, right-6) governing (stop/move-forward), turn-left, and turn-right respectively.
Exhaustive per-block stabilizer search (768 candidates checked, not sampled) gives $\Aut(\PhiY{y^*})$ exactly for all 8 reachable labels, independently reconfirmed by $3{,}000/3{,}000$ random cross-block masks tested directly against the full $2^{20}$ indicator array without relying on the block-factorization assumption. \Cref{tab:three-families} reports all eight.

\begin{table}[t]
\centering
\small
\begin{tabular}{lrrrrrc}
\toprule
Instance & $|\Phi|$ & $|\mathrm{Aut}|$ & orbits & max orbit & pairwise $\rho$ & transitive \\
\midrule
CLE4EVR (6-dim) & 12 & 24 & 1 & 12 & 0.00\% & yes \\
Kandinsky & 162 & 36 & 6 & 36 & 81.99\% & no \\
\midrule
\multicolumn{7}{l}{\textit{BDD-OIA/SDD-OIA: label-driven $\PhiY{y^*}$, $y^*=$(stop, fwd, left, right)}} \\
\midrule
$(1,0,0,0)$ & 812,250 & 2 & 406,125 & 2 & 99.9999\% & no \\
$(1,0,0,1)$ & 99,750 & 2 & 49,875 & 2 & 99.9990\% & no \\
$(1,0,1,0)$ & 99,750 & 2 & 49,875 & 2 & 99.9990\% & no \\
$(1,0,1,1)$ & 12,250 & 2 & 6,125 & 2 & 99.9918\% & no \\
$(0,1,0,0)$ & 19,494 & 2 & 9,747 & 2 & 99.9949\% & no \\
$(0,1,0,1)$ & 2,394 & 2 & 1,197 & 2 & 99.9582\% & no \\
$(0,1,1,0)$ & 2,394 & 2 & 1,197 & 2 & 99.9582\% & no \\
$(0,1,1,1)$ & 294 & 2 & 147 & 2 & 99.6587\% & no \\
\bottomrule
\end{tabular}
\caption{Three benchmark families under $\AutX$, extending Table~\ref{tab:clevr-comparison}; the CLE4EVR row repeats that table's 6-dimensional measurement for comparison. BDD-OIA/SDD-OIA figures are shown to four decimal places because every value lies within $0.35$ percentage points of $100\%$; two decimals would flatten all eight rows to $100.00\%$.}
\label{tab:three-families}
\end{table}

\Cref{tab:three-families}'s dominant finding is that $|\Aut(\PhiY{y^*})| = 2$ on \emph{every} reachable label, and this constant is driven by a single coordinate. One consequence should be read off before the percentages are: for Boolean instances the group acts freely, so every orbit has size exactly $|\Aut|$ and $\rho = 1-(|\Aut|-1)/(|\Phi|-1)$ is an exact law, proved as \Cref{prop:free-action}; every BDD-OIA/SDD-OIA row of the table (orbit count, maximum orbit, and each percentage) is that law evaluated at its $(|\Phi|,|\Aut|)$ pair. The nines are a corollary of the group order and the solution count, so the informative claim in these rows is $|\Aut|=2$ itself, not the percentage.
\texttt{follow} is more subtle than CLE4EVR's free size coordinates: it is not syntactically absent from the rule. It appears explicitly in the unsimplified formula $\texttt{move\_forward\_cond} = \mathrm{Or}(\texttt{green\_light}, \texttt{follow}, \texttt{road\_clear})$, but once $\texttt{stop}=\mathrm{False}$ forces $\texttt{road\_clear}=\mathrm{True}$, Boolean absorption erases its contribution. This fact needed symbolic proof to surface, not inspection of the rule text. This is a second, independent trigger of the same dead-coordinate pathology already documented for CLE4EVR's \texttt{size\_1}/\texttt{size\_2} (syntactic absence there; semantic absorption here), an instance of a phenomenon long known in the classical CSP symmetry-breaking literature: variables the constraints never really pin down contribute spurious automorphisms \cite{cohen2006symmetry}.

Because flipping \texttt{follow} is a fixed-point-free involution, $|\PhiY{y^*}|$ is even whenever $\PhiY{y^*}\neq\varnothing$, so $\mathrm{SM}(y^*) = |\PhiY{y^*}|-1 \geq 1$ for every one of the 8 reachable labels: under Definition~\ref{def:general-csp}, this decision logic cannot be shortcut-free for \emph{any} label a real vehicle can reach.
The diagonal shared-value extension of Definition~\ref{def:takemura-autx}'s action (one global bit-flip applied to all 20 positions at once; the literal definition again has no bijective solutions to act on, since $|N|=20$ against $|S|=2$) does \emph{worse} here than padding does on CLE4EVR. Tested on the two most extreme labels, it preserves neither level set: $\Aut_{\mathrm{pad}}(\PhiY{y^*})$ collapses to the trivial group exactly where the independent-coordinate reading still finds the genuine order-2 symmetry.
rsbench's own named example (a model that cannot tell a pedestrian from a red light, since both correctly imply the \emph{stop} action \cite{rsbench2024}) is a direct instance of the unexplained $99.9999\%$: the all-\texttt{person} and all-\texttt{red\_light} concept vectors both map to $y^*=(1,0,0,0)$, but they differ in two coordinates, not the one (\texttt{follow}) that $\Aut(\PhiY{(1,0,0,0)})$ can flip, so they lie in different orbits.

The quantity measured in all 8 BDD-OIA/SDD-OIA rows is fixed precisely: $\PhiY{y^*}$ is the preimage of $y^*$ under \texttt{sddoiaK} over the full $\{0,1\}^{20}$ hypercube, the object the rule itself defines, independent of any sampling distribution. Restricting to a physically consistent support is a different, distribution-relative measurement, and we ran it. Writing $\Aut_H(\PhiY{y}\!\cap\Omega)$ for the relabelings that stabilize a label's support-restricted fiber, and taking $\Omega$ to be the vectors satisfying the mutual-exclusion and presupposition constraints readable off the concept vocabulary (no simultaneous red and green light, no lane both present and absent, no left green light without a left lane, and symmetrically on the right), the solution sets shrink by factors of $2.6$ to $6.7$ and $\rho$ stays between $96.97\%$ and $99.9993\%$: the nines are a property of the rule's near-rigidity, not an artifact of impossible concept vectors.
Support restriction also does something more interesting than shrinking $\Phi$, and it is worth recording because it is a fact about the object rather than about our instrument: on the four \texttt{move\_forward} labels it \emph{creates} a second dead coordinate. \texttt{green\_light} is live on the full hypercube but inert once contradictory light states are excluded, so the group of fiber-preserving coordinate flips rises from order $2$ to $4$ on those labels while remaining order $2$ on the four \texttt{stop} labels. The support-restricted search tests single-coordinate flips, so these groups, orbit counts, and $\rho$ values are those of the flip-generated subgroup of $\Aut_H(\PhiY{y}\!\cap\Omega)$, which the exhaustive full-hypercube computation confirms is the whole stabilizer there; an element the flips miss could only enlarge the group and lower $\rho$. The global support-restricted flip group, the intersection over all eight fibers, stays at order $2$, since the \texttt{green\_light} flip fails on the \texttt{stop} fibers. A support is not a neutral filter on a symmetry analysis; it is part of the instance.

Across all three families the headline ranges over $0\%$, $81.99\%$, and $99.66$--$99.9999\%$: not uniformly high, not uniformly low. CLE4EVR's conjunction of independent equality constraints is exactly transitive; Kandinsky's and BDD-OIA/SDD-OIA's disjunctive decision structure is not. \Cref{sec:empirical-extension} tests whether this pattern, and not just its two positive instances, generalizes.

\subsection{Eight further rule families, under pre-specified predictions}
\label{sec:empirical-extension}

Three benchmark families are not enough to tell a real structural pattern from a coincidence built on two hand-picked disjunctive examples. This section extends the measurement to eight further rule families, seven independent sources, and 15 separately pre-specified structure-to-pathology predictions.

\paragraph{Prediction discipline.}
Each family's structural classification and predicted outcome are recorded as a comment block in that family's released measurement script, beside the code that tests them.
``Pre-specified'' throughout this paper refers to predictions recorded this way: the checkable fact is the recording in the released source, which the reader can inspect; the artifact carries no independent timestamp beyond it.

\paragraph{Sources and toolchain.}
Seven sources feed the eight families of \Cref{tab:eight-families}: (i) Takemura et al.'s own Example~2 \citeyearonly{takemura2026} (MNIST-Half), the one further fully worked example in their paper, with zero reconstruction uncertainty; (ii)--(iv) three further \texttt{get\_label} rules in rsbench's own MNIST-addition source \cite{rsbench2024} (\texttt{sum}, tracing to DeepProbLog's MNIST-addition task \cite{manhaeve2018deepproblog}; \texttt{product}; and a previously untested third branch, \texttt{multiop}), plus a same-family structural analogue restricted to even digits (MNAdd-EvenOdd); (v) an ensemble of 36 random $l$-CNF instances built from rsbench's own MNLogic recipe (own generator, not a byte-exact port); (vi) order-3 Latin squares and order-4 mini-Sudoku, a public combinatorial rule outside rsbench; (vii) CLEVR-Hans3 \cite{clevrhans2021}, ground-truth class rules fetched verbatim from the released \texttt{Clevr\_Hans\_GTClasses\_3.json}, a second source outside rsbench. All are measured with the same two engines as \Cref{sec:empirical-setup}, plus the nauty-based scalable engine for the largest domains (e.g.\ two free MNIST digits over $\{0,\dots,9\}$, $(10!)^2\approx1.3\times10^{13}$), which is the same engine already cross-validated at the orbit level against the brute-force results on CLE4EVR and Kandinsky.

\Cref{tab:eight-families} reports all eight families. Family-level rows aggregate a variable number of individually pre-specified sub-predictions (17 summed target sums for MNAdd-sum, 3 template classes for CLEVR-Hans3, and so on); the total across all eight rows is 15 pre-specified units of varying granularity, of which \textbf{13 (86.7\%) were confirmed}.

\begin{table}[t]
\centering
\footnotesize
\setlength{\tabcolsep}{4pt}
\begin{tabular}{p{1.9cm}p{2.1cm}p{2.6cm}p{2.4cm}p{3.4cm}c}
\toprule
Family & Source & Structural class (\emph{a priori}) & Prediction & Measured ($\AutX$) & Hit \\
\midrule
MNIST-Half & Takemura Ex.~2 & arithmetic conjunction, 4 sum-equalities & transitive & $|\Phi|{=}3$, transitive, $0\%$ & 1/1 \\
MNAdd-sum & rsbench & single linear equality & transitive, all targets & 17/17 targets transitive, $0\%$ & 1/1 \\
MNAdd-product & rsbench & single non-linear equality & $\geq 1$ target non-transitive & 31/32 transitive; $y{=}0$ non-transitive ($57.89\%$) & 1/1 \\
MNAdd-EvenOdd & rsbench, domain analogue & single linear equality & transitive & 7/7 targets transitive, $0\%$ & 1/1 \\
MNIST-multiop & rsbench, 3rd \texttt{get\_label} branch & 3 conjunctive branches + 1 exclusion residual & branches 0--2 transitive, branch 3 not & classes 0--2 transitive ($0\%$); class 3 non-transitive ($79.43\%$) & 4/4 \\
MNLogic ($l$-CNF) & rsbench recipe, own generator & disjunctive / CNF & majority non-transitive & 36/36 non-transitive, mean $\rho{=}94.55\%$ & 1/1 \\
Latin squares \& Sudoku & public combinatorics & all-different (permutation) & non-transitive (pigeonhole) & order-3 (12, $54.55\%$); Sudoku (288, $91.99\%$); order-4 Latin (576, $96.00\%$) & 3/3 \\
CLEVR-Hans3 & Stammer et al.\ 2021 & disjunctive slot-assignment & non-transitive, all 3 classes & cls.\,0 transitive $0\%$ (miss); cls.\,1 non-transitive $75.29\%$ (hit); cls.\,2 transitive $0\%$ (miss) & 1/3 \\
\bottomrule
\end{tabular}
\caption{Eight further rule families. Structural class and prediction are recorded in each measurement script (Prediction discipline, above). 13/15 pre-specified units confirmed; both misses are CLEVR-Hans3, discussed below.}
\label{tab:eight-families}
\end{table}

\paragraph{The refined criterion: branch exchangeability.}
Both misses are CLEVR-Hans3 classes predicted non-transitive, by direct analogy with Kandinsky's disjunctive slot-matching, that measured fully transitive instead: class~0 ($|\Phi|=512$) and class~2 ($|\Phi|=8$), both $\rho=0\%$.
Each class rule is a disjunction over which of two generic object slots matches which of two attribute templates. Class~0's templates differ only in shape (\texttt{cube} vs.\ \texttt{cylinder}), with size tied to the same value (\texttt{large}) in both and material/color free in both; class~2's templates differ in size and color together (\texttt{large}+\texttt{blue} vs.\ \texttt{small}+\texttt{yellow}), with shape (\texttt{sphere}) fixed in both.
In both cases, swapping the two templates leaves the constraint syntactically unchanged, so there is a coordinatewise permutation (flip the discriminating attribute values in lockstep) that exchanges the two disjuncts; and within each branch the pinned attributes are constant and the rest fully free, a product set on which the branch-preserving coordinatewise permutations already act transitively, so both hypotheses of \Cref{thm:branch-swap} hold and the branches merge into one orbit. Class~1 is different: it pins \texttt{material=metal} in one template and leaves material entirely free in the other, an asymmetric pinning that no per-slot value permutation can exchange, and it measures non-transitive as originally predicted ($\rho = 75.29\%$).
This sharpens Kandinsky's and MNLogic's ``disjunction $\to$ pathology'' pattern into a precise, checkable sufficient condition (\Cref{thm:branch-swap}): branches exchangeable by a coordinatewise permutation, each internally transitive, force transitivity, and the one class that measured non-transitive is exactly the one whose asymmetric pinning defeats exchangeability. Syntactic disjunction by itself decides nothing.

\paragraph{A conjunction exception: multiplicative zero.}
MNAdd-product's 32 reachable target products are 31/32 transitive, matching the ``mild conjunction $\to$ transitive'' pattern, except at $y=0$ ($|\Phi|=19$, every pair with $c_1{=}0$ or $c_2{=}0$), which is $57.89\%$ non-transitive.
Zero is a multiplicative absorbing element: exactly three orbits result, of sizes 9, 9, and 1: $\{(0,0)\}$, $\{(0,k):k\neq0\}$, and $\{(k,0):k\neq0\}$.
Any automorphism fixing the doubly-zero solution $(0,0)$ cannot also merge the two single-zero families, since that would require one coordinate's permutation to send $0$ to a non-zero value while a bijection has already fixed that coordinate at $0$. This is the cleanest counterexample yet to ``conjunction $\Rightarrow$ transitive'': a single arithmetic equality, syntactically as mild as MNAdd-sum's, made pathological by one exceptional value.

\paragraph{A third, independent mechanism: the counting bound.}
Latin squares and mini-Sudoku are all-different constraints -- syntactically a large conjunction of pairwise inequalities, closer to CLE4EVR's side of the divide than to Kandinsky's -- yet all three instances tested are strongly non-transitive: $54.55\%$ (order-3 Latin square, $|\Phi|=12$), $91.99\%$ (order-4 Sudoku, $|\Phi|=288$), $96.00\%$ (order-4 Latin square, no boxes, $|\Phi|=576$).
The mechanism is pure pigeonhole, recorded before running any code: an orbit has size at most $|\mathrm{Aut}|$, and $|\Phi|$ ($12$, $288$, $576$, literature-known counts) exceeds the diagonal cap ($3!{=}6$ or $4!{=}24$) in every case, so a transitive action is numerically impossible, independent of any conjunction/disjunction argument.
The recorded argument took the diagonal cap for granted, and under Definition~\ref{def:componentwise-autx} it is not free: the ambient componentwise group of even a $3\times3$ grid has order $6^9$, so the cap has to be earned. Lemma~\ref{lem:diagonal-collapse} earns it by a forcing argument (pair-rich all-different lines admit only diagonal automorphisms), and the exact groups land on the cap's structure precisely: every element diagonal, $|\AutX|=6$, $24$, $24$.
Grids are also the family on which the hierarchy's rungs can be tested against each other: because every cell draws from the same alphabet, both the componentwise instrument and the diagonal shared-value extension $\Aut_{\mathrm{pad}}$ (\Cref{sec:background}) are defined, and they coincide exactly, $|\AutX| = |\Aut_{\mathrm{pad}}| = 6, 24, 24$, each computed independently rather than copied. The pigeonhole verdict is therefore invariant to the choice of rung, and the padded reading's failures in \Cref{sec:background} are confined to where the hierarchy says they should be: heterogeneous domains. \Cref{tab:eight-families} reports the $\AutX$ figures for consistency with the rest of this paper.

Combined with \Cref{sec:empirical-three}, 15 pre-specified predictions across 8 families and 3 independent benchmark families give the clearest evidence so far that pathology under $\AutX$ tracks specific, checkable structural properties of a rule, not a crude conjunction/disjunction label.

\subsection{The four-mechanism characterization}
\label{sec:empirical-mechanisms}

The coarse pattern from \Cref{sec:empirical-three} (conjunction is safe, disjunction is not) does not survive \Cref{sec:empirical-extension} intact: it is directionally right ($13/15$) but wrong in exactly the two ways that matter most for a mechanistic account. \Cref{tab:four-mechanisms} replaces it with four mechanisms, each tied to specific instances measured above.

\begin{table}[t]
\centering
\small
\begin{tabular}{p{2.3cm}p{4.0cm}p{2.0cm}p{4.2cm}}
\toprule
Mechanism & Criterion & Effect & Instances ($N$) \\
\midrule
1.\ Mild conjunction & arithmetic equalities / diagonal constraints, no absorbing or excluded special value & transitive & CLE4EVR; MNIST-Half; MNAdd-sum (17); MNAdd-EvenOdd (7); MNIST-multiop branches 0--2 (3); MNAdd-product (31) -- $N{=}60$ \\
2.\ Exceptional element & an absorbing/zero value or an exclusion (``otherwise'') branch inside an otherwise mild conjunction & can still be non-transitive & MNAdd-product $y{=}0$; MNIST-multiop residual branch -- $N{=}2$ \\
3.\ Branch (a)symmetry & disjunction: branches exchangeable by a coordinatewise permutation, each internally transitive (sufficiency proved, \Cref{thm:branch-swap}; both hypotheses checked on the three CLEVR-Hans3 classes, matching all three outcomes) & depends & transitive, criterion verified: CLEVR-Hans3 cls.~0, 2 ($N{=}2$); non-transitive: CLEVR-Hans3 cls.~1 (asymmetric pinning), Kandinsky (\Cref{thm:free-slot}), BDD-OIA/SDD-OIA (8), MNLogic (36), MNIST-multiop residual (observed) -- $N{=}47$ \\
4.\ Orbit-stabilizer counting & $|\Phic|$ exceeds the order of the group that could act on it, regardless of syntax & non-transitive (pigeonhole) & 3 Latin-square/Sudoku instances -- $N{=}3$ \\
\bottomrule
\end{tabular}
\caption{Four-mechanism characterization of transitivity under $\AutX$, replacing the coarse conjunction/disjunction split. $N$ counts individual rule/label instances already reported in \Cref{tab:three-families,tab:eight-families}.}
\label{tab:four-mechanisms}
\end{table}

Mechanisms 1 and 3 recover the coarse pattern in the regime where it holds; mechanism 2 shows conjunction is not automatically safe (one exceptional value is enough); mechanism 4 shows non-transitivity can be forced by counting alone, with no reference to conjunction or disjunction. An all-different constraint is, syntactically, exactly the kind of mild conjunction mechanism 1 predicts should be transitive, and it is not, for a reason that has nothing to do with mechanism 1's failure mode.

The four mechanisms are an empirical characterization with an 8-family, 15-prediction basis, and \Cref{sec:algebra} then proves each at the level it admits: mechanism 1 becomes the matching-decomposition condition (\Cref{thm:matching-decomposition}), mechanism 2 the degree invariant (\Cref{thm:degree-invariant}), mechanism 3's symmetric half the branch-swapping condition (\Cref{thm:branch-swap}) and its Kandinsky-type residual the Free Slot Lemma (\Cref{thm:free-slot}), and mechanism 4 the orbit-stabilizer bound (\Cref{thm:orbit-stabilizer}); \Cref{thm:anchor-forcing} additionally exhibits purely conjunctive counterexamples showing mechanism 2 cannot be removed by any syntactic strengthening of mechanism 1. \Cref{sec:algebra-synthesis} tabulates this correspondence and its coverage boundary. The one mechanism \Cref{sec:algebra} leaves as an empirical criterion rather than a closed-form test is branch (a)symmetry (mechanism 3) for disjunctions with more than two branches of unequal syntactic shape; Kandinsky's three-way disjunction and CLEVR-Hans3's two-way disjunctions anchor it, and \Cref{sec:empirical-extension}'s 8 families are its current basis.

\section{Algebraic Conditions for Transitivity and Its Failure}
\label{sec:algebra}

Section~\ref{sec:background} defined componentwise value symmetry $\AutX$
(\Cref{def:componentwise-autx}), and \Cref{sec:empirical} measured its
action on $\Phic$ across real rule families: sometimes transitive,
sometimes not, with both outcomes traceable to specific features of the
constraint set rather than to whether the constraint is written as a
conjunction or a disjunction. This section makes that dependence precise. We give six
results, each a sufficient condition on the syntactic shape of $C$ for
transitivity or for its failure, built on a single forcing technique that
also clarifies what value symmetry can and cannot see. The results are
sufficient conditions, not a classification, and we say so precisely:
Section~\ref{sec:algebra-synthesis} states exactly which of the four
empirical mechanisms from \Cref{sec:empirical} each theorem covers, and which real
instance resists all six.

\subsection{Setup}
\label{sec:algebra-setup}

Fix an instance $X=(N,S,C)$ in the sense of \Cref{def:general-csp}, with
solution set $\Phic\subseteq\prod_{i\in N}S_i$. Write $G:=\Gind$ for the
ambient group and $\AutX\leq G$ for its componentwise automorphism group
(\Cref{def:componentwise-autx}), acting on $\prod_{i\in N}S_i$
coordinatewise by $(\sigma\cdot\phi)_i=\sigma_i(\phi_i)$. Since $\AutX$
stabilizes $\Phic$ as a set, it restricts to an action of $\AutX$ on
$\Phic$ itself. This section asks when that restricted action is
\emph{transitive}: a single orbit, so every two solutions are related by
some element of $\AutX$.

The pair $(G,\Phic)$ is the direct generalization, to arbitrary
finite local domains, of the \emph{autotopism group} of a Latin square or
quasigroup in design theory, where $\Phic$ is the set of triples
satisfying a Latin condition and $G$ is a product of row, column, and
symbol permutation groups; \Cref{sec:related} discusses that connection. In the
Boolean case $S_i\equiv\{0,1\}$, \Cref{obs:boolean-linear} already showed
$\AutX$ is always an $\Ftwo$-linear subspace of $\Ftwo^n$; \Cref{sec:complexity}
develops that special structure. The results below make no assumption on
$|S_i|$ and use only the group action itself.

\subsection{The Forcing Lemma}
\label{sec:algebra-forcing}

Every intransitivity result in this section traces back to one
observation: a sufficiently rigid two-coordinate projection of $\Phic$
forces the corresponding two coordinates of every automorphism to agree.

\begin{lemma}[Forcing]
\label[lemma]{lem:forcing}
Let $X=(N,S,C)$ be an instance in the sense of \Cref{def:general-csp} and
let $i,j\in N$ satisfy $S_i=S_j=:D$. Write $\pi_{ij}(\phi):=(\phi_i,\phi_j)$
for the projection $\prod_{k\in N}S_k\to D\times D$ and
$\Delta_D:=\{(d,d):d\in D\}$ for the diagonal. If
\[
\pi_{ij}(\Phic)=\Delta_D \qquad\text{or}\qquad \pi_{ij}(\Phic)=D\times D\setminus\Delta_D,
\]
then every $\sigma=(\sigma_1,\dots,\sigma_n)\in\AutX$ satisfies
$\sigma_i=\sigma_j$.
\end{lemma}

\begin{proof}
Projection commutes with the coordinatewise action: for every
$\phi\in\Phic$, $\pi_{ij}(\sigma\cdot\phi)=(\sigma_i(\phi_i),\sigma_j(\phi_j))$.
Since $\sigma\in\AutX$ maps $\Phic$ onto itself, the pair
$(\sigma_i,\sigma_j)$ maps $\pi_{ij}(\Phic)$ onto itself.

Suppose $\pi_{ij}(\Phic)=\Delta_D$. Fix $d\in D$; since $\Delta_D$
contains $(d,d)$, so does $\pi_{ij}(\Phic)$, so
$(\sigma_i(d),\sigma_j(d))=(\sigma_i,\sigma_j)\cdot(d,d)\in\Delta_D$,
giving $\sigma_i(d)=\sigma_j(d)$. As $d$ was arbitrary, $\sigma_i=\sigma_j$.

Suppose instead $\pi_{ij}(\Phic)=D\times D\setminus\Delta_D$. Fix $d\in D$
and let $d'$ range over $D\setminus\{d\}$. Each $(d,d')$ lies in
$\pi_{ij}(\Phic)$, so $(\sigma_i(d),\sigma_j(d'))\notin\Delta_D$, that is
$\sigma_i(d)\neq\sigma_j(d')$ for every $d'\neq d$. Since $\sigma_j$ is a
bijection of $D$, this forces $\sigma_i(d)=\sigma_j(d)$: it is the one
value of $D$ that $\sigma_i(d)$ is not excluded from. As $d$ was
arbitrary, $\sigma_i=\sigma_j$.
\end{proof}

The lemma says nothing about whether $\AutX$ is transitive; it only pins
down a relationship between two coordinates of every automorphism, whether
or not one exists. Theorem~\ref{thm:anchor-forcing} and
Theorem~\ref{thm:free-slot} both build intransitivity results on top of
this forced relationship, and the proof of
Theorem~\ref{thm:matching-decomposition} uses the same commuting-projection
identity to establish the opposite conclusion on a different family of
instances.

\subsection{Structural criteria forcing transitivity}
\label{sec:algebra-transitivity}

The first two results identify constraint shapes under which $\AutX$ is
guaranteed to act transitively on $\Phic$.

\begin{definition}[Matching decomposition]
\label[definition]{def:matching-decomposition}
A partition $N=N_1\sqcup\cdots\sqcup N_r\sqcup F$ is a \emph{matching
decomposition} of $\Phic$ if, for each block $N_l=\{i_{l,1},\dots,
i_{l,m_l}\}$, there is a nonempty set $T_l\subseteq S_{i_{l,1}}$ and
injections $f_{l,j}:T_l\to S_{i_{l,j}}$ for $j=2,\dots,m_l$ (with
$f_{l,1}:=\mathrm{id}_{T_l}$), such that
\[
\Phic \;=\; \prod_{l=1}^{r}\Psi_l \;\times\; \prod_{i\in F}S_i,
\qquad
\Psi_l:=\big\{(f_{l,1}(t),\dots,f_{l,m_l}(t)):t\in T_l\big\}\subseteq
\textstyle\prod_{j=1}^{m_l}S_{i_{l,j}}.
\]
\end{definition}

Each block is parametrized by a single value $t\in T_l$ that determines
every coordinate in the block through the injections $f_{l,j}$; blocks and
free positions $F$ combine freely. Taking every $f_{l,j}$ to be an
identity map on a shared domain recovers ``all positions in the block are
equal,'' the pattern behind CLE4EVR's three matching-attribute constraints;
taking $f_{l,j}(t)=k-t$ or another arithmetic bijection recovers a sum or
difference constraint such as MNIST-Half's four summation equalities,
chained across a tree of positions by composing the bijections along the
tree.

\begin{theorem}[Matching decomposition, Theorem A]
\label{thm:matching-decomposition}
If $N=N_1\sqcup\cdots\sqcup N_r\sqcup F$ is a matching decomposition of
$\Phic$, then $\AutX$ acts transitively on $\Phic$.
\end{theorem}

\begin{proof}
Fix a block $N_l$ and $\alpha\in\Sym(T_l)$. Define
$\sigma^{(l)}(\alpha)=(\sigma_1,\dots,\sigma_n)\in G$ by: $\sigma_{i_{l,1}}$
extends $\alpha$ by the identity on $S_{i_{l,1}}\setminus T_l$;
$\sigma_{i_{l,j}}$ for $j\geq 2$ extends
$f_{l,j}\circ\alpha\circ f_{l,j}^{-1}$, a permutation of $f_{l,j}(T_l)$, by
the identity on $S_{i_{l,j}}\setminus f_{l,j}(T_l)$; every coordinate
outside $N_l$ is the identity. Composing two such maps composes the
underlying permutations of $T_l$ coordinatewise, so
$\alpha\mapsto\sigma^{(l)}(\alpha)$ is a group homomorphism
$\Sym(T_l)\to G$, injective because $\sigma^{(l)}(\alpha)$ determines
$\alpha$ on $i_{l,1}$; write $G_l\leq G$ for its image.

$G_l$ stabilizes $\Phic$: $\sigma^{(l)}(\alpha)$ is the identity outside
$N_l$, so it suffices to check it maps $\Psi_l$ onto itself. For
$t\in T_l$, the point $(f_{l,1}(t),\dots,f_{l,m_l}(t))\in\Psi_l$ maps to
$\big(\alpha(t),\, f_{l,2}(\alpha(f_{l,2}^{-1}(f_{l,2}(t)))),\,\dots\big)
=\big(\alpha(t),f_{l,2}(\alpha(t)),\dots,f_{l,m_l}(\alpha(t))\big)$,
using that $f_{l,j}$ is injective; this is again the point of $\Psi_l$
indexed by $\alpha(t)$. Hence $G_l\leq\AutX$, and $G_l$ acts on $\Psi_l$
exactly as $\Sym(T_l)$ acts on $T_l$ under the bijection
$t\leftrightarrow(f_{l,1}(t),\dots,f_{l,m_l}(t))$, transitively.

For the free positions, $G_F:=\prod_{i\in F}\Sym(S_i)\leq\AutX$ trivially,
and $G_F$ is transitive on $\prod_{i\in F}S_i$.

The subgroup $\prod_{l=1}^r G_l\times G_F\leq\AutX$ acts on
$\Phic=\prod_l\Psi_l\times\prod_F S_i$ coordinatewise, factor by factor, on
pairwise disjoint sets of positions. A product of transitive actions on
disjoint coordinates is transitive on the product: given two points of
$\Phic$, transitivity of each factor supplies a group element matching
that factor, and the factors compose into a single element of
$\prod_l G_l\times G_F$ matching every factor at once. Hence $\AutX$,
which contains this subgroup, is transitive on $\Phic$.
\end{proof}

Computationally verified; see \Cref{app:verification}.

The second transitivity result covers instances built from several
disjoint syntactic cases, or branches, rather than a single matching
pattern.

\begin{theorem}[Branch swapping, Theorem B]
\label{thm:branch-swap}
Let $\Phic=\Phi_1\cup\cdots\cup\Phi_r$. Suppose there is a subgroup
$G_1\leq\AutX$ with $G_1(\Phi_1)=\Phi_1$ acting transitively on $\Phi_1$,
and elements $\kappa_2,\dots,\kappa_r\in\AutX$ with $\kappa_l(\Phi_1)=
\Phi_l$ for $l=2,\dots,r$. Then $\AutX$ acts transitively on $\Phic$.
\end{theorem}

\begin{proof}
Set $\kappa_1:=\mathrm{id}$. Given $\phi\in\Phi_l$ and $\phi'\in\Phi_{l'}$
(possibly $l=l'$), $\kappa_l^{-1}(\phi)$ and $\kappa_{l'}^{-1}(\phi')$
both lie in $\Phi_1$, since $\kappa_l(\Phi_1)=\Phi_l$. Transitivity of
$G_1$ on $\Phi_1$ supplies $h\in G_1$ with
$h(\kappa_l^{-1}(\phi))=\kappa_{l'}^{-1}(\phi')$. Then
$g:=\kappa_{l'}\circ h\circ\kappa_l^{-1}\in\AutX$ satisfies $g(\phi)=\phi'$.
\end{proof}

\begin{remark}[Value symmetry is not positional symmetry]
\label[remark]{rem:value-vs-position}
Theorem~\ref{thm:branch-swap} requires the connecting elements
$\kappa_l$ to lie in $\AutX\leq\Gind$: permutations of \emph{values},
fixing every position. This is a real restriction, not a technicality,
and it is exactly what separates two benchmarks that look alike at first
glance. CLEVR-Hans3's class-0 and class-2 rules each split into two
branches distinguished only by which of two attribute values a shared
position takes; swapping those two values, with every object staying in
its own position, carries one branch onto the other, and each branch is
a product set (pinned attributes constant, the rest free) whose
branch-preserving permutations act transitively on it, so both of
Theorem~\ref{thm:branch-swap}'s hypotheses hold. Kandinsky's three branches are
instead indexed by \emph{which pair of objects matches}: branch $l$
requires the objects of one designated pair to agree. Elements of
$\Gind$ carrying one branch onto another do exist, since independent
coordinate permutations can break one cross-position equality while
creating another; what fails is membership in $\AutX$.
Section~\ref{sec:algebra-kandinsky} proves that every element of
$\AutX$ preserves each of the three branches setwise, the opposite of
what a connecting $\kappa_l$ requires, so
Theorem~\ref{thm:branch-swap} cannot apply to Kandinsky's branch
structure no matter how the rule is written: exchanging which object
plays which role is a permutation of $N$, and among value relabelings
that stabilize the full solution set, none simulates it. The
boundary this draws is quantified, not just named: extending $\AutX$ by
the three object-slot permutations, positional symmetries outside every
rung of the value-symmetry hierarchy, yields an order-216 group whose
action merges Kandinsky's six orbits into two, of sizes $[108,54]$,
with $\rho=44.72\%$. Which reading is right is an architectural
question in exactly the sense of \Cref{def:typed-autx}'s discussion, an
object-slot-exchangeable perception design admits the positional
relabelings and an ordered one does not; the value-symmetry figure
$81.99\%$ and the position-extended figure $44.72\%$ bracket
Kandinsky's ambiguity between the two designs, and the $44.72\%$
residue is what no relabeling of either kind explains.
\end{remark}

Computationally verified; see \Cref{app:verification}.

\subsection{Structural criteria forcing intransitivity, I: forced fibers}
\label{sec:algebra-anchor}

The next result gives a sufficient condition for the failure of
transitivity, built directly on the Forcing Lemma. It shows that a
conjunctive constraint set, with no disjunction anywhere in it, can still
fail to be transitive: transitivity is not implied by conjunction, even
though every diagonal-equality example in
Section~\ref{sec:algebra-transitivity} is a conjunction.

\begin{theorem}[Anchor forcing, Theorem C]
\label{thm:anchor-forcing}
Let $X=(N,S,C)$ be an instance in the sense of \Cref{def:general-csp} and
let $i,j,k\in N$ satisfy $S_i=S_j=S_k=:D$. Suppose
\begin{enumerate}
\item[(a)] $\pi_{ij}(\Phic)$ equals $\Delta_D$ or $D\times D\setminus\Delta_D$;
\item[(b)] $\pi_{ik}(\Phic)$ equals $\Delta_D$ or $D\times D\setminus\Delta_D$, independently of the choice made in (a);
\item[(c)] both $\phi_j=\phi_k$ and $\phi_j\neq\phi_k$ occur among solutions $\phi\in\Phic$.
\end{enumerate}
Then $\AutX$ does not act transitively on $\Phic$.
\end{theorem}

\begin{proof}
By Lemma~\ref{lem:forcing} applied to $(i,j)$ and to $(i,k)$, every
$\sigma\in\AutX$ satisfies $\sigma_i=\sigma_j$ and $\sigma_i=\sigma_k$,
hence $\sigma_j=\sigma_k=:\gamma$, a single permutation of $D$ shared by
every automorphism. If $\phi'=\sigma\cdot\phi$ for some $\sigma\in\AutX$,
then $\phi'_j=\gamma(\phi_j)$ and $\phi'_k=\gamma(\phi_k)$, so
$\phi'_j=\phi'_k$ if and only if $\phi_j=\phi_k$, since $\gamma$ is a
bijection. The predicate ``$\phi_j=\phi_k$'' is therefore constant on every
$\AutX$-orbit. By (c) it takes both values on $\Phic$, so $\Phic$ is not a
single orbit.
\end{proof}

The smallest instance exhibiting this mechanism has $D=\{0,1,2\}$ and
\[
\Phic=\{(a,b,c)\in D^3 : a\neq b,\ a\neq c\},
\]
a conjunction of two inequalities with no disjunction. Here
$|\Phic|=12$, $\pi_{ij}(\Phic)=\pi_{ik}(\Phic)=D\times D\setminus\Delta_D$
and $|\AutX|=6$ (the diagonal copy of $\Sym(D)$), splitting into two
orbits of size 6, exactly the classes $b=c$ and $b\neq c$ predicted by the
proof.
This directly falsifies the conjecture that a conjunctive constraint set
implies transitivity: two inequality constraints, sharing no disjunction
with any other structure, already break transitivity, by leaving open
whether the two non-anchor legs agree with each other.

The counterexample predates the theorem: it surfaced in a randomized
search and was then proved, and the searches around it carry
falsification value of their own. A broader random search over
conjunctive and disjunctive constraint sets confirms neither connective
determines the outcome on its own; a second, independent search targeting
false positives of the theorem's hypothesis found none; and the
anchor-and-two-legs pattern generalizes to more attached legs without
degenerating, confirming the mechanism is not an artifact of the smallest
case. Computationally verified; see \Cref{app:verification}.

\subsection{Structural criteria forcing intransitivity, II: two invariants}
\label{sec:algebra-invariants}

The anchor-forcing pattern needs two positions sharing a domain with the
anchor. The next two results give intransitivity criteria that need no
such coincidence, built from invariants any $\sigma \in\AutX$ must
preserve.

\begin{lemma}
\label[lemma]{lem:degree-invariant}
For $\sigma\in\AutX$, position $i\in N$, and value $v\in S_i$, write
$\deg_i(v):=|\{\phi\in\Phic:\phi_i=v\}|$. Then
$\deg_i(v)=\deg_i(\sigma_i(v))$.
\end{lemma}

\begin{proof}
$\sigma$ restricts to a bijection $\Phic\to\Phic$, since $\sigma$ and
$\sigma^{-1}$ both lie in $\AutX$. It carries $\{\phi\in\Phic:\phi_i=v\}$
bijectively onto $\{\phi\in\Phic:\phi_i=\sigma_i(v)\}$, because
$\phi_i=v$ holds if and only if $(\sigma\cdot\phi)_i=\sigma_i(v)$ holds.
Bijective images have equal cardinality.
\end{proof}

\begin{theorem}[Degree invariant, Theorem D]
\label{thm:degree-invariant}
If some position $i\in N$ has two values $v,v'\in S_i$ that both occur in
$\Phic$ with $\deg_i(v)\neq\deg_i(v')$, then $\AutX$ does not act
transitively on $\Phic$.
\end{theorem}

\begin{proof}
Suppose $\AutX$ is transitive. Choose $\phi,\phi'\in\Phic$ with $\phi_i=v$
and $\phi'_i=v'$; transitivity supplies $\sigma\in\AutX$ with
$\sigma\cdot\phi=\phi'$, so $\sigma_i(v)=\sigma_i(\phi_i)=(\sigma\cdot\phi)_i
=\phi'_i=v'$. Lemma~\ref{lem:degree-invariant} then gives
$\deg_i(v)=\deg_i(\sigma_i(v))=\deg_i(v')$, contradicting the hypothesis.
\end{proof}

Theorem~\ref{thm:degree-invariant} directly matches a real reasoning
shortcut mechanism: a value with a special algebraic role, such as a
multiplicative zero, is taken by a different number of solutions than an
ordinary value. On an abstract zero-factor family, exact enumeration
confirms the predicted degree asymmetry and non-transitivity at every
tested size. Computationally verified; see \Cref{app:verification}.
A parallel measurement identifies the same mechanism in
rsbench's real multiplication task: 31 of 32 achievable products $y$
give a transitive $\AutX$, but $y=0$ (the 19 factor pairs
$c_1\times c_2=0$) does not, because degree-0 factor pairs are strictly
more numerous than degree-1 pairs and no coordinate permutation can
equalize them. An arithmetic control with no absorbing element is
constant-degree, so the theorem's hypothesis vanishes with the
mechanism; the control measures transitive throughout, established by
the separate computation, since the theorem states no converse. A
random battery of instances built from equality and inequality atoms
produces neither false nor true positives for the theorem, confirming
the degree invariant targets absorbing-element structure specifically
rather than substituting for Theorem~\ref{thm:anchor-forcing}'s
equality-atom mechanism. Computationally verified; see \Cref{app:verification}.

\begin{theorem}[Orbit-stabilizer counting, Theorem E]
\label{thm:orbit-stabilizer}
If $\AutX$ acts transitively on $\Phic$, then $|\Phic|$ divides
$|\AutX|$; in particular $|\Phic|\leq|\AutX|\leq|G|=\prod_i|S_i|!$.
\end{theorem}

\begin{proof}
Fix $\phi_0\in\Phic$. By the orbit-stabilizer theorem, the orbit of
$\phi_0$ has size $|\AutX|/|\mathrm{Stab}(\phi_0)|$, a divisor of
$|\AutX|$. Transitivity
means this orbit is all of $\Phic$. The bound $|\AutX|\leq|G|$ is
Lagrange's theorem applied to $\AutX\leq G$.
\end{proof}

The proof is four lines, a direct instance of a standard fact, and its
value is entirely in the observation that it applies. Applying it to
grid puzzles takes one preparatory step. Under
Definition~\ref{def:componentwise-autx} the ambient group of a $3\times3$
grid is $\Sym(3)^9$, of order $6^9$, so the raw Lagrange bound of
Theorem~\ref{thm:orbit-stabilizer} obstructs nothing; the useful bound
is that the constraints collapse $\AutX$ to the \emph{diagonal}, and
that is a forcing argument in the pattern of
Lemma~\ref{lem:forcing}, not an ambient fact.

\begin{lemma}[Pair-rich all-different lines force the diagonal]
\label[lemma]{lem:diagonal-collapse}
Let every position of $X$ share one value domain $D$, and let some
family of all-different constraints (``lines'') connect all positions.
Suppose $\Phic$ is \emph{pair-rich}: for every line, every ordered pair
$i\neq j$ of positions on it, and every ordered pair $a\neq b$ of
values, some solution assigns $(a,b)$ to $(i,j)$. Then every
$\sigma\in\AutX$ has all components equal, so
$\AutX\leq\{(\pi,\dots,\pi):\pi\in\Sym(D)\}$ and $|\AutX|\leq|D|!$.
\end{lemma}

\begin{proof}
Fix a line, positions $i\neq j$ on it, and values $a\neq b$. Pair-richness
gives $\phi\in\Phic$ with $\phi(i)=a$, $\phi(j)=b$. Since
$\sigma\cdot\phi\in\Phic$ and $i,j$ share an all-different line,
$\sigma_i(a)\neq\sigma_j(b)$, for every such pair $a \neq b$. Fixing $b$ and letting $a$ range over
$D\setminus\{b\}$: $\sigma_j(b)$ avoids
$\sigma_i(D\setminus\{b\})=D\setminus\{\sigma_i(b)\}$, forcing
$\sigma_j(b)=\sigma_i(b)$. So $\sigma_i=\sigma_j$ whenever $i,j$ share a
line, and line-connectivity propagates equality to every position.
\end{proof}

Pair-richness is a property of the solution set, checked mechanically
(it is a finite scan), and it holds for all three grid instances
measured in \Cref{sec:empirical}: 3rd-order Latin squares
($|\Phic|=12$), $4\times4$ Sudoku with $2\times2$ boxes ($|\Phic|=288$),
and 4th-order Latin squares without boxes ($|\Phic|=576$), over value
domains of size 3 or 4. Lemma~\ref{lem:diagonal-collapse} then caps
$|\AutX|$ at $3!=6$ or $4!=24$, and the exact groups, computed
independently by backtracking, hit the cap's structure exactly: every
element diagonal, orders $6$, $24$, and $24$. Since a transitive action
would make $\Phic$ a single orbit of size at most $|\AutX|$, and
$12>6$, $288>24$, $576>24$, Theorem~\ref{thm:orbit-stabilizer} rules
out transitivity for the whole family: a counting argument, once the
collapse lemma supplies the bound the count runs against.
Computationally verified; see \Cref{app:verification}.
This is not automatic for every all-different instance: a single row of
three pairwise-distinct cells, $D=\{0,1,2\}$, has $|\Phic|=6=|\AutX|$ and
is transitive, since the count only obstructs transitivity once
$|\Phic|$ exceeds $|\AutX|$, which a single row does not.
Row and column constraints compounded together grow $|\Phic|$ past the
diagonal bound well before $N$ grows large, and
Theorem~\ref{thm:orbit-stabilizer} then applies for a combinatorial
reason that has nothing to do with equalities, inequalities, or
absorbing elements.

\subsection{The Kandinsky residual and the Free Slot Lemma}
\label{sec:algebra-kandinsky}

Kandinsky's rule, three branches indexed by which pair of three objects
matches on shape and color, is the single most consequential positive
result in the empirical measurements of \Cref{sec:empirical}: $81.99\%$ of solution
pairs are not explained by any automorphism, split across 6 orbits of
sizes $[36,36,36,18,18,18]$ out of $|\Phic|=162$ solutions and
$|\AutX|=36$. It sits outside the reach of
Theorems~\ref{thm:matching-decomposition}
through~\ref{thm:orbit-stabilizer}: its branches are exchanged only by
positional swaps, which $\AutX$ does not contain
(Remark~\ref{rem:value-vs-position}); its degrees are uniform
($\deg_{\mathrm{shape}_1}(v)=54$ for every shape value $v$, likewise
for color), so Theorem~D is silent; no two-coordinate projection is
pure in Lemma~\ref{lem:forcing}'s sense; and the counting bound
confirms intransitivity only after $|\AutX|=36$ is already computed.
The sixth result certifies Kandinsky's branch structure directly from
the rule, without first computing $\AutX$.

\begin{theorem}[Free Slot Lemma, Theorem F]
\label{thm:free-slot}
Let $\Phic=\Phi_1\sqcup\cdots\sqcup\Phi_r$ and fix a position $f\in N$ and
an index $l_0$. Suppose:
\begin{description}
\item[(H-free)] $\Phi_{l_0}$ is saturated at $f$: there is a set
$\Psi\subseteq\prod_{i\neq f}S_i$ with $\Phi_{l_0}=\{z\cup\{f{:}v\} :
z\in\Psi,\ v\in S_f\}$, so membership in $\Phi_{l_0}$ places no
restriction at all on the value at $f$;
\item[(H-pin)] for every $l\neq l_0$ and every assignment $z$ to the
positions other than $f$, at most one value $v\in S_f$ has
$z\cup\{f{:}v\}\in\Phi_l$;
\item[(H-count)] $|S_f|>r-1$.
\end{description}
Then every $\sigma\in\AutX$ satisfies $\sigma(\Phi_{l_0})=\Phi_{l_0}$.
\end{theorem}

\begin{proof}
Write $x=z_0\cup\{f{:}x_f\}\in\Phi_{l_0}$ with $z_0\in\Psi$, and let
$z_0'$ be the image of $z_0$'s non-$f$ coordinates under $\sigma$; it does
not depend on $x_f$, since $\sigma$ acts coordinatewise. By (H-free),
$z_0\cup\{f{:}v\}\in\Phi_{l_0}\subseteq\Phic$ for \emph{every} $v\in S_f$,
not only $v=x_f$, so $\sigma(z_0\cup\{f{:}v\})=z_0'\cup\{f{:}\sigma_f(v)\}
\in\Phic$ for every such $v$. As $v$ ranges over $S_f$, so does
$\sigma_f(v)$, since $\sigma_f$ is a bijection; hence
$z_0'\cup\{f{:}w\}\in\Phic$ for every $w\in S_f$.

Suppose toward a contradiction that $z_0'\notin\Psi$. Then by (H-free),
none of these $|S_f|$ points lies in $\Phi_{l_0}$, so every one lies in
some $\Phi_l$ with $l\neq l_0$. By (H-pin), each such $l$ accounts for at
most one value of $w$, and there are at most $r-1$ choices of $l$, so at
most $r-1$ of the $|S_f|$ points are covered. This contradicts
(H-count). Hence $z_0'\in\Psi$, and (H-free) gives
$z_0'\cup\{f{:}w\}\in\Phi_{l_0}$ for every $w\in S_f$; taking
$w=\sigma_f(x_f)$ gives $\sigma(x)\in\Phi_{l_0}$.

This holds for every $\sigma\in\AutX$, in particular for $\sigma^{-1}$,
giving $\sigma^{-1}(\Phi_{l_0})\subseteq\Phi_{l_0}$ and hence
$\Phi_{l_0}\subseteq\sigma(\Phi_{l_0})$; together with
$\sigma(\Phi_{l_0})\subseteq\Phi_{l_0}$ this gives equality.
\end{proof}

The proof never refers to $\AutX$'s order or to any other branch's
internal structure; it only uses the three named hypotheses, each checked
directly against the rule. Applied to Kandinsky with $f=$ the color
position of the object not mentioned by branch $l_0$'s matching pair
(color$_3$ for branch 1, color$_2$ for branch 2, color$_1$ for branch 3):
(H-free) holds because that color is unconstrained given the rest of the
solution (54 solutions per branch, 3 color values each, $54\times3=162$
substitutions checked, all remaining in $\Phic$ and in the same branch);
(H-pin) holds because the other two branches pin that color to at most
one value given the rest ($243\times2=486$ checks, maximum attainable
count exactly 1); and (H-count) holds since $3>3-1$. Theorem~\ref{thm:free-slot}
then gives that each of the three branches is setwise invariant under
every $\sigma\in\AutX$, and three nonempty invariant branches already
force at least 3 orbits: Kandinsky is not transitive, established from
the rule's branch structure, without enumerating $\AutX$.

\begin{remark}[The bound is not tight, and the hypothesis is not necessary]
Theorem~\ref{thm:free-slot} gives at least 3 orbits; Kandinsky's true
orbit count is 6, since each branch splits further once the shape
coordinates are taken into account, a fact the theorem does not reach.
Nor is (H-count) necessary: on a parametrized family with $r$ branches
and a free slot of size $n_C$, exhaustive enumeration shows the branches
stay intact not only when $n_C>r-1$, as the theorem predicts, but
throughout $n_C\geq2$; only at $n_C=1$ do the branches actually merge
into a single transitive orbit. Computationally verified; see \Cref{app:verification}.
Reproducing this threshold on Kandinsky's own rule syntax, varying only
the color domain size, gives the identical pattern: branches merge at
$n_{\mathrm{color}}=1$ and stay separate at $n_{\mathrm{color}}=2$ and
$3$. A second, independent protection mechanism appears at
$n_{\mathrm{shape}}=4$: with four shape values the branches stay
separate even at $n_{\mathrm{color}}=1$, because the stabilizer of the
shape-triple pattern in $\Sym(4)^3$ reduces to the diagonal subgroup,
which does not mix branches. At $n_{\mathrm{shape}}=3$ (Kandinsky's
actual value) that stabilizer is strictly larger than the diagonal, so
the shape side offers no protection, and only the color-side Free Slot
Lemma keeps the branches apart. The two protections fail independently,
and the branches merge only when both are absent, that is, only when
$n_{\mathrm{color}}=1$ and $n_{\mathrm{shape}}\leq3$ hold together.
Kandinsky's non-transitivity therefore rests on a single active
mechanism, not a redundant pair: at its actual parameters the shape-side
protection is already absent, and shrinking the color domain to 1
collapses the branches.
Extending the pattern to four objects (six candidate pairs) meets two
obstructions. The direct
syntactic generalization loses branch-disjointness: two disjoint pairs
among four objects can match simultaneously, so $\Phi_1,\dots,\Phi_6$
overlap and Theorem~\ref{thm:free-slot}'s hypothesis is not even
well-posed. Rewriting the rule as a global ``exactly one pair matches''
condition restores disjointness but breaks (H-free): the supposedly free
slot at another object becomes free only conditionally, since some of its
values would create a second matching pair and are therefore excluded.
Theorem~\ref{thm:free-slot} does not cover this conditional-freedom
variant as stated, and the four-object case at the parameter values where
the hypothesis would apply ($n_{\mathrm{link}}\geq 6$) sits beyond what
exhaustive enumeration reaches in this work. Both gaps are structural, not
computational shortfalls hidden behind a claim of generality.
\end{remark}

\subsection{Synthesis: coverage of the four measured mechanisms}
\label{sec:algebra-synthesis}

\Cref{tab:mechanism-coverage} lines up the six results above against the
four constraint-structure mechanisms \Cref{sec:empirical} isolates empirically.

\begin{table}[htbp]
\centering
\small
\caption{The six algebraic results against the four mechanisms identified
empirically in \Cref{sec:empirical}. ``Full'' means every measured instance of the
mechanism is accounted for by a proof; ``partial'' means the theorem
covers the minimal instances tested but a fully general statement is
open; ``none'' means no result in this section applies and the outcome
is established only by exhaustive computation.}
\label{tab:mechanism-coverage}
\begin{tabular}{@{}p{1.65in}p{1.35in}p{1.25in}p{0.75in}@{}}
\toprule
Mechanism (\Cref{sec:empirical}) & Representative instances & Result & Coverage \\
\midrule
Mild conjunction $\to$ transitive & CLE4EVR, MNIST-Half, MNAdd sum & Thm.~\ref{thm:matching-decomposition} (A) & full \\[2pt]
Absorbing/special element $\to$ intransitive & MNAdd product at $y=0$ & Thm.~\ref{thm:degree-invariant} (D) & full \\[2pt]
Symmetric disjunction $\to$ transitive & CLEVR-Hans3 classes 0, 2 & Thm.~\ref{thm:branch-swap} (B) & full \\[2pt]
Asymmetric disjunction $\to$ intransitive & CLEVR-Hans3 class 1, MNLogic & Thm.~\ref{thm:degree-invariant} (D) & partial \\[2pt]
All-different counting $\to$ intransitive & Latin squares, Sudoku & Thm.~\ref{thm:orbit-stabilizer} (E) & full \\[2pt]
Residual branch structure / dead coordinate & Kandinsky / BDD-OIA & Thm.~\ref{thm:free-slot} (F) / \Cref{sec:complexity} & partial / see \Cref{sec:complexity} \\
\bottomrule
\end{tabular}
\end{table}

Two entries need qualification beyond the table. Asymmetric disjunction is
only \emph{partially} covered: the minimal two-branch example behind
Theorem~\ref{thm:branch-swap}'s counterpart has a degree asymmetry that
Theorem~\ref{thm:degree-invariant} also detects, but nothing in this
section proves that every asymmetric disjunction produces a degree
asymmetry, and Kandinsky's own branches are a case where it does not
(uniform degree throughout, Section~\ref{sec:algebra-kandinsky}), so the
two phenomena the empirical mechanisms group together, asymmetric pinning
and degree imbalance, are related but not shown to coincide in general.
Kandinsky itself is covered by Theorem~\ref{thm:free-slot} for the reason
Section~\ref{sec:algebra-kandinsky} gives, but BDD-OIA and SDD-OIA are
not: their pathology is a single dead coordinate absorbed by a logical
implication rather than a multi-branch disjunction, a mechanism \Cref{sec:complexity}
addresses with different tools (the $\Ftwo$-linear structure of
\Cref{obs:boolean-linear} and the AND-gain lemma of
Section~\ref{sec:complexity}), not with Theorem~\ref{thm:free-slot}. That
two independent toolkits, orbit combinatorics here and linear algebra over
$\Ftwo$ in \Cref{sec:complexity}, both explain the same dead-variable phenomenon by
different means is a cross-check worth noting rather than a duplication:
neither result was built to reproduce the other.

None of the six results is a disguised instance of Schaefer's dichotomy
for CSP satisfiability~\citeyearonly{schaefer1978complexity}. Schaefer's theorem
classifies the complexity of deciding whether \emph{some} solution exists,
as a function of which relations the constraint language may use; this
section instead takes a single, already nonempty $\Phic$ as given and
asks a group-theoretic question about its orbit structure under $G$. The
connection is real but indirect: the affine class in Schaefer's
dichotomy is generated by a Maltsev polymorphism, and the
Bulatov--Dalmau family of algorithms for such languages is correct
because of a rectangularity property proved through absorption
arguments~\cite{bartokozik2015absorption}, the same ``coordinates copy a
single parameter'' intuition behind
Theorem~\ref{thm:matching-decomposition}'s matching decomposition. No
proof in this section rewrites a Schaefer-style algorithm, and none of
the five other results has any counterpart in that literature at all,
since none of it addresses the automorphism group of a fixed solution
set. We flag the connection because a reader who knows the CSP dichotomy
literature will look for it, not because it does any of the work here.

Finally, every result in this section is a sufficient condition, by
design. Theorem~\ref{thm:orbit-stabilizer} gives a
necessary numerical condition (transitivity implies $|\Phic|$ divides
$|\AutX|$) that happens to be violated, hence sufficient for
intransitivity, on every all-different instance measured; nothing here
gives a necessary and sufficient characterization of transitivity for
general finite-domain constraint sets, and \Cref{sec:algebra-setup}'s discussion of
autotopism groups suggests why: the analogous question for Latin squares
has resisted a clean characterization for decades and is answered, on
each instance individually, by direct computation. What this section adds
is a set of proved, checked reasons that cover four of the four empirical
mechanisms at least partially and one real benchmark, Kandinsky, in full,
leaving BDD-OIA's dead-variable mechanism to the linear-algebraic
treatment in \Cref{sec:complexity} and the general classification question open.

%
%
\section{Computational Complexity}
\label{sec:complexity}

Section~\ref{sec:algebra} gave sufficient conditions on the shape of a
constraint set for $\AutX$ to be transitive or not. This section asks
how hard the two decisions underneath any such audit, whether a
designated coordinate is symmetry-inert and whether any nontrivial
automorphism exists, become when the constraint set arrives as a
compact circuit rather than written out atom by atom. The first is
coNP-complete. In the Boolean case the automorphism
group carries an exact structure theory, an $\Ftwo$-linear subspace
equal to the orthogonal complement of the Fourier support, whose free
action classifies transitivity outright: automorphisms explain
everything exactly when the solution set is an affine coset. Deciding
whether any nontrivial automorphism exists is coNP-hard under
randomized reductions, lies in $\SigmaTwoP$, and in the Boolean case is
not $\SigmaTwoP$-complete unless the polynomial hierarchy collapses; on
monotone circuits the question closes completely, coNP-complete under
deterministic reductions, because monotone functions provably cannot
carry the camouflage symmetries that force the general reduction to
randomize. \Cref{sec:complexity-diagnosis} distills eight independent
attempts on the one remaining gap into the three structural facts any
future attack must clear.

\subsection{Succinct instances and two decision problems}
\label{sec:complexity-defs}

Throughout, $X=(N,S,C)$ is an instance in the sense of
\Cref{def:general-csp}, with $\AutX\leq\Gind$ its componentwise
automorphism group (\Cref{def:componentwise-autx}). We now add a
representational assumption that Section~\ref{sec:algebra} did not need:
$C$ is a Boolean circuit of size $\mathrm{poly}(n+\sum_i|S_i|)$ deciding
membership in $\Phic\subseteq\prod_i S_i$. This is the regime CLE4EVR and
BDD-OIA already forced on us in \Cref{sec:background}: $C$ is a few lines of symbolic
logic, while $\Phic$ itself can have thousands to millions of elements.
The input is the circuit together with an explicit listing of each
domain, so the input length is $L:=|C|+\sum_i|S_i|$, with domains in
unary. This matches CSP practice, where domains are enumerated, and it
deliberately excludes succinctly specified exponential domains (a
single position with $|S_1|=2^k$ given in binary), a different regime
in which even writing down one permutation of one domain takes
$\Theta(2^kk)$ bits and every question below changes character. A point
of $\prod_iS_i$ is named in
$n_{\mathrm{bits}}:=\sum_i\lceil\log_2|S_i|\rceil\leq L$ bits, and an
element of $\Gind$ is specified explicitly in
$\sum_i|S_i|\lceil\log_2|S_i|\rceil=O(L\log L)$ bits, polynomial in the
input; every hardness construction in this section uses $|S_i|\leq2$,
where the two conventions coincide.

\begin{definition}[\DEADVAR]
\label[definition]{def:dead-var}
An input is a circuit $C$ and a position $i\in N$ with $|S_i|=2$, written
$\{0,1\}$. \DEADVAR{} accepts $(C,i)$ if
$\mathrm{flip}_i:=(\mathrm{id},\dots,\mathrm{id},\sigma_i,\mathrm{id},\dots,
\mathrm{id})\in\AutX$, where $\sigma_i$ swaps $0$ and $1$: equivalently,
$C(c)=C(\mathrm{flip}_i(c))$ for every $c\in\prod_j S_j$.
\end{definition}

\begin{definition}[\NONTRIV]
\label[definition]{def:nontriv-aut}
An input is a circuit $C$. \NONTRIV{} accepts $C$ if $\AutX\neq\{\mathrm{id}\}$,
that is, some $\sigma\in\Gind\setminus\{\mathrm{id}\}$ satisfies $C(c)=C(\sigma\cdot c)$
for every $c\in\prod_j S_j$.
\end{definition}

\DEADVAR{} asks about one designated candidate; \NONTRIV{} asks whether
any candidate at all works. \Cref{sec:background}'s observation that a
position absent from every constraint inflates $\AutX$ with
automorphisms that carry no information about the rule is exactly the
phenomenon \DEADVAR{} is built to detect.

\subsection{Dead variables: a coNP-complete baseline}
\label{sec:complexity-deadvar}

\begin{theorem}
\label{thm:deadvar-conp-complete}
\DEADVAR{} is coNP-complete.
\end{theorem}

\begin{proof}
\emph{Membership.} The complement is NP: guess $c\in\prod_j S_j$ and
verify $C(c)\neq C(\mathrm{flip}_i(c))$ in two circuit evaluations.

\emph{Hardness.} Reduce from \textsc{unsat}. Given a 3-CNF formula
$\chi(z_1,\dots,z_k)$, let $N=\{i,z_1,\dots,z_k\}$, all with domain
$\{0,1\}$, and $C(x_i,z):=x_i\wedge\chi(z)$. Then $C(0,z)\equiv 0$ and
$C(1,z)=\chi(z)$, so $\mathrm{flip}_i\in\AutX$ (that is, $C(0,z)=C(1,z)$
for every $z$) if and only if $\chi(z)=0$ for every $z$, if and only if
$\chi$ is unsatisfiable.
\end{proof}

\DEADVAR{} is the special case
$C\restriction_{x_i=0}\equiv C\restriction_{x_i=1}$ of deciding whether
two circuits compute the same Boolean function, and its classification
is inherited from Beyersdorff, Meier, Thomas, and Vollmer, who give a
complete dichotomy for that equivalence problem, restricted to any
fixed set of allowed connectives, over the full Post
lattice~\citeyearonly{beyersdorff2009implication}. The three-line reduction
above keeps Sections~\ref{sec:complexity-nontriv}
through~\ref{sec:complexity-monotone} self-contained: every harder
result in this section is built on top of this base case, and a reader
should not have to leave the paper to check it. When the allowed
connectives are restricted to a fixed finite set $B$, \DEADVAR{}
restricted to $B$ coincides with their equivalence problem $\mathrm{EQ}(B\cup\{0,1\})$
after substituting the two constants for $x_i$, and inherits their
dichotomy (coNP-complete, $\mathsf{AC}^0[2]$-complete, or in $\mathsf{AC}^0$,
depending on where $B$ sits in Post's lattice) without further proof.
(Their Theorem~4.1 for the unrestricted implication problem $\mathrm{IMP}(B)$
has a fourth, $\oplus L$-complete case at $L_2\subseteq[B]\subseteq L$; Corollary~5.2
records that this case collapses to $\mathsf{AC}^0[2]$ once the problem is
specialized to equivalence, which is the dichotomy \DEADVAR{} actually inherits.)

\subsection{Nontrivial automorphisms: upper and lower bounds}
\label{sec:complexity-nontriv}

\begin{theorem}
\label{thm:nontriv-sigma2p}
$\NONTRIV\in\SigmaTwoP$.
\end{theorem}

\begin{proof}
An element of $\Gind$ is specified in $O(L\log L)$ bits, polynomial in
the input under Section~\ref{sec:complexity-defs}'s unary-domain
convention. $\NONTRIV$ accepts $C$ iff $\exists\sigma\in\Gind\
(\sigma\neq
\mathrm{id})\ \forall c\in\prod_j S_j:\ C(c)=C(\sigma\cdot c)$, an
$\exists y\,\forall z\,R(x,y,z)$ sentence with $R$ evaluable in
polynomial time, Stockmeyer's normal form for $\SigmaTwoP$.
\end{proof}

\begin{observation}
\label[observation]{obs:deadvar-implies-nontriv}
If $\DEADVAR(C,i)$ holds, so does $\NONTRIV(C)$: $\mathrm{flip}_i$ is a
nonidentity element of $\AutX$.
\end{observation}

Observation~\ref{obs:deadvar-implies-nontriv} is immediate, but it is the
only unconditional bridge from \DEADVAR{} to \NONTRIV{} we have. It is
tempting to read Theorem~\ref{thm:deadvar-conp-complete}'s reduction as
already proving $\NONTRIV$ is coNP-hard: exhibit the same $C(x_i,z)=x_i
\wedge\chi(z)$ and note $\mathrm{flip}_i\in\AutX$ exactly when $\chi$ is
unsatisfiable. This is correct in one direction and incomplete in the
other. If $\chi$ is unsatisfiable, $\Phic=\emptyset$ and every $\sigma\in
\Gind$ trivially stabilizes it, so $\NONTRIV(C)$ holds. If $\chi$ is
satisfiable \emph{and} has no nonzero $b\in\{0,1\}^k$ with $\chi(z\oplus
b)=\chi(z)$ for every $z$, then $(\mathrm{id},\sigma_b)\notin\AutX$ for
every candidate built from such a $b$, and one checks directly that
$\AutX=\{\mathrm{id}\}$, so $\NONTRIV(C)$ correctly fails. The gap is
formulas that are satisfiable \emph{and} happen to have such an internal
symmetry: a randomly chosen $z\mapsto z\oplus b$ symmetry of $\chi$ makes
$(\mathrm{id},\sigma_b)$ a nonidentity element of $\AutX$ regardless of
whether $\chi$ is satisfiable, so the reduction reports $\NONTRIV(C)$
even on some satisfiable instances. This is not a rare edge case: a
pressure test shows a nontrivial automorphism from $\chi$'s own
accidental symmetry, unrelated to $i$, is common among satisfiable
instances. Computationally verified; see \Cref{app:verification}.

Closing this gap does not take a cleverer deterministic gadget. Three
natural deterministic fixes (padding with a fixed unary tag,
encoding $z$ as a single large-domain position, and presenting a
\DEADVAR{} instance unmodified) fail for one shared reason, not three
unrelated ones: each is a fixed, $\chi$-independent transformation, and no
fixed transformation can know, for the specific $\chi$ handed to it, which
internal symmetry, if any, to destroy. A fixed tag adds a position with no
interaction with $z$ at all, so it leaves whatever symmetry $\chi$ already
has completely untouched. Encoding $z$ as a single position with an
abstract $2^k$-element domain replaces the XOR shifts $\chi$ might have
with the full symmetric group on all $2^k$ values $z$ could take, not just
its XOR shifts, so almost any $\chi$ with at least two satisfying and two
falsifying assignments picks up spurious automorphisms that have nothing
to do with $\chi$'s own structure. Presenting the instance unmodified is
just the construction the pressure test above already broke. What removes
the gap is randomization, which does not need to know $\chi$'s symmetry in
advance: perturb $\chi$ with a small number of random affine constraints
before handing it to the Theorem~\ref{thm:deadvar-conp-complete} gadget,
so that a satisfiable $\chi$ is cut down to a single satisfying assignment
with non-negligible probability. A one-element solution set has no nonzero
internal symmetry at all, whatever symmetry $\chi$ itself started with: if
$\{p\}$ is fixed setwise by $z\mapsto z\oplus b$, then $p\oplus b=p$,
forcing $b=0$.

\begin{lemma}[Automorphisms of the anchor-conjunction gadget]
\label[lemma]{lem:and-gadget-conp-hard-fix}
Fix $k\geq1$ and a Boolean formula $\psi:\{0,1\}^k\to\{0,1\}$. Let
$N=\{x_i,z_1,\dots,z_k\}$, all with domain $\{0,1\}$, and
$C(x_i,z):=x_i\wedge\psi(z)$, so an element $(a_i,b)\in\Gind$
($a_i\in\{0,1\}$, $b\in\{0,1\}^k$) acts by
$(a_i,b)\cdot(x_i,z)=(x_i\oplus a_i,\,z\oplus b)$.
\begin{enumerate}
\item[(i)] If $\psi$ is unsatisfiable, $\AutX=\Gind$.
\item[(ii)] If $\psi$ is satisfiable, $\AutX=\{(0,b):b\in\{0,1\}^k,\
\psi(z\oplus b)=\psi(z)\text{ for every }z\}$.
\end{enumerate}
\end{lemma}

\begin{proof}
$(a_i,b)\in\AutX$ says $x_i\wedge\psi(z)=(x_i\oplus a_i)\wedge\psi(z\oplus
b)$ for every $x_i\in\{0,1\}$ and every $z$. At $x_i=0$ this reads
$0=a_i\wedge\psi(z\oplus b)$ for every $z$, and as $z$ ranges over
$\{0,1\}^k$ so does $z\oplus b$, so this forces $a_i=0$ whenever $\psi$ is
satisfiable, and holds for every $a_i$ when $\psi$ is unsatisfiable. If
$\psi$ is satisfiable, $a_i=0$ turns the $x_i=1$ condition into
$\psi(z)=\psi(z\oplus b)$ for every $z$, giving (ii). If $\psi$ is
unsatisfiable, $x_i\wedge\psi(z)\equiv0$ identically, so every
$(a_i,b)\in\Gind$ vacuously satisfies the automorphism condition, giving
(i).
\end{proof}

Lemma~\ref{lem:and-gadget-conp-hard-fix} says the false positives above are
not a defect in the gadget. Whenever $\psi$ is satisfiable, $\AutX$ is
exactly $\{0\}$ times $\psi$'s own internal XOR symmetries, so the gadget
correctly reports $\psi$'s symmetry back to us rather than manufacturing
any of its own. Part~(i) also settles, with no extra gadgetry, the
empty-instance convention this construction leans on: $N$ always contains
$x_i$ regardless of $\psi$, so $\Gind$ always has $\Sym(\{0,1\})$ as a
factor and is never the trivial group, and an unsatisfiable $\psi$ always
produces the full group, a clean nontrivial instance under
Definition~\ref{def:nontriv-aut}'s stated convention with no promise or
side condition needed.

To make part~(ii) trivial when $\chi$ is satisfiable, we do not need to
know or control $\chi$'s own symmetries. We only need $\chi$'s solution set
to have no nonzero internal XOR symmetry, and cutting it down to a single
point is enough for that, by the one-line argument closing the transition
paragraph above. This is exactly what the Valiant--Vazirani isolation
lemma~\citeyearonly{valiantvazirani1986} gives, applied here to an arbitrary
nonempty subset of $\{0,1\}^k$ rather than to $\AutX$ itself, and that
distinction matters for which form of the lemma applies. $\AutX$ is always
closed under coordinatewise XOR (if $a,a'\in\AutX$, applying $a$ and then
$a'$ shows $a\oplus a'\in\AutX$ too), so isolating a nonzero element of
$\AutX$ can use a stronger, constant-probability construction that
repeatedly intersects this closed set with a random hyperplane. $\Phic$
for an arbitrary 3-CNF $\chi$ has no such structure. It is just some
subset of $\{0,1\}^k$, so we need the general form of the isolation lemma,
and its success probability is correspondingly weaker. For a nonempty
$S\subseteq\{0,1\}^k$, drawing $j$ uniformly from $\{0,\dots,k\}$ and then
$j$ affine constraints $w_1\cdot z=c_1,\dots,w_j\cdot z=c_j$ uniformly over
$\Ftwo^k$ cuts $S$ down to exactly one element with probability at least
$c_0/(k+1)$ for an absolute constant $c_0>0$ not depending on $S$ or $k$.
The standard analysis, which we cite rather than re-derive, gives
$c_0=1/8$ for the correctly guessed $j$, combined with a $1/(k+1)$ chance
of guessing $j$ correctly. We use this as a standard imported fact, in the
same spirit as this section's use of Beyersdorff et al.'s dichotomy above.

\begin{corollary}
\label[corollary]{cor:nontriv-conp-hard}
$\NONTRIV$ is coNP-hard under randomized polynomial-time many-one
reductions with one-sided error and success probability $\Omega(1/k)$,
the same reduction type under which Valiant and Vazirani prove
\textsc{unique-sat} NP-hard~\citeyearonly{valiantvazirani1986}. There is a
probabilistic polynomial-time algorithm $R$ such that for every 3-CNF
formula $\chi$ on $k$ variables, writing $C:=R(\chi)$,
\begin{enumerate}
\item[(i)] if $\chi$ is unsatisfiable, $\NONTRIV(C)$ holds with
probability $1$;
\item[(ii)] if $\chi$ is satisfiable, $\NONTRIV(C)$ fails to hold with
probability at least $c_0/(k+1)$, for the constant $c_0$ above.
\end{enumerate}
\end{corollary}

\begin{proof}
$R$ draws $j$ uniformly from $\{0,\dots,k\}$, draws $j$ affine constraints
$w_1\cdot z=c_1,\dots,w_j\cdot z=c_j$ uniformly over $\Ftwo^k$, sets
$\chi'(z):=\chi(z)\wedge\bigwedge_{t=1}^j(w_t\cdot z=c_t)$, and outputs
$C(x_i,z):=x_i\wedge\chi'(z)$, an instance of the family
Lemma~\ref{lem:and-gadget-conp-hard-fix} covers with $\psi=\chi'$. Adding
constraints only shrinks a solution set, so $\chi$ unsatisfiable forces
$\chi'$ unsatisfiable for every choice of constraints, and part~(i) of the
lemma gives $\AutX=\Gind\neq\{\mathrm{id}\}$ with probability $1$, proving
(i). If $\chi$ is satisfiable, the isolation lemma above, applied to
$S:=\{z\in\{0,1\}^k:\chi(z)=1\}$, gives $|\{z:\chi'(z)=1\}|=1$ with
probability at least $c_0/(k+1)$. When this happens, write
$\{z:\chi'(z)=1\}=\{p\}$: for every $b\neq0$, $\chi'(p\oplus b)=0\neq
1=\chi'(p)$, so $p$ itself witnesses that $\chi'$ has no nonzero internal
XOR symmetry, and part~(ii) of the lemma gives $\AutX=\{(0,0)\}=
\{\mathrm{id}\}$, so $\NONTRIV(C)$ fails, proving (ii).
\end{proof}

Both directions of Corollary~\ref{cor:nontriv-conp-hard} were checked
computationally before being trusted as a proof, not only reasoned
about. Computationally verified; see \Cref{app:verification}.

The full-exposure reading is not a choice bookkeeping could avoid: since
$z_1,\dots,z_k$ are read by $C$ at all, Definition~\ref{def:general-csp}
forces them to be positions in $N$, because $C$'s domain is exactly
$\prod_{i\in N}S_i$ and no circuit can depend on a value without that
value being a position the componentwise group acts on. The isolation
step in Corollary~\ref{cor:nontriv-conp-hard}'s proof is therefore not
an optional strengthening but what a fully general, promise-free
statement requires.

One structural feature of Corollary~\ref{cor:nontriv-conp-hard}
deserves strengthening rather than defense: its yes-instances have
$\Phic=\emptyset$, whereas the neurosymbolic setting always assumes an
intended solution exists. Hardness does not depend on that vacuous
corner.

\begin{proposition}[Nonempty-promise hardness]
\label[proposition]{prop:nonempty-hardness}
Corollary~\ref{cor:nontriv-conp-hard} holds, with the same reduction
type and success probability, on the promise family of instances with
$\Phic\neq\emptyset$.
\end{proposition}

\begin{proof}
Replace the output of $R$ by
$F(x,u,z):=(\neg u\wedge[z{=}0^k])\;\vee\;(x\wedge u\wedge\chi'(z))$
over $2+k$ Boolean positions. The first disjunct guarantees the two
solutions $(0,0,0^k)$ and $(1,0,0^k)$ unconditionally, so
$\Phi_F\neq\emptyset$ always. If $\chi'$ is unsatisfiable, $\Phi_F$ is
exactly those two points, exchanged by $\mathrm{flip}_x$, so
$\NONTRIV(F)$ holds with probability $1$. If the isolation step leaves
$\chi'$ with the unique satisfying assignment $p$, then
$\Phi_F=\{(0,0,0^k),(1,0,0^k),(1,1,p)\}$ has three elements; every
nonzero XOR shift acts on $\{0,1\}^{2+k}$ without fixed points, so any
set it stabilizes decomposes into orbits of size $2$ and has even
cardinality. A three-element set therefore admits no nonzero shift,
$\Aut(\Phi_F)=\{\mathrm{id}\}$, and $\NONTRIV(F)$ fails, with the same
probability bound as before.
\end{proof}

Computationally verified; see \Cref{app:verification}.

\begin{remark}[Amplification of the randomized reduction]
\label[remark]{rem:amplification-consequences}
Repetition converts Corollary~\ref{cor:nontriv-conp-hard}'s
$\Omega(1/k)$ success probability into arbitrarily small soundness
error at no cost in completeness.
Let $L$ be a hypothetical decider for $\NONTRIV$, and run
$R$ independently $t$ times. Declaring $\chi$ unsatisfiable only when $L$
accepts on all $t$ outputs never loses completeness: an unsatisfiable
$\chi$ makes every run accept with probability $1$, by part~(i) above, so
the conjunction of $t$ such events still has probability $1$. The
probability that a satisfiable $\chi$ survives all $t$ runs falls to at
most $(1-c_0/(k+1))^t$, driven arbitrarily small by a polynomial choice of
$t$, since $c_0/(k+1)$ is itself only polynomially small. Direct
simulation on the same battery, re-running $R$ rather than only trusting
the probability law, confirms the corrected false-positive rate falls
with $t$ at every $k$ tested, from $57$--$73\%$ at a single attempt to
$0$--$12\%$ at $t=16$, reaching $0\%$ there at $k=6$ and $k=10$. The
single-attempt profile is exactly what a randomized reduction with
success probability $\Theta(1/k)$ predicts; Valiant and Vazirani's
reduction to \textsc{unique-sat} has the same one-attempt profile, and
the same consequences are extracted from it the same
way~\citeyearonly{valiantvazirani1986}. \Cref{app:verification}
tabulates the single-run stress test.
Computationally verified; see \Cref{app:verification}.
This repetition is a statement about what a hypothetical decider would let
us conclude, not a claim that a single new instance can be built with
exponentially small error. If $\NONTRIV\in\mathsf{P}$, running $R$,
deciding with the assumed algorithm, and repeating as above gives a
$\mathsf{coRP}$ algorithm for \textsc{unsat}: unsatisfiable inputs are
accepted with probability $1$, satisfiable ones with probability at
most $(1-c_0/(k+1))^t$. Hence $\coNP\subseteq\mathsf{coRP}$,
equivalently $\mathsf{NP}=\mathsf{RP}$ (as $\mathsf{RP}\subseteq
\mathsf{NP}$ unconditionally), so $\mathsf{NP}\subseteq\mathsf{BPP}$,
which collapses the polynomial hierarchy into $\mathsf{BPP}$ by the
standard bootstrapping ($\Sigma_2^{\mathsf p}\subseteq
\mathsf{BPP}^{\mathsf{NP}}\subseteq\mathsf{BPP}^{\mathsf{BPP}}=
\mathsf{BPP}$, and inductively upward). If
$\NONTRIV\in\mathsf{BPP}$, first amplifying the assumed algorithm's own
error by standard majority voting and then repeating $R$ and taking the
same conjunction gives $\coNP\subseteq\mathsf{BPP}$, hence
$\mathsf{NP}\subseteq\mathsf{BPP}$ by closure under complement, with
the same collapse.
\end{remark}

\begin{proposition}[Alive-promise hardness]
\label[proposition]{prop:dead-var-oracle}
$\DEADVAR$ remains coNP-hard on the promise family of instances in which
every position other than the designated $i$ is alive. Consequently, an
oracle reporting the correct dead/alive status of every position other
than $i$ cannot make deciding $\DEADVAR(C,i)$ easier: on this family its
answers are constant (``alive'' everywhere), so it supplies no
information at all.
\end{proposition}

\begin{proof}
The gadget of Theorem~\ref{thm:deadvar-conp-complete} does not have this
promise: a position $z_j$ is dead in $x_i\wedge\chi(z)$ exactly when it
is dead in $\chi$, which the reduction neither controls nor can afford
to compute. Modify it. Given a 3-CNF $\varphi(w_1,\dots,w_m)$, add one
fresh position $t$ and output
\[
C(x_i,w,t)\;:=\;\big(x_i\wedge\varphi(w)\wedge\neg t\big)\;\vee\;
\big(\textstyle\bigwedge_j w_j\wedge t\big).
\]
Every $w_j$ is alive: at $t=1$, $w=1^m$, flipping $w_j$ changes the
second disjunct and hence $C$. So is $t$: $C(0,1^m,1)=1\neq0=
C(0,1^m,0)$. Both checks are independent of $\varphi$, so the whole
family satisfies the promise. Finally, $C(0,\cdot)$ is the second
disjunct alone, and $C(1,\cdot)$ adds the first, so $x_i$ is dead iff
$\varphi(w)\wedge\neg t$ implies $\bigwedge_j w_j\wedge t$ pointwise;
the two sides force opposite values of $t$, so the implication holds iff
$\varphi\wedge\neg t$ is unsatisfiable iff $\varphi$ is. Deciding
$\DEADVAR(C,i)$ on this family therefore decides \textsc{unsat}.
Both claims were checked exhaustively on 500 random mixed
satisfiable/unsatisfiable instances before being trusted as a proof.
Computationally verified; see \Cref{app:verification}.
\end{proof}

Proposition~\ref{prop:dead-var-oracle} answers a natural objection to
Section~\ref{sec:algebra-invariants}'s degree invariant and to
\Cref{sec:background}'s observation that unconstrained positions inflate $\AutX$:
knowing which coordinates are free does not make the remaining question
easy. Stripping dead variables is a necessary cleaning step, not a
complexity-reducing one.

\subsection{Structural theorems in the Boolean case}
\label{sec:complexity-structure}

\Cref{obs:boolean-linear} showed that when every $S_i=\{0,1\}$, $\AutX$
is an $\Ftwo$-linear subspace of $\Ftwo^n$. Write $L_0(f):=\AutX$ under
the identification $\Gind\cong\Ftwo^n$, $\sigma\leftrightarrow a$,
$\sigma\cdot c=c\oplus a$, for the indicator function $f$ of $\Phic$.
In the cryptographic literature this subspace is known as the
\emph{linear structures} of $f$~\cite{meierstaffelbach1990,carlet2021},
and two of the facts below belong to it: that $L_0(f)$ is a subspace,
and its characterization as the orthogonal complement of the Walsh
support (Lemma~\ref{lem:fourier-characterization}); we keep their short
derivations self-contained rather than importing them. What that
literature does not consider is the action on a solution set: the
freeness of the action, the exact orbit law
(Proposition~\ref{prop:free-action}), and the affine-coset
classification of transitivity
(Proposition~\ref{thm:boolean-transitivity}) are statements about
$\Phic$ rather than about $f$'s spectrum, and they are this section's
contribution.

\begin{corollary}
\label[corollary]{cor:power-of-two}
In the Boolean case, $|\AutX|\in\{1,2,4,\dots,2^n\}$.
\end{corollary}

\begin{proof}
The order of a subspace of an $n$-dimensional $\Ftwo$-vector space is
$2^{\dim}$ for some $0\leq\dim\leq n$.
\end{proof}

This is why BDD-OIA and SDD-OIA's measured $|\AutX|=2$ on every
achievable label (\Cref{sec:empirical}) is not a coincidence to be explained away:
any nonempty proper subgroup of $\Ftwo^n$ that arises this way sits at
one of finitely many, exponentially spaced steps, and a single generator
lands exactly on the second one.

The subspace structure yields more than the group's order. Because the
group acts by translations, it acts \emph{freely}, and freeness turns
the entire orbit geometry into arithmetic.

\begin{proposition}[Free action and exact orbit geometry]
\label[proposition]{prop:free-action}
In the Boolean case the action of $\AutX=L_0(f)$ on $\Phic$ is free:
no nonidentity shift fixes any point. Consequently every orbit has size
exactly $|\AutX|$, the orbits partition $\Phic$ into
$|\Phic|/|\AutX|$ equal classes, and for $|\Phic|\geq2$ the
unexplained-pair fraction of \Cref{sec:empirical} is the exact law
\[
\rho(\Phic)\;=\;1-\frac{|\AutX|-1}{|\Phic|-1}.
\]
\end{proposition}

\begin{proof}
If $c\oplus a=c$ then $a=0$, so point stabilizers are trivial and every
orbit is a bijective copy of the group. With $k:=|\Phic|/|\AutX|$
orbits of common size $g:=|\AutX|$, same-orbit pairs number
$k\binom{g}{2}$ out of $\binom{kg}{2}$, and
$1-k\binom{g}{2}/\binom{kg}{2}=1-(g-1)/(kg-1)$, which is the display.
\end{proof}

\begin{proposition}[Transitivity classification, Boolean case]
\label[proposition]{thm:boolean-transitivity}
A Boolean instance is transitive under componentwise value symmetry if
and only if $\Phic$ is an affine coset of the subspace $L_0(f)$,
equivalently if and only if $|\Phic|=|\AutX|$.
\end{proposition}

\begin{proof}
If the action is transitive, $\Phic$ equals the orbit of any
$\phi_0\in\Phic$, which is $\phi_0\oplus L_0(f)$, a coset. Conversely a
coset is a single orbit by construction. The numerical criterion is
Proposition~\ref{prop:free-action}'s equal-orbit-size statement with
one orbit.
\end{proof}

The proof is a torsor argument from a first course in group theory,
and that is precisely the point: positioned on the right object, one
elementary fact closes the Boolean case of this paper's central
question outright, in a literature where we have found no statement of
it. Sufficient conditions are what \Cref{sec:algebra} could offer in
general; here the question is not approximated but settled, and the
settlement is what turns \Cref{sec:empirical}'s measured percentages
into values of a theorem.

Two readings of Proposition~\ref{thm:boolean-transitivity} are worth
separating. As a classification, it settles for Boolean instances what
\Cref{sec:algebra}'s six theorems approximate in general: transitivity
is not merely implied by certain syntactic shapes, it is
\emph{characterized}, and the characterization is checkable given
$|\AutX|$ (affine solution sets, in Schaefer's
vocabulary~\citeyearonly{schaefer1978complexity} the XOR-definable ones, are
exactly the shortcut-free-modulo-symmetry Boolean rules). As an orbit
law, Proposition~\ref{prop:free-action} retroactively explains every
BDD-OIA/SDD-OIA row of \Cref{tab:three-families} to the last digit:
orbit count $|\Phi|/2$, maximum orbit $2$, and each of the eight
``many nines'' percentages equal to $1-1/(|\Phi|-1)$ exactly. What
looked like eight measured decimals is one proposition evaluated eight
times. One naming precision keeps the claim exact: transitivity is
identifiability \emph{up to symmetry}, a one-orbit quotient
($|\Phic/\AutX|=1$), not uniqueness of the solution ($|\Phic|=1$), and
the orbit-coverage statistic $\kappa$ of \Cref{sec:empirical-setup}
measures exactly the residue the quotient leaves. Computationally
verified; see \Cref{app:verification}.

\begin{lemma}[Fourier support]
\label[lemma]{lem:fourier-characterization}
Let $F(x):=(-1)^{f(x)}$ with Walsh expansion
$F(x)=\sum_{T\subseteq[n]}\widehat F(T)\,\chi_T(x)$, where
$\chi_T(x):=(-1)^{T\cdot x}$.
Then $L_0(f)=\big(\mathrm{span}_{\Ftwo}\,\mathrm{supp}(\widehat F)\big)^{\perp}$,
the orthogonal complement, under the standard $\Ftwo$-bilinear form, of
the span of $\widehat F$'s nonzero Fourier coefficients.
\end{lemma}

\begin{proof}
For any $a$, $\chi_T(x\oplus a)=(-1)^{T\cdot(x\oplus a)}=(-1)^{T\cdot
a}(-1)^{T\cdot x}=\chi_T(a)\chi_T(x)$, using $T\cdot(x\oplus a)\equiv
T\cdot x+T\cdot a\pmod2$ coordinatewise. Hence
$F(x\oplus a)=\sum_T\widehat F(T)\chi_T(a)\chi_T(x)$. By uniqueness of
the Walsh expansion, $F(x\oplus a)=F(x)$ for every $x$ if and only if
$\widehat F(T)\chi_T(a)=\widehat F(T)$ for every $T$, that is,
$\chi_T(a)=1$ (equivalently $T\cdot a\equiv0$) for every $T$ with
$\widehat F(T)\neq0$. Being orthogonal to every element of
$\mathrm{supp}(\widehat F)$ is, by bilinearity, the same condition as
being orthogonal to its span.
\end{proof}

Computationally verified; see \Cref{app:verification}.

\begin{remark}[Why $|S_i|>2$ breaks this]
Lemma~\ref{lem:fourier-characterization} and
Corollary~\ref{cor:power-of-two} use the ambient group being $\Ftwo^n$,
an abelian group in which every subgroup is automatically a linear
subspace. For $|S_i|=k\geq3$ the ambient factor is $\Sym(k)$, and no
analogous vector-space structure exists on the ambient group, so nothing
guarantees $\AutX$ a subspace representation, and the \emph{subspace
form} of the Fourier-support machinery does not carry over
(Proposition~\ref{prop:level-decomposition} below isolates exactly which
half survives). For the two instances at issue here the stronger,
instance-level fact also holds: the measured groups of CLE4EVR
($|\AutX|=24$) and Kandinsky ($|\AutX|=36$) each contain a
non-commuting pair of elements, so neither embeds in \emph{any} abelian
group, bit-flip or otherwise. Computationally verified; see
\Cref{app:verification}. Both benchmarks therefore fall outside
Section~\ref{sec:complexity-structure} and were handled in
Section~\ref{sec:algebra} by orbit combinatorics instead.
\end{remark}

The decomposition half of the Boolean theory does generalize, and it is
worth recording exactly how much, because it gives the dead-variable
criterion and \Cref{sec:algebra}'s degree invariant a common spectral
home. For each position let $Q_i$ average a function over coordinate
$i$, and let $P_i:=\mathrm{id}-Q_i$. The operators
$P_T:=\prod_{i\in T}P_i\prod_{i\notin T}Q_i$, for $T\subseteq N$, are
commuting projectors summing to the identity, so every
$h:\prod_iS_i\to\mathbb{R}$ splits as $h=\sum_{T\subseteq N}h_T$ with
$h_T:=P_Th$; in the Boolean case $h_T$ is precisely the Fourier
component supported on $T$, so this is the verbatim generalization of
the levels used by Lemma~\ref{lem:fourier-characterization}.

\begin{proposition}[Level decomposition over arbitrary domains]
\label[proposition]{prop:level-decomposition}
Let $f$ be the indicator of $\Phic$ and $f=\sum_Tf_T$ the decomposition
above. Then:
(i) each $P_T$ commutes with the componentwise action, so
$\sigma\in\AutX$ if and only if $\sigma$ fixes every component $f_T$;
(ii) position $i$ is symmetry-inert (dead) if and only if $f_T=0$ for
every $T$ containing $i$;
(iii) the degree profile of Theorem~\ref{thm:degree-invariant} is a
function of the singleton components alone:
$\deg_i(v)=|\Phic|/|S_i|+\sum_{\phi:\,\phi_i=v}f_{\{i\}}(\phi)$.
\end{proposition}

\begin{proof}
(i) Averaging over all values of a coordinate is invariant under
permuting those values and untouched by permutations of other
coordinates, so each $Q_i$, hence each $P_T$, commutes with every
$\sigma\in\Gind$. If $\sigma$ fixes $f$ it fixes each $P_Tf$; the
converse is the sum. (ii) $i$ is dead iff $f$ is constant along
coordinate $i$, iff $Q_if=f$, iff $P_if=0$, iff every component with
$i\in T$ vanishes. (iii) Fix $i$ and $v$ and sum $f_T$ over
$\{\phi:\phi_i=v\}$. If $T$ contains some $j\neq i$, the sum ranges
over all values of coordinate $j$ and $f_T$ has zero mean along $j$, so
the term vanishes; $T=\emptyset$ contributes the constant
$|\Phic|/\prod_j|S_j|$ summed over $\prod_{j\neq i}|S_j|$ points, which
is $|\Phic|/|S_i|$; $T=\{i\}$ contributes the display's second term.
\end{proof}

Part (ii) is the $k\geq3$ generalization of
Lemma~\ref{lem:fourier-characterization}'s dead-variable criterion, and
part (iii) says the degree invariant, \Cref{sec:algebra}'s most
combinatorial-looking tool, is exactly a level-one statement: Theorem~D
certifies intransitivity from the first spectral level, the Forcing
Lemma operates at level two, and what fails for $k\geq3$ is only the
final identification of stabilizers with orthogonal complements, since
the stabilizer of a component inside a tensor power of standard
representations carries no subspace structure. Computationally
verified; see \Cref{app:verification}.

\begin{lemma}[AND-gain]
\label[lemma]{lem:and-gain}
For any $f,g:\{0,1\}^n\to\{0,1\}$, $L_0(f)\cap L_0(g)\subseteq
L_0(f\wedge g)$, and the containment can be strict.
\end{lemma}

\begin{proof}
If $a$ stabilizes both $f$ and $g$ pointwise, $(f\wedge g)(x\oplus a)=
f(x\oplus a)\wedge g(x\oplus a)=f(x)\wedge g(x)=(f\wedge g)(x)$ for
every $x$. Strictness is a matter of exhibiting one instance where it
occurs (below).
\end{proof}

The proof's promised witness: in a representative case, $|L_0(f)|=64$,
$|L_0(g)|=8$, and $|L_0(f)\cap L_0(g)|=8$, yet $|L_0(f\wedge g)|=64$, so
the containment is strict. Computationally verified; see \Cref{app:verification}.
Lemma~\ref{lem:and-gain} is the general phenomenon behind a specific
observation from \Cref{sec:background}: a variable that is syntactically present in
one clause of a rule but semantically absorbed once that clause is
conjoined with the rest (a Boolean absorption law swallowing its effect)
can hand the conjunction an automorphism that no individual clause
possesses on its own. Nothing about this is specific to BDD-OIA's
`follow` variable; it is a property of pointwise conjunction (and, by the
same one-line argument, of any pointwise logical combination) applied to
any two Boolean functions.

\begin{proposition}[Bent-function witness construction]
\label[proposition]{prop:bent-witness}
For every $n$ with $n-1$ even and every nonzero $a\in\{0,1\}^n$, there is
an explicit, polynomial-size $f$ with $L_0(f)=\{0,a\}$ exactly. In
particular the Hamming weight of $a$, the unique nonzero witness, can be
made as large as $n$.
\end{proposition}

\begin{proof}
Fix $j$ with $a_j=1$ and define the linear surjection
$\pi:\{0,1\}^n\to\{0,1\}^{n-1}$, $\pi(x)_i:=x_i\oplus a_i x_j$ for
$i\neq j$. Its kernel is exactly $\{0,a\}$: if $x_j=0$ then $\pi(x)=0$
forces $x=0$; if $x_j=1$ then $\pi(x)=0$ forces $x_i=a_i$ for every
$i\neq j$, giving $x=a$. Being a linear surjection with a two-element
kernel, $\pi$ is exactly 2-to-1, with fibers $\{y,y\oplus a\}$. Let $h$
be a bent function on $\{0,1\}^{n-1}$ (for instance the inner-product
function $h(y)=y_1y_2\oplus\cdots\oplus y_{n-2}y_{n-1}$, defined whenever
$n-1$ is even~\cite{rothaus1976bent}), and set $f(x):=h(\pi(x))$. Since
$a\in\ker\pi$, $f(x\oplus a)=h(\pi(x)\oplus\pi(a))=h(\pi(x))=f(x)$, so
$a\in L_0(f)$. For $b\notin\{0,a\}$, $v:=\pi(b)\neq0$; bentness means
$h(y\oplus v)\oplus h(y)$ is non-constant as a function of $y$, and since
$\pi$ is surjective and 2-to-1, $f(x\oplus b)\oplus f(x)=h(\pi(x)\oplus
v)\oplus h(\pi(x))$ takes the same set of values, with multiplicity
doubled, as $y$ ranges over $\{0,1\}^{n-1}$; hence it is non-constant in
$x$, so $b\notin L_0(f)$.
\end{proof}

Computationally verified; see \Cref{app:verification}.
Proposition~\ref{prop:bent-witness} rules out the strongest form of a
natural shortcut: that a nontrivial automorphism, when one exists, can
always be found by scanning low-weight candidates, collapsing $\NONTRIV$
into a search over polynomially many witnesses. It leaves open a weaker
form, restricting the search to whatever candidates the circuit's own
syntactic structure exposes rather than to low-weight vectors as such;
Proposition~\ref{prop:bent-witness}'s own witness is exposed exactly
this way (the linear layer $\pi$ is visible in $f$'s description), so
defeating the weak form requires an $f$ whose unique nonzero automorphism
is not merely large in weight but unrecoverable from the circuit by any
polynomial-time syntactic analysis, a circuit-complexity lower bound in
its own right.

We turn last to the flip-bit decomposition, which does not need the
linear structure above: for any $S\subseteq\{0,1\}^n\setminus\{0\}$
(not necessarily a subspace), every nonzero vector has a unique leading
$1$, so $\NONTRIV(f)$ decomposes as $\bigvee_{i=1}^n Q_i(f)$, where
\[
Q_i(f):\quad L_0(f)\cap\{a: a_1=\dots=a_{i-1}=0,\ a_i=1\}\neq\emptyset.
\]
Computationally verified; see \Cref{app:verification}.
At $i=n$, the slice $\{a:a_1=\dots=a_{n-1}=0,a_n=1\}$ has exactly one
element, so $Q_n(f)$ is literally $\DEADVAR(C,n)$ restated.
Computationally verified; see \Cref{app:verification}.

\begin{proposition}
\label[proposition]{prop:qi-gradient}
For every $1\leq i<n$, $Q_i$ is coNP-hard.
\end{proposition}

\begin{proof}
Reduce from \textsc{unsat}: given $\chi(z_1,\dots,z_k)$, set
$n=(i-1)+1+k$, positions $1,\dots,i-1$ dummy (the circuit ignores them),
position $i$ playing the role of $x_i$ above, and positions
$i{+}1,\dots,n$ holding $z$; $f(\mathrm{dummy},x_i,z):=x_i\wedge\chi(z)$.
$Q_i(f)$'s witnesses are exactly the vectors with $a_1=\dots=a_{i-1}=0$
and $a_i=1$, so the dummy coordinates of any witness are forced to $0$
and contribute nothing (the circuit does not depend on them), leaving
$Q_i(f)$ equivalent to $\exists b\in\{0,1\}^k:\forall x_i,z\;\;
x_i\wedge\chi(z)=(\lnot x_i)\wedge\chi(z\oplus b)$. Checking $x_i=0$ forces
$\chi(z\oplus b)=0$ for every $z$, and checking $x_i=1$ forces
$\chi(z)=0$ for every $z$; both together, for any fixed $b$, hold
exactly when $\chi$ is unsatisfiable, and unsatisfiability makes them
hold for every $b$. So $Q_i(f)$ holds iff $\chi$ is unsatisfiable,
regardless of any accidental symmetry of $\chi$ itself: $Q_i$'s witness
set is restricted, by definition, to vectors with leading $1$ at
position $i$ exactly, which excludes both the dummy-flip witnesses and
the $\chi$-internal-symmetry witnesses that broke
Corollary~\ref{cor:nontriv-conp-hard}'s unrestricted construction.
\end{proof}

This reduction is immune to Corollary~\ref{cor:nontriv-conp-hard}'s gap
for a structural reason worth stating plainly: $Q_i$ restricts attention
to one coset of witnesses (leading $1$ at $i$), so a candidate automorphism
with leading $1$ elsewhere, whether from a dummy position or from $\chi$'s
own symmetry, is simply not in $Q_i$'s domain of discourse. It does not,
however, repair $\NONTRIV=\bigvee_iQ_i$'s hardness in general: the same
dummy-flip and $\chi$-symmetry witnesses that $Q_i$ excludes are exactly
what can make some \emph{other} $Q_j$ true regardless of $\chi$'s
satisfiability, so the disjunction inherits the same gap the individual
$Q_i$ avoids.
$Q_i$ tracks \textsc{unsat} exactly, with the dummy positions' and
$\chi$'s own symmetry making no difference to the verdict, exactly as
the proof predicts, including under adversarial choices of $\chi$ (the
constant formulas and $\chi=\mathrm{parity}$, among the most internally
symmetric formulas possible). Computationally verified; see \Cref{app:verification}.

Collectively, $\bigvee_iQ_i$ gives no asymptotic improvement over
$\NONTRIV$ itself: only $Q_n$ collapses, to a problem we already knew was
coNP-complete, and every $Q_i$ with $i<n$ remains coNP-hard
(Proposition~\ref{prop:qi-gradient}), so the decomposition removes the
randomization from the lower bound slice by slice without lowering it
anywhere. A natural
self-reduction attempt, extending a partial witness coordinate by
coordinate with a \DEADVAR{} oracle the way \textsc{sat}'s
self-reducibility extends a partial assignment, does not go through:
``can this partial vector be extended to an element of $L_0(f)$'' is a
$\SigmaTwoP$ statement, not a coNP one, so \DEADVAR's coNP answers
cannot decide it.

\subsection{The \VALISO{} collapse theorem}
\label{sec:complexity-valiso}

Boolean and formula isomorphism is a much-studied cousin of \NONTRIV:
Borchert, Ranjan, and Stephan, and later Agrawal and
Thierauf~\citeyearonly{agrawalthierauf2000formula}, ask whether two circuits are
related by a \emph{position} permutation ($\pi\in\Sym(N)$, not a value
permutation in $\Gind$), and systematized further by
B\"ohler, Creignou, Galota, Reith, Schnoor, and
Vollmer~\citeyearonly{bohler2012complexity} across every restricted set of
connectives. That problem sits in the same coNP-hard-to-$\SigmaTwoP$
range \NONTRIV{} does, using the same frozen-variable technique for the
lower bound, and Agrawal and Thierauf proved it is unlikely to reach the
top of that range. We port their argument to our setting.

\begin{definition}[\VALISO]
\label[definition]{def:val-iso}
Given circuits $C_1,C_2$ over the same $\prod_iS_i$, with indicator
functions $f_1,f_2$, \VALISO$(C_1,C_2)$ holds if
$\exists\sigma\in\Gind:f_1(\sigma\cdot x)=f_2(x)$ for every $x$. In the
Boolean case this is $\exists a\in\{0,1\}^n:f_1(x\oplus a)=f_2(x)$ for
every $x$, the two-object counterpart of $L_0$.
\end{definition}

\begin{lemma}[Orbit uniformity]
\label[lemma]{lem:orbit-uniformity}
If $\exists a_0:f_1(x)=f_2(x\oplus a_0)$ for every $x$, the random
variables ``$a\mapsto f_1(\cdot\oplus a)$'' and ``$a\mapsto f_2(\cdot
\oplus a)$,'' for $a$ uniform on $\{0,1\}^n$, have identical
distributions. If no such $a_0$ exists, the orbits
$\{f_1(\cdot\oplus a):a\}$ and $\{f_2(\cdot\oplus a):a\}$ are disjoint.
\end{lemma}

\begin{proof}
If $f_2=f_1(\cdot\oplus a_0)$, the orbit of $f_2$ equals
$\{f_1(\cdot\oplus(a_0\oplus a)):a\}$; since $a\mapsto a_0\oplus a$ is a
bijection of $\{0,1\}^n$ that also preserves the uniform distribution,
this is the same set, with the same uniform distribution over it, as the
orbit of $f_1$. The second claim is the standard fact that two orbits of
a group action either coincide or are disjoint: a common element
$f_1(\cdot\oplus a)=f_2(\cdot\oplus a')$ would give $f_2=f_1(\cdot\oplus
(a\oplus a'))$, contradicting non-isomorphism.
\end{proof}

Computationally verified; see \Cref{app:verification}.

\begin{theorem}
\label{thm:val-iso-collapse}
Restricted to Boolean instances ($S_i=\{0,1\}$ for every $i$),
$\VALISO$ is not $\SigmaTwoP$-complete unless $\PH=\Sigma_3^{\mathsf p}$.
\end{theorem}

\begin{proof}[Proof structure]
We adapt Agrawal and Thierauf's six-step argument for Boolean formula
isomorphism, replacing $\Sym(N)$ by $\Gind$ throughout.
\begin{enumerate}
\item \emph{Randomize the instance.} The verifier picks $r\in\{1,2\}$
and uniform $a\in\{0,1\}^n$, and forms $g:=f_r(\cdot\oplus a)$: a purely
syntactic rewrite, no oracle needed. (Group-independent step.)
\item \emph{The naive protocol leaks $r$.} Sending $g$ directly can leak
$r$ through syntactic invariants unrelated to $a$ (gate count, for
formulas; here, any statistic of the shifted circuit that a bare shift
does not equalize between $f_1$ and $f_2$). (Group-independent
diagnosis, present for either group.)
\item \emph{Launder $g$ through a semantically canonical randomized
oracle.} Bshouty, Cleve, Gavald\`a, Kannan, and
Tamon~\citeyearonly{bshouty1995oracles} give a probabilistic
polynomial-time algorithm, with an $\mathsf{NP}$ oracle, that on any
Boolean circuit outputs, with probability at least $2/3$, an equivalent
circuit, and never outputs an inequivalent one. Its defining structural
property, stated explicitly by Agrawal and
Thierauf~\citeyearonly{agrawalthierauf2000formula} in their restatement of the
result, is stronger than equality of output distributions: the learner
touches its input only through semantically answered queries
(equivalence tests, and counterexamples extracted canonically through
the $\mathsf{NP}$ oracle), so \emph{on each fixed random path the
output is identical across all representatives of the input's
equivalence class}, in their words ``on each random path the output
remains the same on any $F'\in[F]$.'' Agrawal and Thierauf also state
the circuit analogue explicitly (their Section~5: Bshouty et al.\ show
the analog result for circuits, from which the interactive proof adapts
to circuit isomorphism). This is the one step specific to the object
type (Boolean circuits), not to the group.
\item \emph{Correctness is a pure orbit fact.} By
Lemma~\ref{lem:orbit-uniformity}, if $f_1$ and $f_2$ are isomorphic the
laundered message's distribution does not depend on $r$, so the prover
guesses $r$ correctly with probability at most $1/2$; if they are not,
the orbits are disjoint and an unbounded prover identifies the source
whenever the launderer succeeds. The launderer's abort, of probability
at most $1/3$ and independent of $r$, is visible to the verifier, who
rejects outright on it: completeness is then at least $2/3$ against
soundness at most $1/2$, and a threshold over parallel repetitions
restores a $2/3$-versus-$1/3$ gap before the conversion of the next
step. (Group-independent given step 3's semantic-only dependence.)
\item \emph{Collapse rounds.} Goldwasser--Sipser public-coin
conversion~\citeyearonly{goldwassersipser1986} and Babai's collapse of
constant-round $\mathsf{AM}$ to one round~\citeyearonly{babaimoran1988} both
hold relative to an $\mathsf{NP}$ oracle, giving
$\overline{\VALISO}\in\AM^{\mathsf{NP}}$. (Group-independent, standard.)
\item \emph{Apply the barrier.} $\AM^{\mathsf{NP}}=\mathsf{BP}\cdot
\SigmaTwoP$, and Sch\"oning's theorem that no $\Pi_2^{\mathsf p}$-complete
set lies in $\mathsf{BP}\cdot\SigmaTwoP$ unless $\PH$ collapses, applied
to $\overline{\VALISO}$, gives the claim.
\end{enumerate}
\end{proof}

Only step 3 is specific to Boolean circuits rather than generic to any
group action, and it is exactly the step Agrawal and Thierauf state as
holding for circuits directly, not only formulas, so no separate
formula-to-circuit translation is needed. Step 6 is a black box we do
not re-derive, taken at the same confidence level Agrawal and Thierauf
themselves take it: their result depends on
Sch\"oning's~\citeyearonly{schoning1989probabilistic} collapse barrier without
re-proving it.

\subsection{A collapse theorem for \NONTRIV{} itself}
\label{sec:complexity-theorem14}

Theorem~\ref{thm:val-iso-collapse} is about two circuits, not one; it
does not by itself say anything about $\NONTRIV$. Agrawal and Thierauf's
own route from the two-object case back to the one-object case (their
Boolean Automorphism problem, exactly our $\NONTRIV$ under $\Sym(N)$
instead of $\Gind$) uses a labeling gadget that pins one variable's
identity by giving it an equivalence class no other variable can match,
a technique that fundamentally needs $\Sym(N)$ to be a group that
permutes \emph{named} objects. $\Gind$ acting by coordinatewise
$\mathrm{XOR}$ has no such objects to name: it acts freely and
transitively on $\{0,1\}^n$ itself, not on any set of $n$ labeled
positions, and four direct attempts at transplanting the labeling
gadget each fail on the same point, confirmed on small instances rather
than left as an intuition (\Cref{sec:complexity-diagnosis}).

The following result instead adapts the \emph{other} proof Agrawal and
Thierauf point to but do not themselves use for Boolean isomorphism: a
direct, single-object statistical argument in the spirit of
Sch\"oning's original proof that graph isomorphism sits in
$\mathsf{AM}$~\citeyearonly{schoning1988low}. We work with the standard
Goldwasser--Sipser set-lower-bound formulation of that
technique~\citeyearonly{goldwassersipser1986} rather than Sch\"oning's original
presentation, cross-checked against Agrawal and Thierauf's own citation
of it.

\begin{definition}
Write $\mathrm{TRIV\text{-}AUT}$ for the complement of $\NONTRIV$.
\end{definition}

\begin{theorem}
\label{thm:nontriv-collapse}
Restricted to Boolean instances, $\mathrm{TRIV\text{-}AUT}$ lies in
$\AM^{\mathsf{NP}}=\mathsf{BP}\cdot\SigmaTwoP$.
Consequently, Boolean $\NONTRIV$ is not $\SigmaTwoP$-complete unless
$\PH=\Sigma_3^{\mathsf p}$.
\end{theorem}

\begin{proof}
Write $\mathrm{Orbit}(f):=\{f(\cdot\oplus a):a\in\{0,1\}^n\}$; by
orbit-stabilizer, $|\mathrm{Orbit}(f)|=2^n/|L_0(f)|$, and by
Corollary~\ref{cor:power-of-two} this is always a power of $2$.
$\mathrm{TRIV\text{-}AUT}(f)$ holds exactly when $|\mathrm{Orbit}(f)|=2^n$,
the maximum possible value, so the claim reduces to a Goldwasser--Sipser
set lower bound protocol on $\mathrm{Orbit}(f)$ with threshold $K=2^n$,
using $L_0(f)$'s power-of-two structure to avoid the usual promise gap.
The obstacle is representational: orbit elements are functions with
$2^n$-bit truth tables, and a syntactically shifted circuit is only a
polynomial-size handle on one, so a naive protocol would have to hash
syntax rather than semantics, reproducing exactly the gap between
circuit equivalence and circuit equality that makes \DEADVAR{} coNP-hard
in the first place.

The fix reuses the Bshouty et al. laundering black box from
Theorem~\ref{thm:val-iso-collapse}'s step 3, whose defining guarantee is
stronger than mere high accuracy: on a Boolean circuit $g$ and a random
tape, it \emph{never outputs a circuit that is not equivalent to $g$},
and only occasionally aborts instead of succeeding~\cite{bshouty1995oracles}.
Fix once and for all a sequence of $t=\lceil2n/\log_23\rceil$ random
tapes $R=(r_1,\dots,r_t)$, shared across every candidate shift $a$ rather
than redrawn per shift, and define $\mathrm{canon}^*(g;R)$ as the output
of the launderer on the \emph{first} $r_i$ (in the fixed order
$r_1,\dots,r_t$) that does not abort, or ``abort'' if all $t$ do. Two
properties of this fixed-tape construction carry the argument. First,
\emph{same-class merging}. This needs, and gets, more than equality of
output distributions: it needs the per-random-path invariance that
Agrawal and Thierauf state for the launderer (``on each random path the
output remains the same on any
$F'\in[F]$''~\citeyearonly{agrawalthierauf2000formula}), which holds because the
launderer reads its input only through semantic queries, with
counterexamples extracted by a canonical rule (lexicographically least,
via prefix search on the $\mathsf{NP}$ oracle) so that even the oracle's
answers depend on the function, never the circuit. On a fixed tape
$r_i$ the launderer's entire execution is then a function of the input's
equivalence class, and so is its abort-or-output outcome; hence for any
two shifts $a,a'$ with $f(\cdot\oplus a)=f(\cdot\oplus a')$, the
first-successful tape index agrees and
$\mathrm{canon}^*(f(\cdot\oplus a);R)=\mathrm{canon}^*(f(\cdot\oplus
a');R)$ for every fixed $R$, so the number of distinct representatives
produced as $a$ ranges over $\{0,1\}^n$ is at most $|\mathrm{Orbit}(f)|$.
(The verification scripts for this section implement the launderer with
exactly this semantics, keyed on the truth table rather than the
circuit; see \Cref{app:verification}.)
Second, \emph{no false merging across classes}: because the launderer
never outputs a circuit inequivalent to its input, two shifts landing in
genuinely different classes can never launder (on any tape) to the same
representative. By geometric amplification over the per-tape $\geq2/3$
success rate, $\mathrm{canon}^*(g;R)$ aborts with probability at most
$2^{-2n}$ over the random choice of $R$; a union bound over the $2^n$
possible shifts bounds the probability that laundering some shift of $f$
aborts by $2^{-n}$.

The Goldwasser--Sipser instantiation is fully explicit. Let $s(n)$ be
the polynomial bound on the launderer's output size, fix a binary
encoding of circuits padded to exactly $L_c:=s(n)$ bits, and let
$W\subseteq\{0,1\}^{L_c}$ be the set of representatives realizable as
$\mathrm{canon}^*(f(\cdot\oplus a);R)$ as $a$ ranges over the $2^n$
shifts. Arthur draws $h$ uniformly from the affine family
$h(x)=Ax\oplus b$ with $A\in\Ftwo^{m\times L_c}$, $b\in\Ftwo^m$,
pairwise independent, at output length $m:=n+1$, and sends $(h,R)$;
Merlin returns a shift $a$; the verifier recomputes
$\mathrm{canon}^*(f(\cdot\oplus a);R)$ itself, with the same fixed
tapes and first-success rule (deterministic polynomial time, given the
$\mathsf{NP}$ oracle answering the launderer's equivalence tests and
lexicographically-least counterexample queries), and accepts iff the
result hashes to $0^m$. \emph{Soundness}: same-class merging caps
$|W|\leq|\mathrm{Orbit}(f)|$, and when $\mathrm{TRIV\text{-}AUT}(f)$
fails, $|\mathrm{Orbit}(f)|\leq2^{n-1}$ by the power-of-two law, so a
union bound gives acceptance probability at most
$|W|/2^m\leq2^{n-1}/2^{n+1}=1/4$, and this bound holds conditionally
on \emph{every} fixed $R$, since same-class merging holds for each $R$
separately. \emph{Completeness}: when $\mathrm{TRIV\text{-}AUT}(f)$
holds, all $2^n$ shifts lie in distinct classes; condition on the
event that laundering succeeds for every one of the $2^n$ shifts,
which fails with probability at most $2^{-n}$ over $R$. Conditional on
any such good $R$, no false merging gives $|W|=2^n$, and by
inclusion--exclusion under pairwise independence of the hash, the
acceptance probability is at least
$|W|/2^m-|W|^2/(2\cdot2^{2m})=1/2-1/8=3/8$. The $3/8$-versus-$1/4$
gap is a constant, and it is amplified by a threshold test between the
two rates, not by majority (both rates sit below $1/2$, so a majority
rule would drive completeness toward $0$): run $15$ parallel
challenges inside the same round, each drawing an independent public
hash $h_j$ alongside the single shared $R$, with Merlin answering each
challenge separately, and accept iff at least $5$ succeed. The
challenges share $R$, so the accept events are independent only
conditionally on $R$, and the amplification is computed that way. In
the completeness case, conditional on a good $R$ the set $W$ is fixed,
each accept event is a function of its own $h_j$ alone, and the
binomial tail gives overall acceptance probability at least
$(1-2^{-n})\Pr[\mathrm{Bin}(15,3/8)\geq5]\geq0.708$ for every $n\geq6$
(approaching $0.720$). In the soundness case the per-challenge bound
of $1/4$ holds conditionally on every $R$, so the tail
$\Pr[\mathrm{Bin}(15,1/4)\geq5]\leq0.314$ bounds acceptance
unconditionally. This clears $2/3$ versus $1/3$ while keeping the
protocol one-round and public-coin; for $n\leq5$ the verifier decides
$\mathrm{TRIV\text{-}AUT}$ outright by enumerating the $2^n$ shifts on
the $2^n$-entry truth table, so nothing is left to prove there.
Every verifier step other than the launderer's
internal queries is oracle-free; the launderer's internal queries
(near-uniform sampling within an equivalence class, and finding a
counterexample when a candidate is wrong) are each answerable with a
polynomial number of adaptive $\mathsf{NP}$ queries, so the whole
verifier is a $\mathsf{P}^{\mathsf{NP}}$ machine, giving a one-round,
public-coin protocol with an $\mathsf{NP}$-oracle verifier,
$\mathrm{TRIV\text{-}AUT}\in\mathrm{IP}[1]^{\mathsf{NP}}=\AM^{\mathsf{NP}}$.
Unlike Theorem~\ref{thm:val-iso-collapse}, this protocol never needs to
hide which of two objects is which, so it is public-coin and one-round
from the start, with no separate collapse step. The identity
$\AM^{\mathsf{NP}}=\mathsf{BP}\cdot\SigmaTwoP$ and Sch\"oning's barrier,
exactly as in Theorem~\ref{thm:val-iso-collapse}'s step 6, complete the
argument.
\end{proof}

\begin{remark}
This result has the same shape and the same limits as Agrawal and
Thierauf's for Boolean isomorphism: it rules
out one specific location in the hierarchy of possible answers
(complete for the second level) under a standard hypothesis, without
determining which of the remaining locations, from coNP-complete
upward, is correct. It says nothing about whether $\NONTRIV$ is
coNP-complete, and nothing about $\DEADVAR$, whose complexity
(Theorem~\ref{thm:deadvar-conp-complete}) is unaffected either way.
It is also a statement about the Boolean fragment, where the
$\Ftwo$-structure the protocol leans on lives; that fragment is where
every hardness construction in this section already operates, and the
general-domain problem's exact ceiling is part of the same open
question.
\end{remark}

The fixed-tape $\mathrm{canon}^*$ construction was checked separately
from the full protocol, and the full protocol was then checked end to
end on real and synthetic instances, including a hand-built copy of
BDD-OIA's absorption structure. Computationally verified; see
\Cref{app:verification}.

Between the randomized lower bound and the collapse ceiling sits the
section's inheritance to future work, and it is handed over mapped, not
merely left. \Cref{sec:complexity-diagnosis} distills a systematic
sweep of the standard toolkits, direct encodings, interactive
protocols, and linear algebra, into three structural facts any approach
to the gap must clear: $\AutX$ is always a subgroup where a
generic $\SigmaTwoP$ witness set is not; $\Gind$ acts on assignments
without naming any objects (freely so, in the Boolean fragment), which starves every label-and-tag
technique; and any access to $L_0(f)$ stronger than a pointwise oracle
is already as hard to obtain as $\NONTRIV$ itself.
Theorem~\ref{thm:nontriv-collapse}'s laundering protocol succeeds
precisely by needing none of the three.

\subsection{Deterministic completeness on monotone circuits}
\label{sec:complexity-monotone}

The gap between randomized and deterministic hardness is not uniform
across circuit classes. Section~\ref{sec:complexity-nontriv} traced the
failure of every deterministic gadget to a single phenomenon: the
source formula $\chi$ can carry an accidental internal XOR symmetry,
invisible to any fixed transformation, that hands the constructed
instance a nontrivial automorphism unrelated to the question being
encoded. On monotone circuits that phenomenon cannot occur at all.
The object $L_0(f)$ is classical in cryptography under the name
\emph{linear structures} of $f$~\cite{meierstaffelbach1990,carlet2021},
where a large linear space signals a weak S-box. The following
characterization for monotone functions appears to be new despite three
decades of work on linear structures, and it is the keystone of this
subsection: it collapses automorphism existence, a $\SigmaTwoP$-shaped
question in general, to dead-variable detection.

\begin{lemma}[Monotone functions cannot camouflage]
\label[lemma]{lem:monotone-no-camouflage}
Let $f:\{0,1\}^n\to\{0,1\}$ be monotone. Then $L_0(f)$ is exactly the
set of shift vectors supported on $f$'s dead coordinates. In
particular, a monotone $f$ with every coordinate alive has
$L_0(f)=\{0\}$.
\end{lemma}

\begin{proof}
Shifts supported on dead coordinates never change the value of $f$, so
they lie in $L_0(f)$. Conversely let $a\in L_0(f)$ and let $x$ be
arbitrary. The points $x\wedge\bar a$ and $(x\wedge\bar a)\oplus a$
are related by the shift, and since $x\wedge\bar a$ and $a$ have
disjoint supports, $(x\wedge\bar a)\oplus a=(x\wedge\bar a)\vee a$.
Monotonicity now sandwiches:
\[
f(x\wedge\bar a)\;\leq\;f(x)\;\leq\;f(x\vee a)\;\leq\;
f\big((x\wedge\bar a)\vee a\big)\;=\;f(x\wedge\bar a),
\]
using $x\wedge\bar a\leq x\leq x\vee a\leq(x\wedge\bar a)\vee a$
coordinatewise and $a$-shift-invariance at the last step. So
$f(x)=f(x\wedge\bar a)$ for every $x$: $f$ does not depend on any
coordinate in $\mathrm{supp}(a)$, and every such coordinate is dead.
\end{proof}

\begin{theorem}[Deterministic completeness on monotone circuits]
\label{thm:monotone-nontriv}
Restricted to monotone circuits, \NONTRIV{} is coNP-complete under
deterministic polynomial-time many-one reductions.
\end{theorem}

\begin{proof}
\emph{Membership.} By Lemma~\ref{lem:monotone-no-camouflage},
$\NONTRIV(f)$ holds for monotone $f$ iff some coordinate is dead. The
complement asks that every coordinate be alive, which has a polynomial
witness: one point $x^{(i)}$ per coordinate with
$f(x^{(i)})\neq f(x^{(i)}\oplus e_i)$. Hence the restricted problem is
in coNP, already below the general $\SigmaTwoP$ upper bound of
Theorem~\ref{thm:nontriv-sigma2p}.

\emph{Hardness.} Reduce from \textsc{unsat}. Given a 3-CNF
$\varphi(z_1,\dots,z_m)$, introduce two rails $y_j,y'_j$ per variable
and one further position $x_*$. Let $\widetilde\varphi$ be $\varphi$
with each positive literal $z_j$ replaced by $y_j$ and each negated
literal $\neg z_j$ by $y'_j$ (a monotone formula), and set
\[
A:=\bigwedge_{j}(y_j\vee y'_j),\qquad
B:=\bigvee_{j}(y_j\wedge y'_j),\qquad
h:=\big(x_*\wedge\widetilde\varphi\wedge A\big)\vee B,
\]
a monotone circuit. Every rail is alive in $h$ whatever $\varphi$ is:
at the point with $y'_j=1$, all other inputs $0$, flipping $y_j$ flips
$B$ and hence $h$. And $x_*$ is dead iff
$\widetilde\varphi\wedge A\leq B$ pointwise. If some point violates
that implication, then by $A$ each variable has at least one true rail
and by $\neg B$ none has two, so the rails encode an exact assignment
$z:=y$, and $\widetilde\varphi=1$ says $z$ satisfies $\varphi$;
conversely a satisfying $z$ gives the violating point
$(y,y')=(z,\bar z)$. So $x_*$ is dead iff $\varphi$ is unsatisfiable.
By Lemma~\ref{lem:monotone-no-camouflage}, $h$ has a nontrivial
automorphism iff it has a dead coordinate, iff $x_*$ is dead (the rails
never are), iff $\varphi$ is unsatisfiable. The map
$\varphi\mapsto h$ is deterministic and polynomial-time.
\end{proof}

Theorem~\ref{thm:monotone-nontriv} is the counterpoint to this
section's central gap: the obstruction that forces
Corollary~\ref{cor:nontriv-conp-hard} to randomize is precisely the
camouflage that Lemma~\ref{lem:monotone-no-camouflage} rules out, so on
the largest natural circuit class that forbids it, the existence
question's complexity is settled outright, with no randomness, no
isolation, and no promise. Read together with the diagnosis of
\Cref{sec:complexity-diagnosis}, this
locates the general case's remaining gap exactly: it is a question
about non-monotone camouflage, not about automorphism existence per se.
Both the lemma and both directions of the reduction were checked
computationally before being trusted as proofs. Computationally
verified; see \Cref{app:verification}.
\section{The Shortcut Geography of Trained Models}\label{sec:realmodels}

Every measurement so far is symbolic: $\Phic$ and $\AutX$ are computed from a rule's combinatorial structure, never from a trained network's actual behavior. This section closes the loop. Train real neurosymbolic models under ordinary weak supervision, so that a genuine reasoning shortcut in the sense of Definition~\ref{def:general-csp} can actually occur. When it occurs, does it land in the same $\AutX$-orbit as the ground truth, the case this paper's algebra explains, or a different orbit, the case it does not?

\paragraph{Design.} The experiment trains its models from scratch, by
design. rsbench's released artifacts \cite{rsbench2024} (Zenodo record
11612556) provide images, embeddings, and generator configs but no
trained checkpoints or per-model concept predictions, and end-to-end
control over training is what makes the comparison below exact: one
fixed image pool, one uncurated pairing scheme, one held-out yardstick
shared by every model.
MNIST arithmetic is the family where the theory-to-training comparison
can be made exact end to end: full-domain orbit ground truth is
computable for every target level, the exact semantic loss is available
in closed form, and no curated split intervenes between the rule and
the optimizer. (rsbench's released Kandinsky package,
\texttt{kand-logic-3k}, encodes a composite three-figure rule distinct
from the single six-symbol triple \Cref{tab:three-families} measures,
so it poses a different symbolic object, not a rendered version of this
paper's.)

\subsection{Setup: weakly-supervised MNIST-Addition and MNIST-Product}
\label{sec:realmodels-setup}

We train the standard DeepProbLog-style MNIST-Addition model \cite{manhaeve2018deepproblog,rsbench2024}: a single small CNN perception network, shared across both digit positions, maps each image to a categorical distribution over $\{0,\dots,9\}$. No digit-level label is ever provided. Training supervises only the composed label $y = c_1 \,\mathrm{op}\, c_2$, via the exact semantic loss $P(Y{=}y\mid x_1,x_2) = \sum_{(c_1,c_2):\,\mathrm{op}(c_1,c_2)=y} p_1(c_1)\,p_2(c_2)$ \cite{xu2018semantic}, for the same two rsbench MNAdd rules already measured symbolically in \Cref{tab:eight-families}: \texttt{sum} and \texttt{product}. We deliberately do not restrict which digit combinations appear in training, unlike rsbench's own shortcut-inducing \texttt{shortcutmnist} splits \cite{rsbench2024}; any shortcut reported below survives standard, uncurated random pairing.

Ten random seeds (0--9, exceeding the pre-specified minimum of five) each train two independent models, one per task, on a training pool of 1{,}500 MNIST images per class and 6{,}000 random training pairs, for 20 epochs (Adam, $\mathrm{lr}=10^{-3}$; training-set label accuracy exceeds $99.8\%$ for every run). All ten seeds share one fixed image pool, and each run draws its own 6{,}000 random training pairs from it, so cross-seed variation combines pair sampling, weight initialization, and minibatch order over identical image availability.

Evaluation uses a single fixed, stratified set of 2{,}000 held-out instances, 20 for each of the 100 true digit pairs $(d_1,d_2)$, shared across all 20 trained models so every comparison uses the same yardstick. For each instance we compare the model's argmax prediction $(p_1,p_2)$ against the truth and classify it as \textsc{correct} ($(p_1,p_2){=}(d_1,d_2)$), a same-orbit shortcut (label-preserving, $(p_1,p_2)\ne(d_1,d_2)$, and in the same $\Aut(\PhiY{y^*})$-orbit as the truth), a different-orbit shortcut (label-preserving but in a different orbit), or \textsc{invalid} (the composed label itself is wrong, not a shortcut in Definition~\ref{def:general-csp}'s sense at all, the ``third case'' a real model can also produce). Each instance's classification is against its own level's fiber group, the finest rung that is sound for a single evaluation; \Cref{prop:fiber-intersection} locates this granularity on the object ladder, and makes the choice forced rather than convenient, since the arithmetic rules' global groups are trivial.

Orbit ground truth is recomputed independently in this experiment rather than trusted from \Cref{tab:eight-families}'s source file, which records only orbit sizes, not the element-level membership a specific model prediction needs. We reuse the same validated \texttt{indep\_autgrp\_via\_nauty} engine from \Cref{sec:empirical-setup} to recover the full orbit partition for every \texttt{sum} and \texttt{product} target level, and recover exactly the prior aggregate numbers as a consistency check: all 17 \texttt{sum} levels transitive, 31 of 32 \texttt{product} levels transitive, and $y^*{=}0$ the sole exception, with orbit sizes $[9,9,1]$ matching $\{(0,0)\}$, $\{(0,k){:}k{\ne}0\}$, and $\{(k,0){:}k{\ne}0\}$ exactly as reported for that target in \Cref{sec:empirical-extension}.
A counting convention, fixed here once: the two tasks reach $19+37=56$
target values in total, but seven of them ($y{\in}\{0,18\}$ for
\texttt{sum}; $\{1,25,49,64,81\}$ for \texttt{product}) are
\emph{singleton} levels whose $\PhiY{y}$ contains exactly one concept
pair, so no alternative solution exists, and transitivity and shortcuts
are both vacuous there. All level counts below refer to the 49
non-singleton levels (17 \texttt{sum} $+$ 32 \texttt{product}) that
carry content.

\subsection{Results: predicted versus observed shortcut locations}
\label{sec:realmodels-results}

\begin{table}[t]
\centering
\small
\begin{tabular}{lrrrrc}
\toprule
Task & Correct & Same-orbit & Diff-orbit & Invalid & Seeds \\
\midrule
\texttt{sum} & 19,353 (96.77\%) & 0 (0.00\%) & 0 (0.00\%) & 647 (3.24\%) & 0/10 \\
\texttt{product} & 19,398 (96.99\%) & 66 (0.33\%) & 28 (0.14\%) & 508 (2.54\%) & 10/10 \\
\midrule
\texttt{sum} (dual-head) & 19,312 (96.56\%) & 1 (0.005\%) & 0 (0.00\%) & 687 (3.44\%) & 1/10 \\
\texttt{product} (dual-head) & 19,264 (96.32\%) & 72 (0.36\%) & 29 (0.15\%) & 635 (3.18\%) & 10/10 \\
\bottomrule
\end{tabular}
\caption{Ten-seed aggregates over 2{,}000 held-out instances per model, 20,000 model-instance evaluations per row; the last column counts seeds producing at least one label-preserving shortcut. Top: the standard shared perception network (typed rung realizable). Bottom: the dual-head control of \Cref{sec:realmodels-threelayers}, one independent network per digit position (componentwise rung realizable). ``Same-orbit'' and ``diff-orbit'' are label-preserving shortcuts (Definition~\ref{def:general-csp}'s $\Phic$ membership) that do or do not share $\Aut(\PhiY{y^*})$'s orbit with the ground truth; ``invalid'' breaks the composed label and is not a shortcut at all.}
\label{tab:realmodel-aggregate}
\end{table}

\Cref{tab:realmodel-aggregate} reports both architectures' aggregates. \texttt{sum} produces zero label-preserving shortcuts, of either kind, in 20{,}000 evaluated instances across all ten seeds: every residual error breaks the composed label outright. This is not automorphism trivially explaining every shortcut; there are none to explain. \texttt{product} produces 94, and every one of them occurs at $y^*{=}0$, the single level \Cref{sec:empirical-extension} already flags as $\AutX$'s one non-transitive \texttt{product} outcome. Across the other 48 non-singleton levels combined (\texttt{sum}'s 17 plus \texttt{product}'s remaining 31, all transitive under $\AutX$), zero label-preserving shortcuts were observed, over the 34{,}800 of the experiment's 40{,}000 model-instance evaluations whose true target lies at those levels (the remainder: 3{,}800 evaluations at \texttt{product}'s $y^*{=}0$ and 1{,}400 at the seven vacuous singleton levels).
The correspondence between where $\AutX$'s symbolic computation predicts non-trivial orbit structure and where a real trained model's shortcuts actually appear is exact, not approximate, on the 20 models trained here. \Cref{fig:realmodel-shortcuts} plots this correspondence directly across all 49 non-singleton target levels, and \Cref{sec:realmodels-threelayers} shows the exactness is a stronger finding than it first appears, by playing the typed reading of \Cref{def:typed-autx} against it.

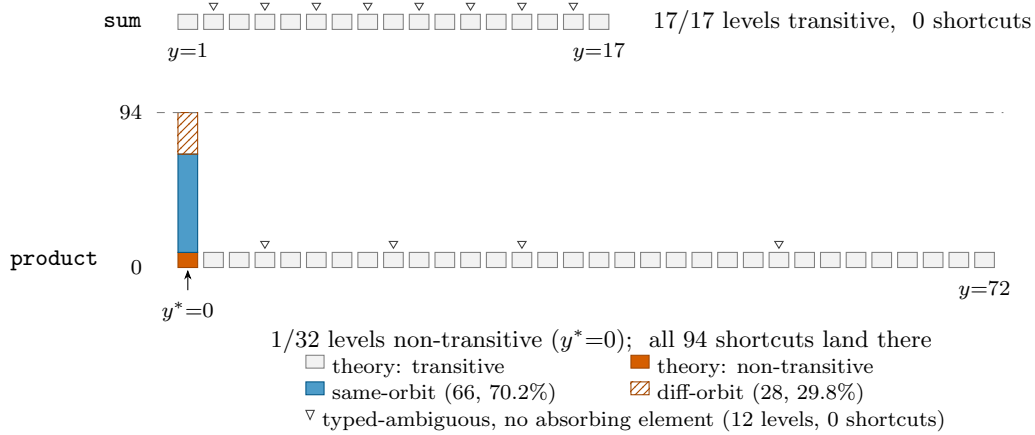
\begin{figure}[t]
\centering
%
%
\definecolor{figblue}{HTML}{0072B2}
\definecolor{figverm}{HTML}{D55E00}
\tikzset{
  fscell/.style={draw=figmid, thin, fill=figpale},
  fscellflag/.style={draw=figverm!75!black, thin, fill=figverm},
  fsbarsame/.style={draw=figblue!75!black, thin, fill=figblue!70},
  fsbardiff/.style={draw=figverm!75!black, thin, fill=figverm!12,
    pattern=north east lines, pattern color=figverm!80!black},
  fstick/.style={font=\scriptsize, align=center},
  fsrowcap/.style={font=\footnotesize, align=center},
  fsrowlab/.style={font=\footnotesize\ttfamily, anchor=east},
  fslegend/.style={font=\scriptsize, anchor=west},
  fsguide/.style={draw=figmid, thin, dashed},
  fstri/.style={draw=figdark, thin, fill=white}
}

\begin{tikzpicture}[x=1cm, y=1cm]
  \def\cellw{0.26}
  \def\pitch{0.34}
  \def\hbase{0.20}

  \foreach \i in {0,...,16} {
    \pgfmathsetmacro{\x}{\i*\pitch}
    \draw[fscell] (\x,0) rectangle ++(\cellw,\hbase);
  }
  \node[fstick, anchor=north] at (0*\pitch+0.13, -0.05) {$y{=}1$};
  \node[fstick, anchor=north] at (16*\pitch+0.13, -0.05) {$y{=}17$};
  \foreach \i in {1,3,5,7,9,11,13,15} {
    \pgfmathsetmacro{\cx}{\i*\pitch+0.13}
    \draw[fstri] (\cx-0.055,\hbase+0.145) -- (\cx+0.055,\hbase+0.145)
      -- (\cx,\hbase+0.055) -- cycle;
  }
  \node[fsrowlab] at (-0.35, \hbase*0.5) {sum};
  \node[fsrowcap, anchor=west] at (16*\pitch+\cellw+0.45, \hbase*0.5)
    {17/17 levels transitive, \ 0 shortcuts};

  \begin{scope}[yshift=-3.15cm]
    \draw[fscellflag] (0,0) rectangle ++(\cellw,\hbase);
    \draw[fsbarsame]  (0,\hbase) rectangle ++(\cellw,1.30);
    \draw[fsbardiff]  (0,\hbase+1.30) rectangle ++(\cellw,0.55);

    \foreach \i in {1,...,31} {
      \pgfmathsetmacro{\x}{\i*\pitch}
      \draw[fscell] (\x,0) rectangle ++(\cellw,\hbase);
    }

    \draw[fsguide] (-0.28,\hbase+1.85) -- (31*\pitch+\cellw+0.12,\hbase+1.85);
    \node[fstick, anchor=east] at (-0.34,\hbase+1.85) {94};
    \node[fstick, anchor=east] at (-0.34,0) {0};
    \node[fsrowlab] at (-0.92, \hbase*0.5) {product};

    \draw[-{Stealth[length=1.3mm]}, draw=figblack, thin]
      (0.13,-0.30) -- (0.13,-0.03);
    \node[fstick, anchor=north, font=\bfseries\scriptsize] at (0*\pitch+0.13, -0.32)
      {$y^{*}{=}0$};
    \node[fstick, anchor=north] at (31*\pitch+0.13, -0.05) {$y{=}72$};
    \foreach \i in {3,8,13,23} {
      \pgfmathsetmacro{\cx}{\i*\pitch+0.13}
      \draw[fstri] (\cx-0.055,\hbase+0.145) -- (\cx+0.055,\hbase+0.145)
        -- (\cx,\hbase+0.055) -- cycle;
    }

    \node[fsrowcap, anchor=north] at (5.60,-0.68)
      {1/32 levels non-transitive ($y^*{=}0$); \ all 94 shortcuts land there};

    \node[fslegend] at (1.55,-1.30)
      {\tikz[baseline=-0.5ex]{\fill[figpale, draw=figmid, thin] (0,0) rectangle (0.24,0.17);}~theory: transitive};
    \node[fslegend] at (5.85,-1.30)
      {\tikz[baseline=-0.5ex]{\fill[figverm, draw=figverm!75!black, thin] (0,0) rectangle (0.24,0.17);}~theory: non-transitive};
    \node[fslegend] at (1.55,-1.66)
      {\tikz[baseline=-0.5ex]{\fill[figblue!70, draw=figblue!75!black, thin] (0,0) rectangle (0.24,0.17);}~same-orbit (66, $70.2\%$)};
    \node[fslegend] at (5.85,-1.66)
      {\tikz[baseline=-0.5ex]{\fill[figverm!12, pattern=north east lines, pattern color=figverm!80!black, draw=figverm!75!black, thin] (0,0) rectangle (0.24,0.17);}~diff-orbit (28, $29.8\%$)};
    \node[fslegend] at (1.55,-2.02)
      {\tikz[baseline=-0.5ex]{\draw[figdark, thin, fill=white] (0,0.14) -- (0.11,0.14) -- (0.055,0.02) -- cycle;}~typed-ambiguous, no absorbing element (12 levels, 0 shortcuts)};
  \end{scope}

\end{tikzpicture}
\caption{Observed concept-level shortcuts of the shared-network arm
(bar height, stacked by same-/different-orbit) against the componentwise
transitivity of each of the 49 non-singleton target levels (\texttt{sum}: 17, top;
\texttt{product}: 32, bottom; dark cell background marks the one level
Section~\ref{sec:empirical-extension} proves componentwise-non-transitive;
open triangles mark the levels the typed reading
(Definition~\ref{def:typed-autx}) additionally flags as ambiguous, the
dark $y^*{=}0$ cell being the thirteenth). Every one of the 94 observed
shortcuts lands on the single level flagged by both readings and backed
by the absorbing element, split 66 same-orbit ($70.2\%$) and 28
different-orbit ($29.8\%$); the other 48 componentwise-transitive levels
produce zero over their 34{,}800 model-instance evaluations, including
all twelve typed-ambiguous levels without the absorbing element
(12{,}000 evaluations; \Cref{sec:realmodels-threelayers}).}
\label{fig:realmodel-shortcuts}
\end{figure}

Every one of the ten \texttt{product} seeds produces both same-orbit and different-orbit shortcuts (same-orbit counts 4--9 and diff-orbit counts 2--4 per seed), not just a pooled aggregate driven by one atypical run.
Pooled, $66/94{=}70.2\%$ of \texttt{product}'s shortcuts share the ground truth's orbit and $28/94{=}29.8\%$ do not; per seed, the diff-orbit share ranges from $20\%$ to $50\%$, never zero and never dominant. This stable $70.2/29.8$ orbit split, reproduced in every seed, is the empirical target the mechanism analysis that follows must explain.

\subsection{Mechanism: the absorbing element behind both outcomes}
\label{sec:realmodels-mechanism}

\begin{table}[t]
\centering
\small
\begin{tabular}{p{2.6cm}p{5.4cm}p{5.6cm}}
\toprule
Outcome & Pattern & Count \\
\midrule
Same-orbit (66) & the zero-valued slot is recognized correctly; the paired slot's true non-zero digit is misclassified as some \emph{other} non-zero digit & 28 from $\{(0,k)\}$, 38 from $\{(k,0)\}$ \\
Diff-orbit, into $\{(0,0)\}$ (26) & a true non-zero digit paired with a true zero is misclassified specifically \emph{as zero} & 17 from $\{(0,k)\}$, 9 from $\{(k,0)\}$ \\
Diff-orbit, out of $\{(0,0)\}$ (2) & both true digits are zero; one is misclassified as non-zero, the other's confidently-correct zero keeps the product at 0 & 2 into $\{(0,k)\}$ \\
\bottomrule
\end{tabular}
\caption{Element-level breakdown of the shared-network arm's 94 \texttt{product} $y^*{=}0$ shortcuts, pooled over the 10 seeds, by which of the three orbits $\{(0,0)\}$, $\{(0,k){:}k{\ne}0\}$, $\{(k,0){:}k{\ne}0\}$ the true and predicted pairs occupy.}
\label{tab:realmodel-mechanism}
\end{table}

\Cref{tab:realmodel-mechanism} breaks all 94 shortcuts down by the orbits their true and predicted pairs occupy. Concrete instances: same-orbit shortcuts include true $(0,1)\to$ predicted $(0,3)$ and true $(7,0)\to$ predicted $(2,0)$; different-orbit shortcuts are, with only two exceptions out of 28, of the form true $(0,3)\to$ predicted $(0,0)$ or true $(8,0)\to$ predicted $(0,0)$: the paired non-zero digit collapses specifically onto zero itself, crossing into the singleton orbit rather than landing on some other non-zero digit.

This is not a post-hoc story fitted to the numbers; it follows from the loss function's algebra, stated before the mechanism was inspected. $P(Y{=}0\mid x_1,x_2) = 1-(1-p_1(0))(1-p_2(0))$ already approaches 1 whenever \emph{either} factor's zero-probability does, so once the network confidently recognizes one slot as zero, the semantic loss carries essentially no gradient signal about what the other slot's network output actually is. The proof in \Cref{sec:empirical-extension} that $\{(0,0)\}$, $\{(0,k)\}$, and $\{(k,0)\}$ are separate orbits is exactly this same absorbing-element fact, read off $\AutX$ instead of a loss gradient: a bijection fixing $0$ in one coordinate cannot also send it to a non-zero value, so no independent-coordinate automorphism can merge the three. One structural fact produces both the symbolic orbit split and the real optimizer's failure to correct the under-constrained slot; whether the resulting drift lands on another non-zero digit (same-orbit) or on zero itself (26 of 28 diff-orbit cases) is the one part of the outcome $\AutX$ does not speak to.

\subsection{Three layers: what symmetry permits, architecture realizes, and optimization selects}
\label{sec:realmodels-threelayers}

The exact correspondence above is a statement about the
\emph{componentwise} reading. Replaying the same 49 levels under the
typed reading of \Cref{def:typed-autx}, the one matched to the shared
perception network, whose output relabelings act on both digit slots
in lockstep, changes the prediction table dramatically: the diagonal action preserves whether the two digit
slots agree, so every level containing a doubled pair splits. Exactly
thirteen levels are typed-non-transitive, computed level by level: the
eight even \texttt{sum} levels $y\in\{2,4,\dots,16\}$ (each isolates its
doubled pair $(y/2,y/2)$ in its own orbit) and the five \texttt{product}
levels $\{0,4,9,16,36\}$ (the square levels, plus zero).
These fiber-level splits are per-instance ambiguities, not
loss-preserving reparameterizations of the task: the arithmetic rules'
global groups are trivial (\Cref{prop:fiber-intersection}), so no fixed
relabeling preserves the whole task, and what the table marks is where
a single evaluation's label-preserving error can cross an orbit
boundary.

The trained models now deliver a verdict the symbolic analysis alone
never could. Of those thirteen typed-ambiguous levels, twelve produced
\emph{zero} shortcuts across their 12{,}000 model-instance evaluations
($9{,}600$ at the even \texttt{sum} levels, $2{,}400$ at the four square
\texttt{product} levels); the thirteenth, $y^*{=}0$, produced all 94, in
every seed. \texttt{sum} $y{=}2$ makes the contrast concrete: its typed
orbit split, $\{(1,1)\}$ against $\{(0,2),(2,0)\}$, is exactly as real
as $y^*{=}0$'s, a label-preserving error can cross it on any $y{=}2$
instance, and in 600 evaluations no model ever produced one.

The architecture layer is not left as an inference. We reran the entire
experiment with the one change the three-layer reading singles out:
each digit position gets its own independently initialized,
independently weighted perception network, so the
architecture-realizable relabelings become the full componentwise group
rather than the typed subgroup. Everything else, pools, pairs,
evaluation set, orbit ground truth, hyperparameters, and the ten seeds,
is held fixed, and two predictions were recorded in the script:
shortcut location is selected by the
optimization layer and should not move, and the $y^*{=}0$ mechanism
should persist. The dual-head models produced 101 \texttt{product}
shortcuts, every one at $y^*{=}0$, in all ten seeds; the twelve
typed-ambiguous levels produced zero, again, across their further
12{,}000 evaluations; the same-orbit share barely moved, $72/101 =
71.3\%$ against the shared network's $70.2\%$; and the collapse
mechanism sharpened, with all 29 different-orbit cases landing in
$\{(0,0)\}$ (the shared run had 26 of 28). One event missed the
pre-specified prediction, and we report it rather than round it away:
\texttt{sum} produced a single label-preserving shortcut in its
20{,}000 evaluations, true $(2,7)$ predicted as $(7,2)$, a pure swap.
It is the only such event in 40{,}000 dual-head evaluations, and it is
same-orbit; it is not a componentwise-only form, since the diagonal
permutation exchanging $2$ and $7$ already stabilizes the $y{=}9$
fiber, so the typed reading classifies the swap identically. The event
stresses the geography without demonstrating a new architectural
degree of freedom, and the symmetry analysis already classifies it as
explained.

Three further measurements pin down the optimization layer: a null
model that fixes what the geometry alone accounts for, a relocation of
the mechanism, and a reading in probability mass rather than argmax.

\paragraph{Geometry and optimization at the absorbing level.}
Take each trained model's own per-digit confusion distribution and
sample independent slot errors on the same evaluation grid. Under this
null, $99.9\%$ of the label-preserving errors it produces already fall
at $y=0$: the absorbing element makes that level the easy landing site
for label-preserving error, because holding the product fixed after an
error requires only that one slot stay zero, while every other level
requires two coordinated errors. The location fact therefore follows
from the same absorbing-element geometry that splits the orbits,
acting on the models' own marginal error rates over this evaluation
grid, exactly as \Cref{sec:realmodels-mechanism} argues, and is not
independent evidence about optimization. The \emph{rate} is where the
trained models part from the null: they produce $9.4$ label-preserving
errors per seed where the null predicts $6.2$, an excess of $3.2$ per
seed (seed-clustered bootstrap 95\% CI $[2.3,4.2]$, positive in
$10/10$ seeds, pooled Poisson $z=4.1$). The models land on the
geometry's easy site half again as often as independent slot errors at
their own confusion rates would. What the excess rejects is the null
as a whole, slot errors drawn independently from each model's pooled
per-digit confusion rates, not cross-slot independence in isolation;
that the departure reflects optimizer-induced coordination rather than
context-dependent marginals is the reading the loss algebra supports,
not a further measured fact.

\paragraph{Relocation of the absorbing element.}
The sharpest test available is to move the mechanism and see whether the
failures follow. Relabeling the product task through a digit permutation
$\pi$, $a\star b:=q(\pi(a)\,\pi(b))$ with $q$ a bijection of the
reachable products agreeing with $\pi^{-1}$ on $\{0,\dots,9\}$ (products
above $9$ receive fresh labels, exactly as the released script
executes), gives an isomorphic operation whose absorbing value is
$\pi^{-1}(0)$ rather than $0$; with $\pi(v)=v+3\bmod 10$ that value is
$7$. Retraining all ten seeds on
$\star$, with everything else held fixed and the prediction recorded
in the script, every one of the 33 label-preserving shortcuts occurs at
$y^*{=}7$ and none at $y{=}0$, with the same-orbit share at $78.8\%$,
inside the band the unconjugated task measures. Under this tested
relocation the failure geography is not attached to the digit zero; it
follows the value the rule makes absorbing.

\paragraph{The predictive-mass decomposition.}
Argmax is a coarse readout of a distributional model, so we also
decompose the model's predictive mass at each instance into the
ground-truth orbit, the label-preserving alternatives outside it, and
the label-breaking remainder. At the $48$ transitive levels the middle
term is identically zero: there is no orbit-external alternative for
mass to sit on. At $y^*{=}0$ it is $0.75\%$ on average against
$98.96\%$ on the true orbit, so the semantic loss leaves real mass on
orbit-external solutions precisely where the theory says such solutions
exist, at a magnitude comparable to the observed argmax shortcut rate.

\paragraph{Two interventions that leave the rate unchanged.}
Two further arms probe where in training the absorbing element does
its work. Adding an entropy penalty on the under-determined slot of
$y{=}0$ pairs leaves the rate at 95 shortcuts against a 93-shortcut
baseline; as the released script itself records, the penalty supplies
gradient magnitude but no direction, so this arm never delivered the
corrective signal it was built to test. Supervising the true concept
on a quarter of those pairs, a directional signal, leaves the rate at
101 against 94 at $99.5\%$ label accuracy, a single-dose null with no
equivalence margin. Neither intervention reduced the rate, and both
null results are consistent with the reading the loss algebra
predicts: with perception shared across slots, each digit's
representation is learned from the $90\%$ of its pair occurrences
whose partner digit is non-zero, and is never degraded to begin with;
what the absorbing element removes is correction pressure on an
already-adequate representation, at exactly the instances where
$P(Y{=}0\mid x_1,x_2)$ stops depending on the second slot.

The three layers this separates are worth naming, because they are the
paper's answer to what an orbit computation is and is not for.
The \emph{symmetry} layer says which alternative solutions are
structurally interchangeable: componentwise draws the most generous
boundary, typed the architecture-matched one. The \emph{architecture}
layer says which form a systematic misgrounding can take: a shared
CNN's output relabelings act on both slots in lockstep, the typed
form, while independent heads can drift per slot, the componentwise
form; since the global groups are trivial, these are forms available
to a per-instance error, not free reparameterizations of a trained
model. The
\emph{optimization} layer says which ambiguities training actually
falls into, and the four measurements above separate its contribution
from the geometry's: the absorbing element fixes \emph{where}
label-preserving failure is cheap (the confusion null, built from
independent slot errors at the models' own confusion rates,
concentrates at the same level, and relocating the element moves the
site), while the trained models exceed the independent-error rate at
that site by half again, together with
the orbit-external probability mass that exists only where the theory
says orbit-external solutions do. Twelve typed-ambiguous levels without
an absorbing element stay clean across $12{,}000$ evaluations; one
level with both fails in all ten seeds, under two architectures, and
follows the absorbing value under the tested relocation. Bare orbit splits describe
where shortcuts \emph{can} hide; the absorbing element decides where
they \emph{do}; the optimizer decides how often. That every observed
\texttt{product} shortcut landed on the one level the componentwise
table flags is explained rather than lucky, and it is not a promise
transitivity ever made: a transitive level's shortcuts would simply
all be same-orbit, exactly what the dual-head \texttt{sum} swap, the
one event at any transitive level, went on to show. On this
family the only level where componentwise and typed ambiguity coincide
is also the only one carrying an absorbing element, so the coarsest
instrument inherits the credit for a location the absorbing element
selected. Read together, the twelve empty levels and the
excess rate at the thirteenth are this paper's cleanest evidence that
identifiability analysis and learning-dynamics analysis are different
instruments, and that a symmetry verdict alone, at any rung, is a map
of candidate failure sites rather than a forecast.

\subsection{Heterogeneous domains end to end: CLE4EVR}
\label{sec:realmodels-cle4evr}

The methodological argument of \Cref{sec:background} is about
heterogeneous concept domains, so the closing experiment trains on one.
We take CLE4EVR's public rule and domains verbatim, the same instance
Table~\ref{tab:clevr-comparison} measures symbolically: two objects with
color, shape, and material domains of sizes $2$, $3$, $2$, and the rule
requiring all three attributes to match. Perception is a synthetic front
end, one random prototype vector per attribute value plus Gaussian
noise, with one head per attribute type shared across objects in the
typed arm and one head per position in the componentwise arm. The noise
level is set high enough that label-preserving concept errors are
frequent, which is what the orbit question needs statistical power for;
the front end is synthetic by design rather than a model of image
difficulty. Supervision is the binary rule label
alone, through the exact semantic loss over $\Phic$. Nothing about
attribute identity is ever supervised. Evaluation draws fresh instances
whose true attributes satisfy the rule, $1{,}500$ per seed: the
positive fiber, where $\Phic$ and its orbit structure are defined and
where a label-preserving error means a predicted tuple in $\Phic$
differing from the truth.

Two predictions were recorded in the script. The first follows from
transitivity: $\AutX$ is transitive on CLE4EVR's $\Phic$, so every
label-preserving concept error any predictor makes on those instances
must be same-orbit. Across ten seeds it is: $10{,}154$ label-preserving
errors in the typed arm and $10{,}069$ in the componentwise arm, with
zero different-orbit exceptions, validating the symbolic orbit
computation and the error-classification pipeline end to end against
$20{,}223$ trained-model errors.

The second prediction is where the trained models carry evidential
weight, because its value depends on where their errors actually land:
the padded diagonal extension of \Cref{sec:padding-false-positive},
which reports $90.91\%$ of this instance's solution pairs as
orbit-external, misclassifies these real model errors at $88.5\%$
in the typed arm and $78.1\%$ in the componentwise arm. The $90.91\%$
of \Cref{tab:clevr-comparison} is the uniform rate over unordered
solution pairs; on the error distributions real training produces, the
padded instrument would report $78.1$--$88.5\%$ of a trained model's
genuine, symmetry-explained shortcuts as unexplained pathology.

\subsection{Scope}
\label{sec:realmodels-limits}

\texttt{sum} produces essentially no label-preserving shortcut under this architecture and training regime, so its orbit question has nothing to classify: the zero there is a fact about the task, not a verdict on $\AutX$. Whether shortcuts occur is itself governed by task structure, here the multiplicative zero that rsbench's product task has and that additive structure does not, consistent with this paper's recurring finding that structure, not task family, determines pathology; this time the structure determines whether a real optimizer finds a pathology to fall into in the first place, not just whether one exists combinatorially.

Training used uncurated random pairing specifically so that any observed shortcut would not be an artifact of a shortcut-inducing split; \texttt{product} still produced one in every seed. The reported rates (0.33\% same-orbit, 0.14\% diff-orbit) are the rates this uncurated random-pairing regime produces, without the deliberately restricted training coverage of rsbench's own \texttt{shortcutmnist} \cite{rsbench2024}. The 40{,}000 figures are model-instance evaluations, not independent samples: the same 2{,}000 held-out instances are scored by all ten seeds by design, so that seeds are compared on one yardstick, and per-seed counts are reported wherever a rate is pooled. Rates are therefore reported with the seed as the unit of analysis: seed-clustered bootstrap 95\% confidence intervals are $[8.6,10.4]$ shortcuts per seed for the shared network and $[9.2,11.0]$ for the dual-head control, a paired difference of $0.7$ with interval $[-0.7,2.0]$ that does not separate the two architectures, and $[24.5\%,36.3\%]$ for the different-orbit share whose pooled value is $29.8\%$. At the $48$ transitive levels the observation is zero shortcuts in $34{,}800$ evaluations; because those evaluations reuse the same held-out instances across seeds, no independence-based confidence bound is attached to the zero, which remains an observation, not proof of a zero event rate.
The shared-network and dual-head runs together cover both ends of the
architecture rung (typed and componentwise,
\Cref{sec:realmodels-threelayers}); the shortcut geography survived
the change intact, so widening the architecture-realizable group
leaves the landing site unmoved: whatever selects it, it is not the
architecture layer. The orbit classification
of the observed shortcuts is likewise rung-independent: on $\PhiY{0}$
the typed and componentwise orbit partitions coincide, both giving
exactly $\{(0,0)\}$, $\{(0,k)\}$, $\{(k,0)\}$.
\Cref{tab:three-families,tab:eight-families}'s figures for the other families are symbolic measurements in the sense of \Cref{sec:empirical}.
\section{Related Work}
\label{sec:related}

\paragraph{Reasoning shortcuts and Takemura et al.'s framework.}
Marconato, Teso, Vergari, and Passerini characterize reasoning shortcuts
as unintended optima of a neurosymbolic training objective and give four
conditions under which they occur~\citeyearonly{marconato2023shortcuts}; their
analysis is distributional, over what a learner can converge to, and
never introduces a symmetry group. The same group's recent JAIR survey
of reasoning shortcuts and symbol grounding~\cite{marconato2026gentle}
maps the phenomenon's causes and mitigations across the field; the
symmetry instrument built here supplies the piece that survey's
landscape does not contain, a per-rule algebraic account of which
alternative groundings are structurally interchangeable. Their rsbench suite supplies the
majority of the real rule families measured in \Cref{sec:empirical}: the
CLE4EVR, Kandinsky, and BDD-OIA/SDD-OIA core measurements and five of
the eight extension families~\cite{rsbench2024}. Takemura, Inoue, and
Nishino introduce a permutation-group account of this problem,
Definition~\ref{def:takemura-autx}, and leave a complete
characterization of uniqueness beyond triviality of that group as their
paper's most pressing open question~\citeyearonly{takemura2026}. \Cref{sec:background} shows
their definition does not extend to any of the four benchmarks it was
evaluated on without an embedding step the framework itself does not
specify, and that the most
direct such step gives a confident wrong answer rather than a
conservative one (Table~\ref{tab:clevr-comparison}). \Cref{sec:algebra,sec:complexity}
take up the componentwise, up-to-symmetry analogue of that question:
sufficient, checked conditions for transitivity and its failure in
general, an exact classification in the Boolean case
(\Cref{thm:boolean-transitivity}), and a complexity landscape for the
two detection problems any such audit rests on; the general
characterization remains open, and
\Cref{sec:algebra} notes it has no clean form even for the classical,
decades-older special case of Latin-square autotopism groups.

\paragraph{Classical CSP symmetry.}
Cohen, Jeavons, Jefferson, Petrie, and Smith give the vocabulary this
paper builds on, separating constraint symmetry from solution symmetry
and proving the former is a subgroup of the
latter~\citeyearonly{cohen2006symmetry}; \Cref{sec:background} already used their Example~3
to justify working with $\Gind$ instead of their unrestricted
variable-value-pair permutations, which degenerate to an
uninformative $n!\,(n(d{-}1))!$ on near-unique solution sets. Puget
detects symmetries algorithmically, by encoding a CSP instance into a
colored graph and calling a graph-automorphism solver, covering global
constraints and arithmetic
expressions~\citeyearonly{puget2005automatic}. The lineage of exploiting such
groups goes back to Crawford, Ginsberg, Luks, and Roy's
symmetry-breaking predicates, which cut the searched space to orbit
representatives once the group is known~\citeyearonly{crawford1996symmetry};
this paper asks the converse question, what the orbits fail to cover.
The transformations themselves are classical: a per-position value
permutation is a special case of the solution symmetries Cohen et
al.\ admit and of the symmetries a detector like Puget's can return on
an explicit instance. What Definition~\ref{def:componentwise-autx}
contributes is not a new kind of permutation but the selection of
exactly this subgroup as the analysis instrument on heterogeneous
domains, where the shared-domain Definition~\ref{def:takemura-autx} is
undefined, together with the orbit-structure and complexity questions
this paper asks of it.
Gent, Petrie, and Puget survey
the field and are the classical reference Takemura et al. cite for the
connection their own Definition~7 does not fully use~\citeyearonly{gentpetriepuget2006symmetry}.
All three ask how to detect or use symmetry given an explicit
instance; none gives a complexity classification for the existence
question itself, and none proves a sufficient condition for transitivity
in the sense of Theorems~\ref{thm:matching-decomposition}
through~\ref{thm:free-slot}. Their framework characterizes which
permutations qualify as symmetries; the questions answered here, the
orbit structure a fixed solution set carries and provable conditions
for a group to act transitively on it, are not questions that line
poses.

\paragraph{Circuit and formula isomorphism.}
Borchert, Ranjan, and Stephan open the question of the complexity of
deciding whether two Boolean objects are related by a \emph{position}
permutation, and Agrawal and Thierauf resolve one direction of
it~\cite{borchertranjanstephan1998,agrawalthierauf2000formula};
B\"ohler, Creignou, Galota, Reith, Schnoor, and Vollmer extend the
classification across Post's lattice and confirm the isomorphism
question itself is still open sixteen years later~\citeyearonly{bohler2012complexity}.
This is the same coNP-hard-to-$\SigmaTwoP$ shape our \NONTRIV{}
occupies, using the same frozen-variable hardness technique, but for a
different group ($\Sym(N)$, permuting named positions, rather than
$\Gind$, permuting values independently at each position) and a
different question (two objects compared, rather than one object's own
automorphisms). To our knowledge the complexity of
\emph{translation}-automorphism existence, for general or for monotone
circuits, does not appear in this line or elsewhere; the nearest
objects are the permutation-and-negation (NPN) equivalences charted by
Borchert et al.~\citeyearonly{borchertranjanstephan1998}, and
\Cref{thm:monotone-nontriv}'s deterministic completeness on monotone
circuits has no analogue there.
Section~\ref{sec:complexity-valiso} transplants their
collapse theorem to the value-permutation group directly;
\Cref{sec:complexity-theorem14} and \Cref{sec:complexity-diagnosis}
report that their own labeling technique for descending from the
two-object to the one-object question does not transplant, because it
needs a group that permutes named objects and $\Gind$ does not, and
diagnose what does.

\paragraph{Boolean function analysis and linear structures.}
O'Donnell's Fourier-analytic treatment of Boolean
functions~\citeyearonly{odonnell2014analysis} supplies the exact language of
Lemma~\ref{lem:fourier-characterization}: $\AutX$, in the Boolean case,
is the orthogonal complement of the Fourier support, not merely an
analogy to it. Rothaus's bent functions~\citeyearonly{rothaus1976bent}, maximally
non-linear in the sense of having no nonzero linear structure, are the
building block behind Proposition~\ref{prop:bent-witness}'s witness
construction. The cryptographic literature on linear structures (S-boxes
invariant under a fixed input XOR) works with explicit truth tables,
constructing designed functions with few or zero linear structures
and characterizing their existence through the Walsh transform. The
separation is by representation and question: here the rule arrives
as a \emph{succinctly given} circuit, and the question is the
classical complexity of deciding whether any nonzero linear structure
exists at all, trivial once the truth table is in hand and coNP-hard
once it is not (Theorem~\ref{thm:deadvar-conp-complete}).

\paragraph{Latin square autotopism groups.}
The autotopism group of a Latin square, the subgroup of
$\Sym(\text{rows})\times\Sym(\text{columns})\times\Sym(\text{symbols})$
stabilizing its defining relation, is the special case of $\AutX$
studied since Section~\ref{sec:algebra-setup}'s design-theory remark.
McKay, Meynert, and Myrvold reduce computing it to graph automorphism
through a vertex-colored encoding, enumerate autotopism groups for all
Latin squares up to order 10, and find no closed-form classification
even at that scale~\citeyearonly{mckay2007small}. This is the same pattern
Section~\ref{sec:algebra} finds for reasoning-shortcut constraint sets:
Theorem~\ref{thm:orbit-stabilizer} already rules out transitivity for
every measured all-different instance by counting alone, matching the
combinatorial-design literature's own experience that exhaustive
computation, not a general theorem, is the working tool once the
instance is not built from an equality or absorbing-element pattern
simple enough for Theorems~\ref{thm:matching-decomposition}
through~\ref{thm:degree-invariant} to reach.

\section{Conclusion}
\label{sec:conclusion}

Whether the automorphisms of a rule explain its reasoning shortcuts is
not one question but three. Which group? The published global
definition does not apply as stated to any heterogeneous benchmark it
was evaluated on; using it requires an embedding step the framework does
not specify, the most direct embedding produces measured false pathology
whose content is configuration-file bookkeeping, and componentwise value
symmetry is the instrument that needs no embedding at all
(\Cref{sec:background}). When do orbits
explain the shortcuts? Not uniformly. Across eleven rule families the
answer spans $0\%$ to $99.9999\%$, and it tracks provable structure:
matching decompositions and exchangeable branches force transitivity,
while anchored inequalities, absorbing elements, counting bounds and
occupied free slots force its failure
(\Cref{sec:empirical,sec:algebra}). In the Boolean case the second
question closes completely: the group acts freely, transitivity holds
exactly for affine solution sets, and the orbit law
$\rho=1-(|\AutX|-1)/(|\Phic|-1)$ turns \Cref{sec:empirical}'s measured
percentages into a theorem's values. How hard is the symmetry to
detect?
coNP-complete for a designated coordinate, and for existence, coNP-hard
under randomized reductions, in the Boolean case not
$\SigmaTwoP$-complete unless the polynomial hierarchy collapses, and
with the gap closing entirely, to deterministic coNP-completeness, on
monotone circuits (\Cref{sec:complexity}). And the theory
is not merely internally consistent: trained models place all $195$ of
their observed \texttt{product} shortcuts, under both architectures, at
exactly the level the componentwise analysis flags, while the twelve
extra candidates of the typed reading stay empty and the one remaining
event in $80{,}000$ evaluations, a dual-head \texttt{sum} swap, lands
same-orbit, a failure the analysis already classifies as explained;
relocating the absorbing
element moves the whole geography with it, and on CLE4EVR's
heterogeneous domains every one of $20{,}223$ label-preserving errors
falls in the orbit transitivity predicts, where the padded instrument
would have called most of them pathology
(\Cref{sec:realmodels}).

Two findings deserve to outlive the paper's specific numbers. The
first is methodological. A definition applied outside its stated
hypotheses did not fail loudly on real data. It failed quietly,
returning a stable-looking
$90.91\%$ whose content rotated with an arbitrary file ordering. The
padding sweep suggests this failure mode is generic: any instrument
ported outside its stated domain by an unexamined embedding can fail
the same way.
Checking the embedding, not just the instrument, is the transferable
lesson. The second is structural. The arbitrary-domain level
decomposition (\Cref{prop:level-decomposition}) unifies dead
coordinates and degree imbalance with Boolean Fourier support; only
the stabilizer--orthogonal-complement step is Boolean-specific, so the
boundary between the orbit-combinatorial and linear-algebraic toolkits
is exact, not a matter of taste.

\appendix
\crefalias{section}{appendix}
\section{The Obstruction Landscape of \NONTRIV{}}
\label{sec:complexity-diagnosis}

$\NONTRIV$'s complexity is still not pinned to a single class: we know
it is coNP-hard under randomized reductions
(Corollary~\ref{cor:nontriv-conp-hard}), in $\SigmaTwoP$
(Theorem~\ref{thm:nontriv-sigma2p}), and not
$\SigmaTwoP$-complete unless $\PH$ collapses
(Theorem~\ref{thm:nontriv-collapse}). Seven
attempts to close what gap remains failed before the eighth reached that
last result, and the seven failures are not seven unrelated dead ends.

\begin{proposition}[Witnesses are not subgroups]
\label[proposition]{prop:witness-not-subgroup}
There is no reduction from $\Sigma_2\textup{\textsc{-sat}}$ to $\NONTRIV$ of the
form ``encode a candidate witness $x\in\{0,1\}^p$ directly as a
coordinate shift $a_x\in\{0,1\}^p$, so that $a_x\in L_0(f)$ exactly when
$\forall y\,\varphi(x,y)$'': for a fixed instance $\varphi$, the target
set $T:=\{x:\forall y\,\varphi(x,y)\}$ need not be a subgroup of
$\{0,1\}^p$, while $L_0(f)$, for any $f$ whatsoever, always is.
\end{proposition}

\begin{proof}
$L_0(f)$ is a stabilizer of the coordinatewise action of $\{0,1\}^p$ on
functions, hence a subgroup for any $f$: this holds before any
reduction-specific reasoning. A subgroup has order dividing $2^p$ and
contains $0$; on 216 random small $\Sigma_2$-\textsc{sat} instances
($p,q\in\{1,2,3\}$), 50 (23.1\%) have a $T$ whose size does not divide
$2^p$, ruling out any such $T$ from being a subgroup on the spot, and a
direct implementation of the naive encoding disagrees with the true
truth value of $\varphi$ on 201 of 216 instances (93.1\%).
\end{proof}

The seven attempts, spanning direct encoding, two indirect encodings,
literature transplant, and interactive protocols:
\begin{enumerate}
\item[1.] \emph{Direct witness encoding} (Proposition~\ref{prop:witness-not-subgroup}):
blocked because coordinatewise stabilizers are always subgroups and
generic $\Sigma_2$-\textsc{sat} witness sets are not.
\item[2--3.] \emph{Two indirect encodings} (extra switch bit; XORed
copies with bent-function padding): $96.0\%$ and $73.3\%$ mismatch
rates, since both still let the witness set control a coordinate block,
inheriting the same subgroup obstruction one layer down.
\item[4.] \emph{Small-Hamming-weight witnesses}
(Proposition~\ref{prop:bent-witness}): the strongest form is false by
explicit construction.
\item[5.] \emph{Agrawal--Thierauf's labeling gadget, transplanted} (four
tagging constructions): $28.9\%$ to $64.4\%$ mismatch rates, since the
gadget needs a group that permutes named objects, and coordinatewise
$\mathrm{XOR}$ does not.
\item[6.] \emph{Direct hashing of $\AutX$}: technically sound but
information-free, since it proves $\NONTRIV\in\mathsf{BP}\cdot\SigmaTwoP$,
already implied for free by Theorem~\ref{thm:nontriv-sigma2p}.
\item[7.] \emph{The flip-bit decomposition} (Proposition~\ref{prop:qi-gradient}):
tautologically correct but every disjunct below the last is exactly as
hard as the whole.
\end{enumerate}
Every one of these seven failures traces back to one of three structural
facts, not to seven unrelated missed tricks:
\begin{enumerate}
\item[(i)] $\AutX$ is always a subgroup; a generic $\SigmaTwoP$ witness
set is not (attempt 1).
\item[(ii)] $\Gind$ acts transitively on $\prod_iS_i$ (and, in the
Boolean fragment where every hardness construction lives, freely, as
$\Ftwo^n$ translations); it
does not permute a set of named objects the way $\Sym(N)$ does, so
techniques that need to distinguish ``object $i$'' from ``object $j$''
have nothing to grab onto (attempts 2, 3, and 5).
\item[(iii)] Any access to $L_0(f)$ stronger than a pointwise oracle
(an explicit basis, or a syntactically exposed witness) is already at
least as hard to obtain as solving $\NONTRIV$ itself (attempts 4 and
6, and the residual, non-collapsing disjuncts of attempt 7).
\end{enumerate}
Interactive-proof and structural-isomorphism toolkits are built for
(ii): they turn one-object questions into two-object comparisons by
naming and tagging, at the cost of needing a group that permutes named
objects. Linear-algebraic toolkits are built for (iii): they turn
implicit membership queries into explicit structure, at the cost of
needing exactly the access $\NONTRIV$'s succinct representation denies.
$\NONTRIV$ combines succinctness with single-object existence, and no
toolkit here is built for both at once. This is a diagnosis, not a
barrier theorem: we have not shown, and do not claim, that every future
technique must fail this way, only that eight independently designed
attempts, spanning direct encoding, interactive protocols, and linear
algebra, either failed for one of these three reasons or (the eighth,
Theorem~\ref{thm:nontriv-collapse}) found
the one combination the diagnosis does not rule out: laundering an
object's syntax away before hashing it, which sidesteps (ii) by never
naming anything and sidesteps (iii) by never asking for more than a
laundered, single-representative view of one object at a time.
\section{Verification Artifacts}
\label{app:verification}

Sections~\ref{sec:algebra} and~\ref{sec:complexity} mark each
computational check behind a theorem, lemma, proposition, or corollary
with a short pointer, ``Computationally verified; see \Cref{app:verification},''
in place of an inline description of instance counts and outcomes. This
appendix collects those checks in one place. The verification code and
its result artifacts, per-run files included, are released alongside
this paper, and every number
below carries a source-file provenance comment (\texttt{\% prov:}) at
its original point of use in the \LaTeX{} source, recording the exact
script, script section, and reported figures it traces to.
\Cref{tab:verification-artifacts} reports, for each result, the
population of instances it was checked against, the outcome, and the
artifact file responsible. None of these checks substitutes for the
proofs given in the main text: every result listed here is proved there,
and the table documents an independent computational cross-check of that
proof, not the proof itself.

\begingroup
\scriptsize
\setlength{\tabcolsep}{3pt}
\renewcommand{\arraystretch}{1.05}
\begin{longtable}{@{}p{0.95in}p{1.85in}p{1.85in}p{1.35in}@{}}
\caption{Computational-verification artifacts referenced from
Sections~\ref{sec:algebra} and~\ref{sec:complexity}. ``Outcome'' reports
exactly what the cited script measured; it supports, rather than
substitutes for, the adjacent proof.}
\label{tab:verification-artifacts} \\
\toprule
\textbf{Result} & \textbf{Verification scope} & \textbf{Outcome} & \textbf{Artifact} \\
\midrule
\endfirsthead
\multicolumn{4}{l}{\textit{\tablename\ \thetable{} -- continued from previous page}} \\
\toprule
\textbf{Result} & \textbf{Verification scope} & \textbf{Outcome} & \textbf{Artifact} \\
\midrule
\endhead
\midrule
\multicolumn{4}{r}{\textit{continued on next page}} \\
\endfoot
\bottomrule
\endlastfoot

\Cref{def:typed-autx} (hierarchy robustness) &
CLE4EVR 6-dim and 8-dim, Kandinsky 6-dim (exhaustive typed and componentwise groups and orbits, cross-checked against the main pipeline's stored groups); MNAdd-product $\PhiY{0}$ (structural characterization sampled on 4{,}000 random permutation pairs, orbits from generators) &
typed $=$ componentwise exactly on CLE4EVR 6-dim and Kandinsky (groups and orbit partitions); they differ only on CLE4EVR's free size slots (144 vs.\ 864, 2 orbits vs.\ 1); $\PhiY{0}$ orbit partition $[9,9,1]$ identical under both; both measured componentwise groups non-abelian via explicit non-commuting pairs &
\seqsplit{theoryF\_typed\_symmetry.py} \\

\Cref{thm:matching-decomposition} &
184 random equality-block instances (1--4 blocks, 1--3 positions/block, domain sizes 2--4, 0--2 free positions; equality case only, $f_{l,j}=\mathrm{id}$); 116 further instances skipped (ambient group order over budget) &
0 counterexamples &
\seqsplit{theoryB\_theorem1\_verify.py} \\

\Cref{thm:branch-swap} &
two minimal instances (2 attribute slots $\times$ 2 values): symmetric ($|\Phi_C|=12$) and asymmetric ($|\Phi_C|=4$) &
symmetric: transitive, 6/6 $\AutX$ elements map branch 1 onto branch 2; asymmetric: not transitive, 3 orbits of sizes $[2,1,1]$ &
\seqsplit{theoryB\_theorem3\_verify.py} \\

\Cref{thm:anchor-forcing} (broad search) &
random search, 4,000 constraint sets built from equality/inequality atoms with \textsc{and}/\textsc{or} (3,058 non-trivial: 2,479 conjunction-only, 579 disjunctive) &
22 of 2,479 purely-conjunctive instances non-transitive; 349 of 579 disjunctive instances remain transitive &
\seqsplit{theoryB\_falsification\_search.py} \\

\Cref{thm:anchor-forcing} (false-positive search) &
second random search, 2,295 instances (ambient group order $\leq15{,}000$) &
0 false positives; by connective, 15/15 (100\%) conjunctive non-transitive instances caught, 6/156 (3.8\%) \textsc{or}-containing non-transitive instances caught &
\seqsplit{theoryB\_theorem2\_verify.py}, \seqsplit{theoryB\_followups.py} \\

\Cref{thm:anchor-forcing} (multi-leg sweep) &
domain size $D=3$ ($k\in\{2,3,4\}$ attached legs) and $D=4$ ($k\in\{2,3\}$ legs) &
$D=3$: 2, 4, 8 orbits respectively; $D=4$: 2, 5 orbits respectively &
\seqsplit{theoryB\_followups.py} \\

\Cref{thm:degree-invariant} (abstract family) &
zero-factor family $\Phi=\{(c_1,c_2)\in\{0,\dots,m-1\}^2:c_1c_2=0\}$, exact enumeration at $m=3,4,5$ &
$\deg_1(0)=m$, $\deg_1(v)=1$ for $v\neq0$ in every case; non-transitive at all three values of $m$ &
\seqsplit{theoryB\_mechanisms\_D\_E\_verify.py} \\

\Cref{thm:degree-invariant} (additive control) &
additive control family $c_1+c_2=k$ over $\{0,\dots,m-1\}$, two tested configurations (no absorbing element) &
constant degree at every value; transitive in both configurations &
\seqsplit{theoryB\_mechanisms\_D\_E\_verify.py} \\

\Cref{thm:degree-invariant} (random battery) &
619 random instances built from equality/inequality atoms (ambient group order budget $\leq2{,}000$) &
0 false positives, 0 true positives &
\seqsplit{theoryB\_mechanisms\_D\_E\_verify.py} \\

\Cref{thm:free-slot} ((H-count) necessity) &
parametrized family, $r\in\{2,3,4,5\}$ branches, free-slot size $n_C$ swept above and below the threshold $r-1$ (14 exhaustive configurations) &
5/5 configurations with $n_C>r-1$ preserve branches; 5 further configurations with $2\leq n_C\leq r-1$ also preserve branches (hypothesis fails, conclusion still holds); 4 configurations at $n_C=1$ (all four values of $r$): branches merge, transitive &
\seqsplit{theoryB2\_free\_slot\_lemma\_generalization.py} \\

\Cref{lem:diagonal-collapse} (pair-richness and exact groups) &
order-3 Latin squares, $4\times4$ Sudoku with $2\times2$ boxes, order-4 Latin squares without boxes; full solution enumeration ($|\Phi|=12$, $288$, $576$) &
pair-richness holds for every same-row and same-column position pair in all three instances; exact componentwise groups by backtracking have orders 6, 24, 24 with every element diagonal; $|\Phi|>|\mathrm{Aut}|$ in all three &
\seqsplit{theoryE\_latin\_diagonal\_collapse.py} \\

\Cref{thm:deadvar-conp-complete} (gap pressure test) &
pressure test, 400 random small \textsc{unsat}-reduction instances (varying $k$, clause count, 3-CNF) &
108/400 (27\%) satisfiable $\chi$ yield a full-instance nontrivial automorphism from $\chi$'s own accidental symmetry &
\seqsplit{theoryA4\_leadA\_flipbit\_decomposition.py} \\

\Cref{cor:nontriv-conp-hard} (both directions) &
$k=4,\dots,10$, fresh random 3-CNF instances plus an adversarial battery with a chosen internal symmetry, plus $\chi=$ parity &
natural false-positive rate $18\%$--$30\%$ across $k=4..8$ (reproducing the 27\% figure above); isolation success rate tracks the predicted $\Theta(1/(k{+}1))$ rate ($8.7\%$ vs.\ $9.1\%$ reference at $k=10$; $9.3\%$ vs.\ $11.1\%$ at $k=8$); resulting $\chi'$ symmetry independently confirmed trivial on every successful run &
\seqsplit{theoryA7\_corollary21\_isolation\_fix\_results.json} \\

\Cref{sec:realmodels-cle4evr} (heterogeneous end-to-end) &
CLE4EVR's public rule and domains (color 2, shape 3, material 2; six constrained positions), synthetic prototype perception, two architectures (one head per attribute type; one head per position), 10 seeds each, 1{,}500 rule-satisfying evaluations per seed; predictions pre-specified in the script docstring; symbolic side asserted in-script against \Cref{tab:clevr-comparison}'s 6-dim row &
$|\Phic|=12$, componentwise $|\Aut|=24$ with 1 orbit, padded $|\Aut|=2$ with 6 orbits (assertion passes); typed arm 10{,}154 label-preserving errors, componentwise arm 10{,}069, \emph{zero} different-orbit errors in either; padded reading would misclassify 8{,}991 ($88.5\%$) and 7{,}865 ($78.1\%$) of them &
\seqsplit{realmodel\_cle4evr\_hetero.py} \\

\Cref{sec:realmodels-threelayers} (absorbing-element relocation, confusion null, probability mass, seed statistics) &
10 seeds each: product conjugated by $\pi(v)=v{+}3 \bmod 10$ (absorbing value $0\to7$); entropy-penalty and concept-supervision arms; per-instance probability-mass decomposition; confusion null at 200 replicates per seed on the same evaluation grid; seed-clustered bootstrap at 20{,}000 resamples &
relocation: all 33 shortcuts at $y^*{=}7$, none at $y{=}0$, same-orbit share $78.8\%$; null puts $99.9\%$ of label-preserving error at $y{=}0$ yet predicts only $6.2$ per seed against $9.4$ observed (excess CI95 $[2.3,4.2]$, $10/10$ seeds, $z=4.1$); mass at $y^*{=}0$ is $98.96\%$ true-orbit, $0.75\%$ orbit-external, $0.0\%$ orbit-external at every transitive level; two gradient interventions leave the rate unchanged (95 vs.\ 93; 101 vs.\ 94) &
\seqsplit{realmodel\_interventions.py}, \seqsplit{realmodel\_intervention\_b2.py}, \seqsplit{theoryJ\_support\_and\_null.py} \\

\Cref{sec:empirical-three} (support-restricted BDD-OIA) &
all $2^{20}$ concept vectors under two supports (full hypercube; the mutual-exclusion and presupposition subset), fiber-preserving single-coordinate flips and the subgroup they generate, per reachable label &
solution sets shrink by $2.6$--$6.7\times$; $\rho$ of the flip-generated subgroup stays in $[96.97\%, 99.9993\%]$; \texttt{green\_light} becomes a second dead coordinate on the four \texttt{move\_forward} labels, raising the flip-generated subgroup from order 2 to 4 there while the four \texttt{stop} labels keep order 2 &
\seqsplit{theoryJ\_support\_and\_null.py} \\

\Cref{sec:realmodels-threelayers} (dual-head control) &
10 seeds $\times$ 2 tasks retrained with two independent perception networks; identical pools, pairs, eval set, orbit tables, hyperparameters; predictions pre-specified in the script docstring before first execution &
\texttt{product}: 101 shortcuts (72 same-orbit, 29 diff-orbit), all at $y^*{=}0$, 10/10 seeds; all 29 diff-orbit cases collapse into $\{(0,0)\}$; 12 typed-ambiguous levels zero across 12{,}000 evaluations; \texttt{sum}: one same-orbit swap $(2,7)\to(7,2)$ in 20{,}000 &
\seqsplit{realmodel\_mnadd\_dualhead\_pipeline.py} \\

\Cref{obs:diagonal-containment} (bijective/general divide), \Cref{def:typed-autx} (typed levels), \Cref{sec:realmodels-threelayers} (counting conventions) &
the three-solution counterexample $\Phic=\{(0,0),(0,1),(1,0)\}$ exhaustively; 167 random bijective-set instances for the same-set embedding; exact typed orbits for every non-singleton \texttt{sum} and \texttt{product} level (value-support brute force; two size-10 supports handled by verified explicit constructions); all level and evaluation counts &
swap preserves $\Phic^{\mathrm{bij}}$ but not $\Phic$ (0 same-set violations); typed-non-transitive levels exactly $\{2,4,\dots,16\}$ (\texttt{sum}) and $\{0,4,9,16,36\}$ (\texttt{product}); counts $56=19{+}37$ reachable, $49$ non-singleton, $40{,}000=34{,}800{+}3{,}800{+}1{,}400$, $12{,}000$ at the twelve typed-ambiguous levels &
\seqsplit{theoryH\_review2\_checks.py} \\

\Cref{prop:free-action}, \Cref{thm:boolean-transitivity} (free action, coset classification, $\rho$ law) &
372 random Boolean instances ($n=3..6$, mixed random subsets and random affine cosets) plus all eight BDD-OIA/SDD-OIA rows of \Cref{tab:three-families} &
0 violations of free action, of transitive-iff-coset, and of the closed form; all eight table rows reproduced (orbit count, max orbit, $\rho$) to within $10^{-9}$ &
\seqsplit{theoryG\_depth\_explorations.py} \\

\Cref{lem:monotone-no-camouflage}, \Cref{thm:monotone-nontriv} (monotone package) &
300 random monotone functions ($n=2..5$, upward closures); 300 random 3-CNFs ($m=2..4$) through the double-rail gadget, exhaustive dead/alive checks of every position of $h$ &
$L_0=$ span of dead flips in all 300; implication-iff-\textsc{unsat} in all 300; rails always alive, $x_*$ dead iff unsatisfiable, $\NONTRIV(h)$ iff \textsc{unsat} in all 300 &
\seqsplit{theoryG\_depth\_explorations.py} \\

\Cref{prop:level-decomposition} (levels over mixed domains) &
40 random instances with domain sizes in $\{2,3\}$ ($n=2..4$ positions); 1{,}976 group elements checked against the fix-every-level criterion &
0 violations of dead-iff-absent, of the degree-profile identity, and of the equivariance criterion &
\seqsplit{theoryG\_depth\_explorations.py} \\

\Cref{prop:dead-var-oracle} (alive-promise gadget) &
500 random 3-CNF $\varphi$ over $m\in\{2,3,4\}$ variables (482 satisfiable, 18 unsatisfiable), exhaustive dead/alive check of every position of the gadget &
every $w_j$ and $t$ alive in all 500 instances; $x_i$ dead on exactly the unsatisfiable ones (0 violations in either direction) &
\seqsplit{theoryD\_deadvar\_alive\_promise.py} \\

Remark~\ref{rem:amplification-consequences} (single attempt) &
pooled $k=4,\dots,10$, same battery, single run of $R$ &
nontrivial on $56\%$--$67\%$ of satisfiable instances pooled; $60\%$--$94\%$ on the adversarial battery; $45\%$--$56\%$ on natural instances &
\seqsplit{theoryA7\_corollary21\_isolation\_fix\_results.json} \\

Remark~\ref{rem:amplification-consequences} (repetition) &
same battery, $t$ independent repetitions of $R$ &
corrected false-positive rate falls with $t$ at every $k$ tested, from $57$--$73\%$ at $t=1$ to $0$--$11.8\%$ at $t=16$ ($11.8\%$, $9.8\%$, $0\%$, $7.8\%$, $3.9\%$, $0\%$ for $k=4,5,6,7,8,10$; $k=10$: $56.9\%,35.3\%,21.6\%,2.0\%,0.0\%$ at $t=1,2,4,8,16$); Monte Carlo confirms the $1-(1-p)^t$ law to within $0.0095$ absolute error &
\seqsplit{theoryA7\_corollary21\_isolation\_fix\_results.json} \\

block-conjunction polarity (\Cref{sec:complexity-nontriv}) &
product-structure check across all four trivial/nontrivial block combinations, plus the degenerate empty-block case &
disjoint-block conjunction gives $\AutX=\prod_r\mathrm{Aut}(\Phi_{C_r})$, so it cannot amplify the reduction: 40/40 exact product-structure matches; 20/20 exact full-group matches (degenerate empty block) &
\seqsplit{theoryA7\_corollary21\_isolation\_fix\_results.json} \\

\Cref{lem:fourier-characterization} &
288 instances, dimensions 2--7, four generating families (uniform random truth tables, random clause conjunctions, affine functions, XORs of two random juntas) &
definitional $\AutX$ and the Fourier-support orthogonal complement agree exactly in every case (also re-confirms \Cref{cor:power-of-two} on the same instances) &
\seqsplit{theoryA2\_equivalence\_check.py} \\

\Cref{lem:and-gain} &
300 random pairs $(f,g)$, $n=6$ &
containment never fails; strict in 108/300 (36\%); representative case $|L_0(f)|=64$, $|L_0(g)|=8$, $|L_0(f)\cap L_0(g)|=8$, $|L_0(f\wedge g)|=64$ &
\seqsplit{theoryA2\_equivalence\_check.py} \\

\Cref{prop:bent-witness} &
$n=3,5,7,9$, $a$ set to the all-ones vector (largest Hamming weight) &
exhaustive enumeration confirms $L_0(f)=\{0,a\}$ exactly in all four cases; underlying inner-product bent functions ($n{-}1=2,4,6,8$) independently confirmed to have no nonzero linear structure &
\seqsplit{theoryA2\_directionB\_hamming\_weight.py} \\

flip-bit decomposition (\Cref{sec:complexity-structure}) &
600 instances, dimensions 1--6, four generating families &
$\NONTRIV(f)$ and $\bigvee_iQ_i(f)$ agree in every case &
\seqsplit{theoryA4\_leadA\_flipbit\_decomposition.py} \\

$Q_n=\DEADVAR(C,n)$ identity (\Cref{sec:complexity-structure}) &
960 instances, dimensions 1--6, four generating families &
$Q_n$ identical to direct \DEADVAR{} check in all cases &
\seqsplit{theoryA4\_leadA\_flipbit\_decomposition.py} \\

\Cref{prop:qi-gradient} &
$i=1,\dots,4$, $k=2,\dots,5$, 5,048 instances total, including three adversarial $\chi$ (constant-true, constant-false, $\chi=\mathrm{parity}$) &
$Q_i$ tracks \textsc{unsat} exactly in every one of the 5,048 trials, 0 mismatches &
\seqsplit{theoryA4\_leadA\_flipbit\_decomposition.py} \\

\Cref{lem:orbit-uniformity} &
160 isomorphic pairs and 160 non-isomorphic pairs, dimensions 2--5, checked exactly (not by sampling) &
distributional equality (isomorphic case) and orbit disjointness (non-isomorphic case) each hold without exception &
\seqsplit{theoryA3\_val\_iso\_orbit\_uniformity.py} \\

\Cref{thm:nontriv-collapse} ($\mathrm{canon}^*$ sub-check) &
69 instances (dimensions 2--7, five generating families including a hand-built copy of BDD-OIA's absorption structure), generous tape budget; 32 further instances with tape budget deliberately tightened to force aborts &
69/69 exact orbit-size match, zero aborts; representative set never exceeds the true orbit size on the tightened-budget instances; amplifying the tape count recovers exact orbit sizes in all 32 &
\seqsplit{theoryA5\_orbit\_canon\_counting\_v2\_results.json} \\

\Cref{thm:nontriv-collapse} (end to end) &
28 real instances, five generating families (random, clause conjunctions, affine, constant, hand-built BDD-OIA absorption copy); 3 trivial and 25 nontrivial $\AutX$ &
laundered orbit-representative set recovers the true orbit size exactly in 28/28 cases; single-round trivial/nontrivial acceptance gap is a factor of about 4.3 ($0.41$ vs.\ $0.10$); the proof's 5-of-15 threshold amplification (which replaces the earlier majority scheme, unsound at completeness $3/8<1/2$) is verified by exact binomial computation, $0.720$ vs.\ $0.314$, giving $(1-2^{-n})\cdot0.720\geq0.708$ after conditioning on a good tape at $n{=}6$ &
\seqsplit{theoryA5\_setlowerbound\_on\_canon\_orbit\_v2\_results.json} \\

\end{longtable}
\endgroup

\paragraph{Execution environment.}
Every script in the released artifact runs on CPU except the four training
pipelines, which use Apple Silicon GPU acceleration through PyTorch's MPS
backend (Apple M-series, 16\,GB unified memory, macOS 15). Software:
Python 3.11, PyTorch 2.13 with torchvision 0.28, \texttt{pynauty} for the
graph-automorphism engine, SymPy for the BDD-OIA rule extraction, and the
Python standard library elsewhere. Wall-clock times are recorded in each
result file: the four training runs take 147--473\,s
each; the symbolic scripts run in seconds
apart from the two exhaustive $2^{20}$ sweeps, which take a few minutes
each. Random number generation is seeded explicitly in every script
(\texttt{random.seed}, \texttt{numpy.random.seed}, and
\texttt{torch.manual\_seed} where applicable), with the seed values
reported alongside each experiment. Every number in the paper traces,
through a source comment, to a named field in a result file or to a
named figure printed by a seeded script. The artifact ships the
verification scripts, the graph-automorphism helper module, the
per-run result files behind the symbolic and training analyses
(\texttt{realmodel\_mnadd\_full\_results.json} and
\texttt{theoryB\_falsification\_search\_results.json} among them),
and a \texttt{requirements.txt} listing the symbolic-check
dependencies (\texttt{numpy}, \texttt{sympy}, \texttt{pynauty}).
All scripts resolve their inputs and outputs relative to the artifact
directory; the symbolic checks were re-run from the released tree in
a fresh virtual environment. The four training pipelines additionally
require \texttt{torch}, \texttt{torchvision}, and the public
\texttt{rsseval} package, and their shipped result files carry every
number the paper cites, so re-training is not required to check any
reported value.

\bibliographystyle{plainnat}
\bibliography{refs}

\end{document}